\documentclass{article}

\PassOptionsToPackage{numbers, compress}{natbib}

\usepackage[final]{neurips_2020}

\usepackage[markup=underlined]{changes}
\definechangesauthor[color=red]{DJ}

\usepackage[utf8]{inputenc} 
\usepackage[T1]{fontenc}    
\usepackage{hyperref}  
\usepackage{url}            
\usepackage{booktabs}       
\usepackage{amsfonts}       
\usepackage{nicefrac}       
\usepackage{microtype}      
\usepackage{comment}
\usepackage{subfigure}
\usepackage[skip=3ex,font=small]{caption} 

\usepackage{amsmath}
\usepackage{bm}
\usepackage{bbm}
\usepackage{amssymb}
\usepackage{amsthm}
\usepackage{cleveref}
\usepackage{color}
\usepackage{graphicx}
\usepackage{enumitem}
\usepackage{soul}

\usepackage[hang,flushmargin]{footmisc} 

\usepackage{color}

\usepackage{listings}
\DeclareFixedFont{\ttb}{T1}{txtt}{bx}{n}{9.5} 
\DeclareFixedFont{\ttm}{T1}{txtt}{m}{n}{9.5}  
\definecolor{codeblue}{rgb}{0,0,0.6}
\definecolor{codegreen}{rgb}{0,0.6,0}
\definecolor{dark-blue}{rgb}{0.15,0.15,0.4}
\definecolor{codepurple}{rgb}{0.6,0,0.6}

\hypersetup{%
  colorlinks=true,
  linkcolor=blue,
  citecolor=dark-blue,
  urlcolor=black,
}
\theoremstyle{definition}

\theoremstyle{plain}
\newtheorem{theorem}{Theorem}
\newtheorem{proposition}{Proposition}

\newtheorem{lemma}{Lemma}

\newtheorem{manualtheoreminner}{Theorem}
\newenvironment{manualtheorem}[1]{%
  \renewcommand\themanualtheoreminner{#1}%
  \manualtheoreminner
}{\endmanualtheoreminner}

\newcommand\pythonstyle{\lstset{
    language=Python,
    basicstyle=\scriptsize\ttfamily,
    otherkeywords={self,with},             
    keywordstyle=\color{codepurple},
    emph={__init__, dim, None},
    emphstyle=\color{codeblue},
    stringstyle=\color{codegreen},
    commentstyle=\color{codegreen},
    frame=none,              
    showstringspaces=false,
    breaklines=true,
    numbers=left,
    numbersep=3pt,
    tabsize=2,
    breakatwhitespace=false,
    abovecaptionskip=2ex,
    captionpos=b,
}}

\lstnewenvironment{python}[1][]
{
    
    \pythonstyle
    \lstset{#1}
}{}

\lstnewenvironment{pythoninline}[1][]
{
    \pythonstyle
    \lstset{#1}
}{}

\DeclareMathOperator*{\argmax}{arg\,max}

\DeclareMathOperator*{\variance}{Var}

\newif\ifboldmatrix
\boldmatrixfalse
\ifboldmatrix\else\fi

\newcommand{\bx}{\ensuremath{\mathbf{x}}}

\newcommand{\by}{\ensuremath{\mathbf{y}}}

\newcommand{\bbE}{\ensuremath{\mathbb{E}}}

\newcommand{\bbP}{\ensuremath{\mathbb{P}}}
\newcommand{\bbR}{\ensuremath{\mathbb{R}}}
\newcommand{\bbX}{\ensuremath{\mathbb{X}}}

\newcommand{\calD}{\ensuremath{\mathcal{D}}}

\newcommand{\botorch}{\textsc{BoTorch}}

\newcommand{\OKG}{\textsc{OKG}}

\usepackage{float}
\newfloat{codeexample}{thp}{lop}
\floatname{codeexample}{Code Example}

\defcitealias{malariaatlas2019}{The Malaria Atlas Project (2019)}

\newcommand{\papertitle}{\botorch: A Framework for Efficient Monte-Carlo Bayesian Optimization}

\title{\papertitle}

\author{%
  Maximilian Balandat\\
  Facebook\\
  \texttt{balandat@fb.com} \\
  \And
  Brian Karrer\\
  Facebook\\
  \texttt{briankarrer@fb.com} \\
  \And
  Daniel R. Jiang\\
  Facebook\\
  \texttt{drjiang@fb.com} \\
  \And
  Samuel Daulton\\
  Facebook\\
  \texttt{sdaulton@fb.com} \\
  \And
  Benjamin Letham\\
  Facebook\\
  \texttt{bletham@fb.com} \\
  \And
  Andrew Gordon Wilson\\
  New York University\\
  \texttt{andrewgw@cims.nyu.edu} \\
  \And
  Eytan Bakshy\\
  Facebook\\
  \texttt{ebakshy@fb.com} \\
}

\begin{document}

\maketitle

\begin{abstract}
Bayesian optimization provides sample-efficient global optimization for a broad range of applications, including automatic machine learning, engineering, physics, and experimental design. We introduce \botorch{}, a modern programming framework for Bayesian optimization that combines Monte-Carlo (MC) acquisition functions, a novel sample average approximation optimization approach, auto-differentiation, and variance reduction techniques. \botorch{}'s modular design facilitates flexible specification and optimization of probabilistic models written in \mbox{PyTorch}, simplifying implementation of new acquisition functions. 
Our approach is backed by novel theoretical convergence results and made practical by a distinctive algorithmic foundation that leverages fast predictive distributions, hardware acceleration, and deterministic optimization. We also propose a novel ``one-shot'' formulation of the Knowledge Gradient, enabled by a combination of our theoretical and software contributions. In experiments, we demonstrate the improved sample efficiency of \botorch{} relative to other popular libraries.
\end{abstract}

\section{Introduction}
\label{sec:Introduction}

Computational modeling and machine learning (ML) have led to an acceleration of scientific innovation in diverse areas, ranging from drug design to robotics to material science. These tasks often involve solving time- and resource-intensive global optimization problems to achieve optimal performance. Bayesian optimization (BO)~\citep{ohagan1978curve, jones98, osborne2010bayesian}, an established methodology for sample-efficient sequential optimization, has been proposed as an effective solution to such problems, and has been applied successfully to tasks ranging from hyperparameter optimization \citep{feurer2015automl,shahriari16review,wu2019cfkg}, robotic control~\citep{calandra2016gait,antonova2017deep}, chemical  design~\citep{griffiths2017constrained,li2017rapid,zhang2020materials}, and tuning and policy search for internet-scale software systems~\citep{agarwal2018bayesopt,letham2019noisyei,letham2019bayesian,cbo}. 
Meanwhile, ML research has been undergoing a revolution driven largely by new programming frameworks and hardware that reduce the time from ideation to execution \citep{jia2014caffe, chen2015mxnet, abadi2016tensorflow, paszke2017automatic}. 
While BO has become rich with new methodologies, today there is no coherent framework that leverages these computational advances to simplify and accelerate BO research in the same way that modern  frameworks have for deep learning.
In this paper, we address this gap by introducing \botorch{}, a modular and scalable Monte Carlo (MC) framework for BO that is built around modern paradigms of computation, and theoretically grounded in novel convergence results. Our contributions include:

\begin{itemize}[leftmargin=2.6ex,topsep=-0.25ex,itemsep=0.05ex]
    \item A novel approach to optimizing MC acquisition functions that effectively combines with deterministic higher-order optimization algorithms and variance reduction techniques.
    \item The first 
    convergence results for sample average approximation (SAA) of MC acquisition functions, including novel general convergence results for SAA via randomized quasi-MC.
    \item A new, SAA-based ``one-shot'' formulation of the Knowledge Gradient, a look-ahead acquisition function, with improved performance over the state-of-the-art.
    \item Composable model-agnostic abstractions for MC BO that leverage modern computational technologies, including auto-differentiation and scalable parallel computation on CPUs and GPUs.
\end{itemize}

We discuss related work in Section \ref{subsec:Introduction:RelatedWork} and then present the methodology underlying \botorch{} in Sections \ref{sec:Acquisition} and \ref{sec:OptimizeAcq}. Details of the \botorch{} framework, including its modular abstractions and implementation examples, are given in Section \ref{sec:ModularBO}. Numerical results are provided in Section \ref{sec:Experiments}.

\section{Background and Related Work}
\label{subsec:Introduction:RelatedWork}

In BO, we aim to solve $\max_{x \in \mathbb X} f_{\mathrm{true}}(x)$, where $f_{\mathrm{true}}$ is an expensive-to-evaluate function and $\mathbb X \!\subset\! \bbR^d$ is a feasible set. BO consists of two main components: a \emph{probabilistic surrogate model} of the observed function---most commonly, a Gaussian process (GP)---and an \emph{acquisition function} that encodes a strategy for navigating the exploration vs. exploitation trade-off~\citep{shahriari16review}. Taking a model-agnostic view, our focus in this paper is on MC acquisition functions.

Popular libraries for BO include
Spearmint \citep{snoek2012practical},
GPyOpt \citep{gpyopt2016},
Cornell-MOE~\citep{wu2016parallel},
RoBO \citep{klein17robo},
Emukit \citep{emukit2018},
and Dragonfly \citep{kandasamy2019dragonfly}.
We provide further discussion of these packages in Appendix~\ref{appdx:sec:PackageComparison}.
Two other libraries, ProBO \citep{neiswanger2019probo} and GPFlowOpt \citep{knudde2017gpflowopt}, are of particular relevance. 
ProBO is a recently suggested framework\footnote{No implementation of ProBO is available at the time of this writing.} 
for using general probabilistic programming in BO. While its model-agnostic approach is similar to ours, ProBO, unlike \botorch{}, does not benefit from  gradient-based optimization provided by differentiable programming, or algebraic methods designed to exploit GPU acceleration.
GPFlowOpt inherits support for auto-differentiation and hardware acceleration from TensorFlow \citep[via GPFlow,][]{matthews2017gpflow}, but unlike \botorch{}, it does not use algorithms designed to specifically exploit this potential. Neither ProBO nor GPFlowOpt naturally support MC acquisition functions. 
In contrast to all existing libraries, \botorch{} is a modular programming framework and employs novel algorithmic approaches that achieve a high degree of flexibility and performance.

The MC approach to optimizing acquisition functions has been considered in the BO literature to an extent, typically using stochastic methods for optimization \citep{wang2016parallel,wu2016parallel,wu2017bayesiangrad, wilson2017reparamacq}. Our work takes the distinctive view of sample average approximation (SAA), an approach that combines sampling with deterministic optimization and variance reduction techniques. To our knowledge, we provide the first theoretical analysis and systematic implementation of this approach in the BO setting. 

\section{Monte-Carlo Acquisition Functions}
\label{sec:Acquisition}

We begin by describing a general formulation of BO in the context of MC acquisition functions.
Suppose we have collected data $\calD = \{(x_i, y_i)\}_{i=1}^n$, where $x_i \in \bbX$ and $y_i = f_\mathrm{true}(x_i) + v_i(x_i)$ with $v_i$ some noise corrupting the true function value $f_\mathrm{true}(x_i)$. We allow~$f_\mathrm{true}$ to be multi-output, in which case $y_i, v_i \in \bbR^m$. In some applications we may also have access to distributional information of the noise $v_i$, such as its (possibly heteroskedastic) variance. %
Suppose further that we have a probabilistic surrogate model $f$ that for any  $\bx := \{x_1, \dotsc, x_q\}$ provides a distribution over $f(\bx) := (f(x_1), \dotsc, f(x_q))$ and $y(\bx) := (y(x_1), \dotsc, y(x_q))$. We denote by $f_{\calD}(\bx)$ and $y_{\calD}(\bx)$ the respective $\emph{posterior}$ distributions conditioned on data~$\calD$.
In BO, the model $f$ traditionally is a GP, and the $v_i$ are assumed i.i.d. normal, in which case both $f_{\calD}(\bx)$ and $y_{\calD}(\bx)$ are multivariate normal. The MC framework we consider here makes no particular assumptions about the form of these posteriors.

The next step in BO is to optimize an acquisition function evaluated on $f_{\calD}(\bx)$ 
over the \emph{candidate set}~$\bx$.
Following \citep{wilson2018maxbo, bect2019supermartingaleGP}, many acquisition functions can be written as
\begin{align}
\label{eq:Acquisition:Myopic:Basic}
    \alpha(\bx; \Phi, \calD) = \bbE \bigl[ a(g(f(\bx)), \Phi) \,|\, \calD\bigr], 
\end{align}
where $g: \bbR^{q \times m\!} \rightarrow \bbR^q$ is a (composite) \emph{objective function}, $\Phi \!\in\! \mathbf{\Phi}$ are parameters independent of~$\bx$ in some set $\mathbf{\Phi}$, and $a: \bbR^{q\!} \times \mathbf{\Phi} \!\rightarrow\! \bbR$ is a \emph{utility function} that defines the acquisition function. 

In some situations, the expectation over  $f_{\calD}(\bx)$  in~\eqref{eq:Acquisition:Myopic:Basic} and its gradient $\nabla_{\bx} \alpha(\bx; \Phi, \calD)$ can be computed analytically, e.g. if one considers a single-output ($m\!=\!1$) model, a single candidate ($q\!=\!1$) point~$x$, a Gaussian posterior $f_\calD(x) = \mathcal{N}(\mu_x, \sigma_x^2)$, and the identity objective $g(f) \equiv f$. Expected Improvement (EI) is a popular acquisition function that maximizes the expected difference between the currently observed best value $f^*$ (assuming noiseless observations) and $f$ at the next query point, through the utility  $a(f, f^*) = \max(f - f^*, 0)$. EI and its gradient have a well-known analytic form~\citep{jones98}.

In general, analytic expressions are not available for arbitrary objective functions~$g(\cdot)$, utility functions $a(\cdot\, , \cdot)$, non-Gaussian model posteriors, or collections of points $\bx$ which are to be evaluated in a parallel or asynchronous fashion \citep{ginsbourger11, snoek2012practical,wu2016parallel,wang2016parallel, wilson2017reparamacq}. 
Instead, MC integration can be used to approximate the expectation~\eqref{eq:Acquisition:Myopic:Basic} using samples from the posterior. 
An MC approximation $\hat{\alpha}_{\!N}(\bx; \Phi,\calD)$ of~\eqref{eq:Acquisition:Myopic:Basic} using $N$ samples $\xi_\calD^i(\bx) \sim f_\calD(\bx)$ is straightforward:
\begin{align}
\label{eq:Acquisition:MyopicMC}
    \hat{\alpha}_{\!N}(\bx; \Phi,\calD) = \frac{1}{N}\! \sum_{i=1}^N a(g(\xi_\calD^i(\bx)), \Phi).
\end{align}
The obvious way to evaluate~\eqref{eq:Acquisition:MyopicMC} is to draw i.i.d. samples $\xi_\calD^i(\bx)$. Alternatively, randomized quasi-Monte Carlo (RQMC) techniques \citep{caflisch1998monte} can be used to significantly reduce the variance of the estimate and its gradient (see Appendix~\ref{appdx:sec:NonStochOpt} for additional details).

\begin{figure}[t]
  \centering
  \includegraphics[width=0.95\textwidth]{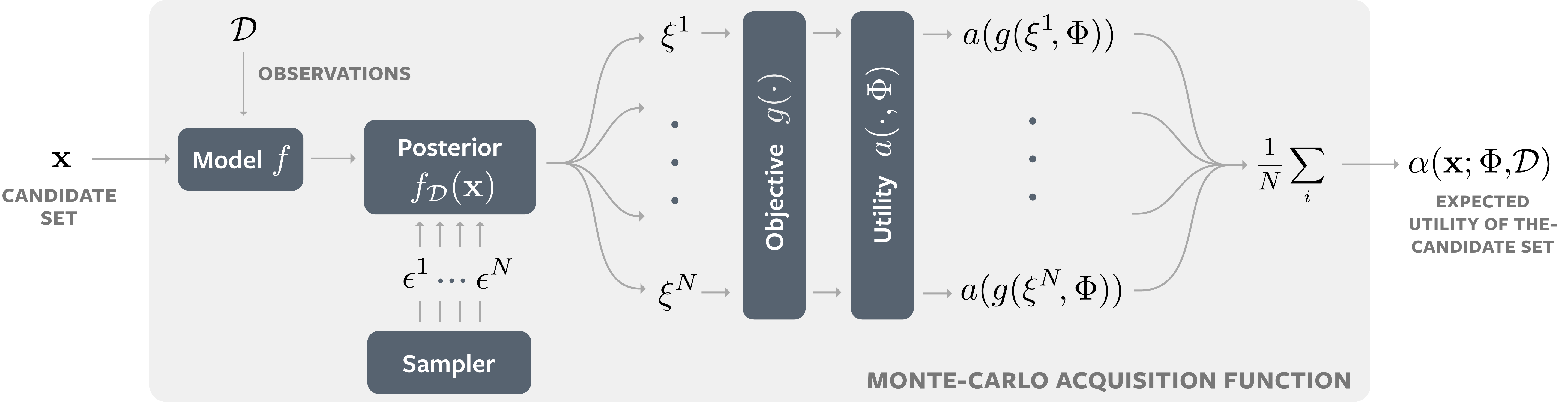}\\[-2ex]
  \caption{MC acquisition functions. 
  Samples $\xi_\calD^i$ from the posterior $f_\calD(\bx)$ 
  provided by the model $f$ at $\bx$ 
  are evaluated in parallel and averaged as in~\eqref{eq:Acquisition:MyopicMC}. All operations are fully differentiable.}
  \label{fig:Acquisition:MCAcq}
\vspace{-2ex}
\end{figure}

\section{MC Bayesian Optimization via Sample Average Approximation}
\label{sec:OptimizeAcq}

To generate a new candidate set $\mathbf{x}$, one must optimize the acquisition function~$\alpha$. 
Doing this effectively, especially in higher dimensions, typically requires using gradient information. 
For differentiable analytic acquisition functions (e.g. EI, UCB), one can either manually implement gradients, or use auto-differentiation to compute $\nabla_{\!x} \alpha(x; \Phi,\calD)$, provided one can differentiate through the posterior parameters (as is the case for Gaussian posteriors).

\subsection{Optimizing General MC Acquisition Functions}
An unbiased estimate of the MC acquisition function  gradient $\nabla_{\!\bx} \alpha(\bx; \Phi,\calD)$ can often be obtained from~\eqref{eq:Acquisition:MyopicMC} via the reparameterization trick~\citep{kingma2013reparam,rezende2014stochbackprop}. The basic idea is that $\xi \sim f_\calD(\bx)$ can be expressed as a suitable (differentiable) deterministic transformation $\xi = h_{\calD}(\bx, \epsilon)$ of an auxiliary random variable $\epsilon$ independent of~$\bx$.
For instance, if $f_\calD(\bx) \sim \mathcal{N}(\mu_\bx, \Sigma_\bx)$, then $h_{\calD}(\bx, \epsilon) = \mu_\bx + L_{\bx} \epsilon$, with $\epsilon \sim \mathcal{N}(0, I)$ and $L_{\bx}L_{\bx}^T = \Sigma_\bx$.
If $a( \cdot ,\Phi)$ and $g( \cdot )$ are differentiable, then $\nabla_{\!\bx} a(g(\xi), \Phi) = \nabla_{\!g} a \nabla_{\!\xi} g \nabla_{\!\bx} h_{\calD}(\bx, \epsilon)$. 

Our primary methodological contribution is to take a sample average approximation \citep{kleywegt2002sample} approach to BO. The conventional way of optimizing MC acquisition functions of the form~\eqref{eq:Acquisition:MyopicMC} is to re-draw samples from $\epsilon$ for each evaluation and apply stochastic first-order
methods such as Stochastic Gradient Descent (SGD) \citep{wilson2018maxbo}.
In our SAA approach, rather than re-drawing samples from $\epsilon$ for each evaluation of the acquisition function, we draw a set of base samples $E:=\{\epsilon^i\}_{i=1}^N$ once, and hold it fixed between evaluations throughout the course of optimization (this can be seen as a specific incarnation of the method of common random numbers). Conditioned on~$E$, the resulting MC estimate $\hat{\alpha}_{\!N}(\bx;\Phi,\calD)$ is deterministic. We then obtain the candidate set~$\hat{\bx}_{\!N}^*$ as
\begin{align}
\label{eq:OptimizeAcq:SAA:optimizer}
    \hat{\bx}_{N}^* \in \argmax_{\bx \in \bbX^q} \hat{\alpha}_{\!N}(\bx; \Phi, \calD).
\end{align}
The gradient $\nabla_\bx\hat{\alpha}_{\!N}(\bx; \Phi, \calD)$ can be computed as the average of the sample-level gradients, exploiting auto-differentiation. We emphasize that whether this average is a ``proper'' (i.e., unbiased, consistent) estimator of $\nabla_\bx \alpha(\bx; \Phi,\calD)$ is irrelevant for the convergence results we will derive below.

While the convergence properties of MC integration are well-studied \citep{caflisch1998monte}, the respective literature on SAA, i.e., convergence of the \emph{optimizer}~\eqref{eq:OptimizeAcq:SAA:optimizer}, is far less comprehensive. 
Here, we derive what, to the best of our knowledge, are the first SAA convergence results for (RQ)MC acquisition functions in the context of BO. 
To simplify our exposition, we limit ourselves to GP surrogates 
and i.i.d. base samples; more general 
results and proofs are presented in Appendix~\ref{appdx:sec:GeneralSAAresults}.
For notational simplicity, we will drop the dependence of $\alpha$ and $\hat{\alpha}_{\!N}$ on $\Phi$ and $\calD$ for the remainder of this section. Let $\alpha^* := \max_{\bx\in \bbX^q} \alpha(\bx)$, and denote by $\mathcal{X}^*$ the associated set of maximizers. 
Similarly, let $\hat{\alpha}_{\!N}^* := \max_{\bx\in \bbX^q} \hat{\alpha}_{\!N}(\bx)$. 
With this we have the following key result:
\begin{theorem}
\label{thm:OptimizeAcq:SAAConvergence}
Suppose (i) $\bbX$ is compact,
(ii) $f$ has a GP prior with continuously differentiable mean and covariance functions, and (iii) $g( \cdot )$ and $a( \cdot , \Phi)$ are Lipschitz continuous. 
If the base samples $\{\epsilon^i\}_{i=1}^N$ are  i.i.d. $\mathcal N(0,1)$, 
then 
(1) $\hat{\alpha}_{\!N}^* \rightarrow \alpha^*$ a.s., and
(2) $\textnormal{dist}(\hat{\bx}_{\!N}^*, \mathcal{X}^*) \rightarrow 0$ a.s..
Under additional regularity conditions, (3) $\forall\, \delta>0$, $\exists\, K <\infty$, $\beta > 0$ s.t.  $\mathbb{P}\bigl(\textnormal{dist}(\hat{\bx}_{\!N}^*, \mathcal{X}^*) > \delta \bigr) \le K e^{-\beta N}\!,$  $\forall \, N \geq 1$.
\end{theorem}

Under relatively weak conditions,\footnote{Many utility functions~$a$ are Lipschitz, including those representing (parallel) EI and UCB \citep{wilson2017reparamacq}. Lipschitzness is a sufficient condition, and convergence can also be shown in less restrictive settings (see Appendix~\ref{appdx:sec:GeneralSAAresults}).}
Theorem~\ref{thm:OptimizeAcq:SAAConvergence} ensures not only that the optimizer $\hat{\bx}_{\!N}^*$ of $\hat{\alpha}_{\!N}$ converges to an optimizer of the true $\alpha$ with probability one, but also that the convergence (in probability) happens at an exponential rate. 
We stated Theorem~\ref{thm:OptimizeAcq:SAAConvergence} informally and for i.i.d. base samples for simplicity. In Appendix~\ref{appdx:subsec:GeneralSAAresults:RQMC} we give a formal statement, and extend it to base samples generated by a family of RQMC methods, leveraging recent theoretical advances \citep{owen2020rqmcslln}. 
While at this point we do not characterize improvements in theoretical convergence rates of RQMC over MC for SAA, we observe empirically that RQMC methods work remarkably well in practice (see Figures~\ref{fig:OptimizeAcq:MC_QMC_qualitative} and~\ref{fig:OptimizeAcq:SAAConvergence:MCqMC}).\\[-3ex]
\begin{figure*}[ht]
    \begin{minipage}[b]{0.55\columnwidth}
        \centering
        \hspace{-2ex}
        \includegraphics[width=\textwidth]{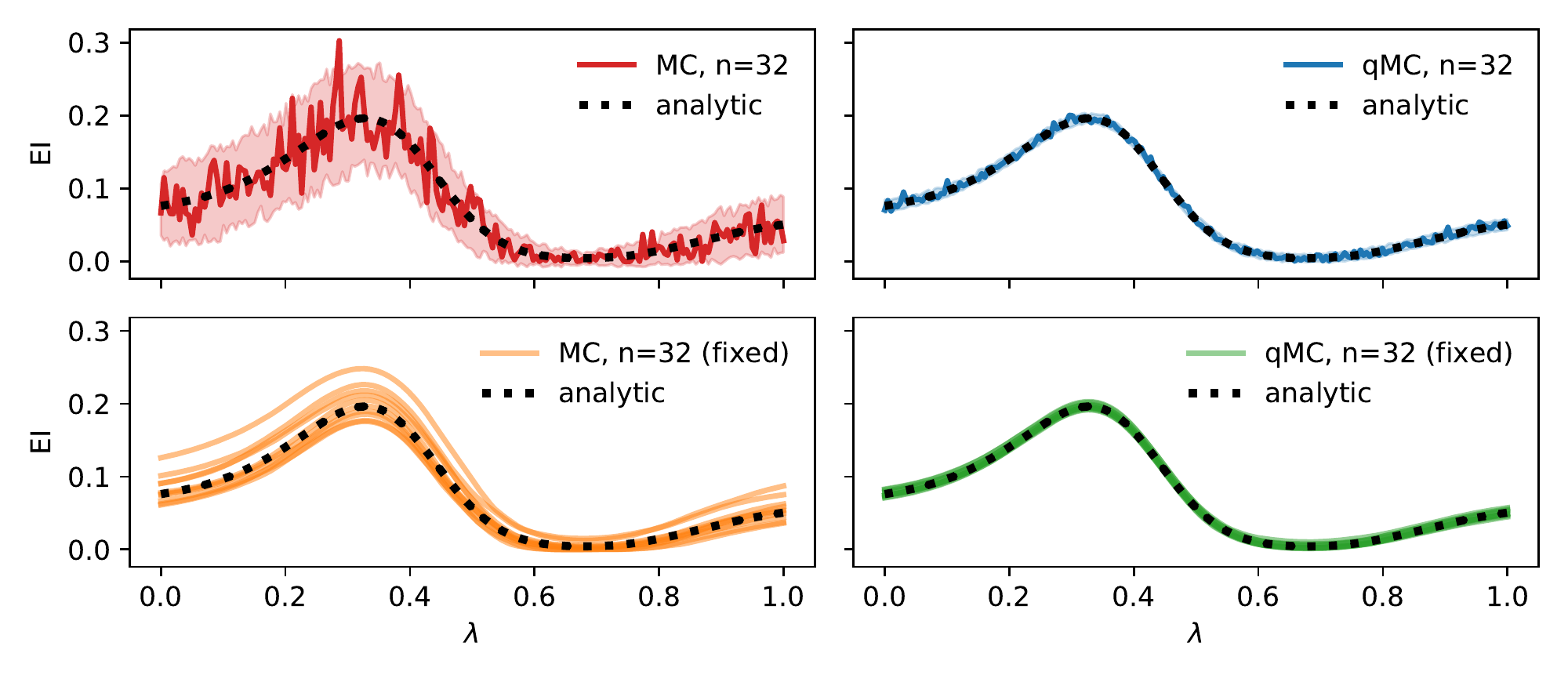}\\[-3ex]
        \caption{MC and RQMC acquisition functions, with and without (``fixed'') re-drawing base samples between evaluations. The model is a GP fit on 15 points randomly sampled from $\bbX = [0, 1]^6$ and evaluated on the Hartmann6 function along the slice $x(\lambda) = \lambda \mathbf{1}$.}
        \label{fig:OptimizeAcq:MC_QMC_qualitative}
    \end{minipage}
    \hfill
    \begin{minipage}[b]{0.41\columnwidth}
        \centering
        \hspace{-4.1ex}
        \includegraphics[width=1.1\textwidth]{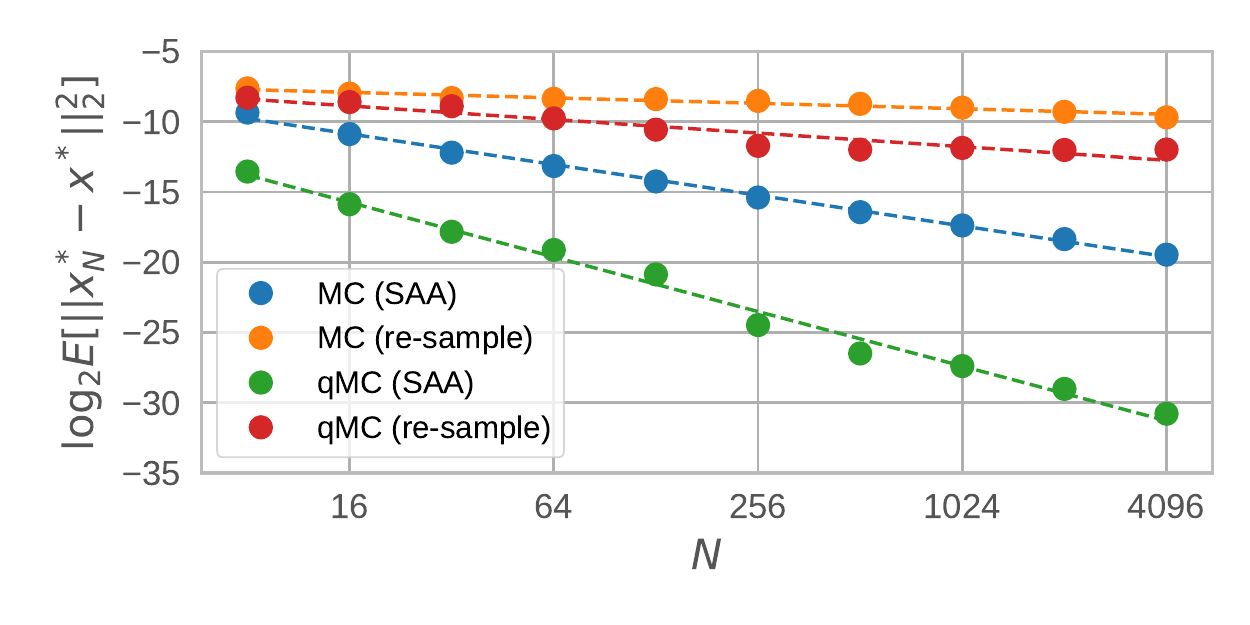}\\[-3ex]
        \caption{Empirical convergence rates of the optimizer for EI using MC / RQMC sampling under SAA / stochastic optimization (``re-sample''). Appendix~\ref{appdx:sec:NonStochOpt} provides additional detail and discussion.}
        \label{fig:OptimizeAcq:SAAConvergence:MCqMC}
    \end{minipage}
\vspace{-2ex}
\end{figure*}

The primary benefit from SAA comes from the fact that in order to optimize $\hat{\alpha}_{\!N}(\bx; \Phi, \calD)$ for fixed base samples~$E$, one can now employ the full toolbox of deterministic optimization, including quasi-Newton methods that provide faster convergence speeds and are generally less sensitive to optimization hyperparameters than stochastic first-order methods.
By default, we use multi-start optimization via \mbox{L-BFGS-B} in conjunction with an initialization heuristic that exploits fast batch evaluation of acquisition functions (see Appendix~\ref{appdx:subsec:ImplementatonDetails:ICHeuristic}). We find that in practice the bias from using SAA only has a minor effect on the performance relative to using the analytic ground truth, and often improves performance relative to stochastic approaches (see Appendix~\ref{appdx:sec:NonStochOpt}), while avoiding tedious tuning of optimization hyperparameters such as learning rates. 

\subsection{One-Shot Formulation of the Knowledge Gradient using SAA}
\label{sec:OneShotKG}

The acquisition functions mentioned above, such as EI and UCB, are \emph{myopic}, that is, they do not take into account the effect of an observation on the model in future iterations. In contrast, \emph{look-ahead} methods do. Our SAA approach enables a novel formulation of a class of look-ahead acquisition functions. For the purpose of this paper we focus on the Knowledge Gradient (KG)~\citep{frazier2008knowledge}, but our methods extend to other look-ahead acquisition functions such as two-step EI~\citep{wu2019practical}.

KG quantifies the expected increase in the maximum of~$f$ from 
obtaining the additional (random) observation data $\{\bx, y_\calD(\bx)\}$.
KG often shows improved BO performance relative to simpler, myopic acquisition functions such as EI~\citep{scott11kg}, but in its traditional form it is computationally expensive and hard to implement, two challenges that we address in this work. 
Writing $\calD_{\bx} := \calD \cup \{\bx, \by_\calD(\bx)\}$, we introduce a generalized variant of parallel KG (qKG) \citep{wu2016parallel}: 
\begin{align}
\label{eq:Acquisition:Lookahead:GenKG}
    \alpha_{\mathrm{KG}}(\bx; \calD) = \bbE \Bigl[\, \max_{x' \in \bbX} \bbE \bigl[ g(f(x')) \, | \, \calD_{\bx} \bigr] \, | \, \calD \Bigr] - \mu_\calD^*,
\end{align}
with $\mu_\calD^* := \max_{x \in \bbX}\bbE[g(f(x)) \!\mid\! \mathcal{D}]$. 
Equation \eqref{eq:Acquisition:Lookahead:GenKG} quantifies the expected increase in the maximum posterior mean of $g \circ f$ after gathering samples at~$\bx$.
For simplicity, we only consider standard BO here, but extensions for multi-fidelity optimization \citep{poloczek2017misokg,wu2019cfkg} are also available in \botorch{}.

Maximizing KG requires solving a nested optimization problem. 
The standard approach is to optimize the inner and outer problems separately, in an iterative fashion. The outer problem is handled using stochastic gradient ascent, with each gradient observation potentially being an average over multiple samples \citep{wu2016parallel,wu2017bayesiangrad}. For each sample $y_\calD^i(\bx) \sim y_\calD(\bx)$, the inner problem $\max_{x_i \in \bbX} \bbE \left[ f(x_i) \mid \calD_\bx^i \right]$ is solved numerically, either via another stochastic gradient ascent \citep{wu2017bayesiangrad} or multi-start L-BFGS-B \citep{frazier2018tutorial}. An unbiased stochastic gradient can then be computed by leveraging the envelope theorem. 
Alternatively, the inner problem can be discretized \citep{wu2016parallel}. 
The computational expense of this nested optimization can be quite large; our main insight is that it may also be unnecessary.

We treat optimizing $\alpha_{\mathrm{KG}}(\bx, \calD)$ in~\eqref{eq:Acquisition:Lookahead:GenKG} as an entirely deterministic problem using SAA.
Using the reparameterization trick, we express $y_\calD(\bx) = h_\calD^y(\bx, \epsilon)$ for some deterministic~$h_\calD$,\footnote{For a GP, $h^y_{\calD}(\bx, \epsilon) = \mu_\calD(\bx) + L^\sigma_\calD(\bx) \epsilon$, with $L^\sigma_\calD(\bx)$ a root decomposition of $\Sigma^\sigma_\calD(\bx) := \Sigma_\calD(\bx,\bx) + \Sigma^v(\bx)$.} 
and draw $N$ fixed base samples~$\{ \epsilon^i\}_{i=1}^{N}$ for the outer expectation. The resulting MC approximation of KG is:
\begin{align}
\label{eq:OneShotKG:SAA:nested}
    \hat{\alpha}_{\mathrm{KG},N}(\bx; \calD) = \frac{1}{N} \sum_{i=1}^{N} \max_{x_i \in \mathbb{X}} \bbE \bigl[ g(f(x_i)) \, | \, \calD_\bx^i \bigr] - \mu^*.
\end{align}
\begin{theorem}
\label{thm:OneShotKG:OneStep:Convergence}
    Suppose conditions (i) and (ii) of Theorem \ref{thm:OptimizeAcq:SAAConvergence} hold, and that (iii) $g( \cdot )$ is affine. If the base samples $\{\epsilon^i\}_{i\geq 1}$ are drawn i.i.d from $\mathcal N(0,1)$, then 
    (1) $\hat{\alpha}_{\mathrm{KG},N}^* \rightarrow \alpha_{\mathrm{KG}}^*$ a.s.,
    (2) $\textnormal{dist}(\hat{\bx}_{\mathrm{KG},N}^*, \mathcal{X}_{\mathrm{KG}}^*) \rightarrow 0$ a.s.,
    and (3) $\forall\, \delta>0$, $\exists\, K <\infty$, $\beta > 0$ s.t.  $\mathbb{P}\bigl(\textnormal{dist}(\hat{\bx}_{\mathrm{KG},N}^*, \mathcal{X}_{\mathrm{KG}}^*) > \delta \bigr) \le K e^{-\beta N}$ for all $N \ge 1$.
\vspace{-0.25ex}
\end{theorem}

Theorem~\ref{thm:OneShotKG:OneStep:Convergence} also applies when using RQMC (Appendix~\ref{appdx:subsec:GeneralSAAresults:RQMC}), in which case we again observe improved empirical convergence rates. 
In Appendix~\ref{appdx:subsec:GeneralSAAresults:AsympOptimKG}, we prove that if $f_\textnormal{true}$ is drawn from the same GP prior as $f$ 
and $g(f) \equiv f$, then the MC-approximated KG \emph{policy} (i.e., when~\eqref{eq:OneShotKG:SAA:nested} is maximized in each period to select measurements) is \emph{asymptotically optimal} \citep{frazier2008knowledge,frazier2009knowledge,poloczek2017misokg, bect2019supermartingaleGP}, meaning that as the number of measurements tends to infinity, an optimal point $x^* \in \mathcal{X}_f^* := \argmax_{x\in \bbX} f(x)$ is identified.

Conditional on the fixed base samples, \eqref{eq:OneShotKG:SAA:nested} does not exhibit the nested structure used in the conventional formulation (which requires solving an optimization problem to get a noisy gradient estimate). Moving the maximization outside of the sample average yields the \emph{equivalent} problem\\[-3ex]
\begin{align}
\label{eq:OneShotKG:OneStep}
    \max_{\bx \in \bbX}\hat{\alpha}_{\mathrm{KG}, N}(&\bx, \calD) \equiv \max_{\bx, \, \bx'}
    \frac{1}{N} \sum_{i=1}^{N} \bbE \bigl[ g(f(x_i)) \, | \, \calD_\bx^i \bigr],
\end{align}
where $\bx' := \{x^i\}_{i=1}^N \in \bbX^{\!N}$ represent ``next stage'' solutions, or ``fantasy points.''
If $g$ is affine, the expectation in~\eqref{eq:OneShotKG:OneStep} admits an analytical expression. If not, we use another MC approximation of the form~(\ref{eq:Acquisition:MyopicMC}) with $N_I$ fixed inner based samples~$E_I$.\footnote{Convergence results can be established in the same way, and will require that $\min\{N, N_I\} \rightarrow \infty$.} 
The key difference from the envelope theorem approach \citep{wu2017bayesiangrad} is that we do not solve the inner problem to completion for every fantasy point for every gradient step w.r.t. $\bx$. Instead, we solve~\eqref{eq:OneShotKG:OneStep} jointly over $\bx$ and the fantasy points $\bx'$.
The resulting optimization problem is of higher dimension, namely $(q + N)d$ instead of $qd$, but unlike the envelope theorem formulation it can be solved as a single problem, using methods for deterministic optimization. Consequently, we dub this KG variant the ``One-Shot Knowledge Gradient'' (\OKG{}). The ability to auto-differentiate the involved quantities (including the samples $y_\calD^i(\bx)$ and $\xi_{\calD_\bx}^i(\bx)$ through the posterior updates) w.r.t. $\bx$ and $\bx'$ allows \botorch{} to solve this problem effectively. The main limitation of \OKG{} is the linear growth of the dimension of the optimization problem in $N$, which can be challenging to solve - however, in practical settings, we observe good performance for moderate $N$. 
We provide a simplified implementation of \OKG{} in the following section.


\section{Programmable Bayesian Optimization with \botorch{}}
\label{sec:ModularBO}

SAA provides an efficient and robust approach to optimizing MC acquisition functions through the use of deterministic gradient-based optimization.  In this section, we introduce \botorch{}, a complementary differentiable programming framework for Bayesian optimization research. 
Following the conceptual framework outlined in Figure~\ref{fig:Acquisition:MCAcq}, \botorch{} provides modular abstractions for representing and implementing sophisticated BO procedures. 
Operations are implemented as PyTorch modules that are highly parallelizable on modern hardware and end-to-end differentiable, which allows for efficient optimization of acquisition functions. 
Since the chain of evaluations on the sample level does not make any assumptions about the form of the posterior, \botorch{}'s primitives can be directly used with any model from which re-parameterized posterior samples can be drawn, including probabilistic programs~\citep{tran2017deep,bingham2018pyro}, Bayesian neural networks~\citep{neal2012bayesian,saatci2017bayesian, maddox2019fast,izmailov2019subspace}, and more general types of GPs~\citep{damianou2013deep}.  In this paper, we focus on an efficient and scalable implementation of GPs, GPyTorch~\citep{gardner2018gpytorch}.

\label{sec:Abstractions:ImplementationExamples}

To illustrate the core components of \botorch{}, we demonstrate how both known and novel acquisition functions can readily be implemented. For the purposes of exposition, we show a set of simplified implementations here; details and additional examples are given in Appendices~\ref{sec:appdx:ActiveLearning} and~\ref{appdx:sec:Examples}.

\subsection{Composing {\botorch} Modules for Multi-Objective Optimization}
\label{subsec:Abstractions:ImplementationExamples:ParEGO}

In our first example, we consider $q$ParEGO \citep{daulton2020ehvi}, a variant of ParEGO~\citep{knowles2006parego}, a method for multi-objective optimziation.\\[-2ex]

\begin{figure}[h]
\begin{python}[
caption={Multi-objective optimization via augmented Chebyshev scalarizations.},
label=codex:moo,
lineskip=-2ex,
]
 weights = torch.distributions.Dirichlet(torch.ones(num_objectives)).sample()
 scalarized_objective = GenericMCObjective(
   lambda Y: 0.05 * (weights * Y).sum(dim=-1) + (weights * Y).min(dim=-1).values
 )
 qParEGO = qExpectedImprovement(model=model, objective=scalarized_objective)
 candidates, values = optimize_acqf(qParEGO, bounds=box_bounds, q=1)
\end{python}
\vspace{-5pt}
\end{figure}

Code Example~\ref{codex:moo} implements the inner loop of $q$ParEGO. We begin by instantiating a {\ttm{GenericMCObjective}} module that defines an augmented Chebyshev scalarization. This is an instance of \botorch's abstract {\ttm{MCObjective}}, which applies a transformation $g( \cdot )$ to samples $\xi$ from a posterior in its {{\ttm forward($\xi$)}} pass.  
In line 5, we instantiate an {{\ttm MCAcquisitionFunction}} module, in this case,  {{\ttm qExpectedImprovement}}, parallel EI. Acquisition functions combine a model and the objective into a single module that assigns a utility $\alpha(\bx)$ to a candidate set~$\bx$ in its {{\ttm forward}} pass.  Models can be any PyTorch module implementing a probabilistic model conforming to \botorch{}'s basic {{\ttm Model}} API. 
Finally, candidate points are selected by optimizing the acquisition function, through the use of the {{\ttm optimize\_acqf()}} utility function, which finds the candidates $\bx^* \in \argmax_\bx \alpha(\bx)$. Auto-differentiation makes it straightforward to use gradient-based optimization even for complex acquisition functions and objectives. Our SAA approach permits the use of deterministic higher-order optimization to efficiently and reliably find $\bx^*$. 

In \citep{astudillo2019composite} it is shown how performing operations on independently modeled objectives yields better optimization performance when compared to modeling combined outcomes directly (e.g., for the case of calibrating the outputs of a simulator). {\ttm{MCObjective}} is a powerful abstraction that makes this straightforward. It can also be used to implement  unknown (i.e. modeled) outcome constraints:  {\botorch} implements a {\ttm{ConstrainedMCObjective}} to compute a feasibility-weighted objective using a sample-level differentiable relaxation of the feasibility~\citep{schonlau1998globloc,gardner2014constrained,gelbart2014unknowncon,letham2019noisyei}.

\subsection{Implementing Parallel, Asynchronous Noisy Expected Improvement}
\label{subsec:Abstractions:ImplementationExamples:NEI}

Noisy EI (NEI) \citep{letham2019noisyei} is an extension of EI that is well-suited to highly noisy settings, such as A/B tests. 
Here, we describe a novel \emph{full MC} formulation of NEI that extends the original one from~\citep{letham2019noisyei} to joint parallel optimization and generic objectives. 
Letting $(\xi, \xi_\mathrm{obs}) \sim f_\calD( (\bx, \bx_\mathrm{obs}))$, our implementation avoids the need to characterize the (uncertain) best observed function value 
explicitly by averaging improvements on samples from the joint posterior over new and previously evaluated points:\\[-4ex]

\begin{align}
\label{eq:Abstractions:ImplementationExamples:NEI:fullMC}
    \text{qNEI}(\bx; \calD) = \bbE\bigl[ \bigl( \max g(\xi) - \max g(\xi_\mathrm{obs}) \bigr)_+\! \mid \calD \bigr].
\end{align}

Code Example~\ref{codex:Abstractions:ImplementationExamples:NEI:qNEI} provides an implementation of qNEI  as formulated in~\eqref{eq:Abstractions:ImplementationExamples:NEI:fullMC}. New MC acquisition functions are defined by extending an {{\ttm MCAcquisitionFunction}} base class and defining a {{\ttm forward}} pass that compute the utility of a candidate $\bx$. In the constructor (not shown), the programmer sets {\ttm{X\_baseline}} to an appropriate subset of the points at which the function was observed.

\begin{figure}[h!]
\begin{python}[
caption=Parallel Noisy EI,
label=codex:Abstractions:ImplementationExamples:NEI:qNEI,
lineskip=-2ex,
]
class qNoisyExpectedImprovement(MCAcquisitionFunction):
  @concatenate_pending_points
  def forward(self, X: Tensor) -> Tensor:
    q = X.shape[-2]
    X_full = torch.cat([X, match_shape(self.X_baseline, X)], dim=-2)
    posterior = self.model.posterior(X_full)
    samples = self.sampler(posterior)
    obj = self.objective(samples)
    obj_new = obj[...,:q].max(dim=-1).value
    obj_prev = obj[...,q:].max(dim=-1).value
    improvement = (obj_new - obj_prev).clamp_min(0)
    return improvement.mean(dim=0).value
\end{python}
\end{figure}

Like all MC acquisition functions, qNEI can be extended to support \emph{asynchronous} candidate generation, in which a set $\tilde{\bx}$ of \emph{pending points} have been submitted for evaluation, but have not yet completed.
This is done by concatenating pending points into~$\bx$ with the {\ttm{@concatenate\_pending\_points}} decorator.
This allows us to compute the joint 
utility $\alpha(\bx \cup \tilde{\bx}; \Phi, \calD)$ of all points, pending and new, but optimize only with respect to the new~$\bx$.
This strategy also provides a natural way of generating parallel BO candidates using \emph{sequential greedy} optimization~\citep{snoek2012practical}: We generate a single candidate, add it to the set of pending points, and proceed to the next. Due to submodularity of many common classes of acquisition functions (e.g., EI, UCB) ~\citep{wilson2018maxbo}, this approach can often yield better optimization performance compared to optimizing all candidate locations simultaneously  (see  Appendix~\ref{appdx:subsec:ImplementatonDetails:SeqGreedy}).  

With the observed, pending, and candidate points ({\ttm{X\_full}}) in hand, we use the 
 {\ttm{Model}}'s {\ttm{posterior()}} method to generate an object that represents the joint posterior across all points. The {\ttm{Posterior}} returned by {\ttm{posterior($\bx$)}} represents $f_\calD(\bx)$ (or $y_\calD(\bx)$, if the {\ttm{observation\_noise}} keyword argument is set to {{\ttm True}}), 
and may be be explicit (e.g. a multivariate normal 
in the case of GPs), or implicit (e.g. a container for a warmed-up MCMC chain).  
Next, samples are drawn from the posterior distribution {\ttm{p}} via a {\ttm{MCSampler}}, which employs the reparameterization trick \citep{kingma2013reparam,rezende2014stochbackprop}. 
Given base samples $E \in \bbR^{N_{\!s} \times qm}$, a {\ttm{Posterior}} object produces $N_{\!s}$ samples $\xi_\calD \in \bbR^{N_{\!s}\times q \times m}$ from the joint posterior.  
Its {\ttm{forward(p)}} pass draws samples $\xi^i_\calD$ from {\ttm{p}} by automatically constructing base samples $E$. 
By default, \botorch{} uses RQMC via scrambled Sobol sequences~\citep{owen2003quasi}. Finally, these samples are mapped through the objective, and the expected improvement between the candidate point x and observed/pending points is computed by marginalizing the improvements on the sample level.

\subsection{Look-ahead Bayesian Optimization with One-Shot KG}

Code Example~\ref{codex:Abstractions:ImplementationExamples:OKG} 
shows a simplified \OKG{} implementation, as discussed in Section \ref{sec:OneShotKG}.

\begin{figure}[h]
\begin{python}[
caption=Implementation of One-Shot KG,
label=codex:Abstractions:ImplementationExamples:OKG,
]
class qKnowledgeGradient(OneShotAcquisitionFunction):
  def forward(self, X: Tensor) -> Tensor:
    X, X_f = torch.split(X, [X.size(-2) - self.N, self.N], dim=-2)
    fant_model = self.model.fantasize(X=X, sampler=self.sampler, observation_noise=True)
    inner_acqf = SimpleRegret(
      fant_model, sampler=self.inner_sampler, objective=self.objective,
    )
    with settings.propagate_grads(True):
      return inner_acqf(X_f).mean(dim=0).value
\end{python}
\vspace{-2.5ex}
\end{figure}

Here, the input {\ttm{X}} to {\ttm{forward}} is a concatenation of $\bx$ and $N$ \emph{fantasy points} $\bx'$ (this setup ensures that \OKG{} can be optimized using the same APIs as all other acquisition functions). 
After {\ttm{X}} is split into its components, we utilize the {\ttm{Model}}'s {\ttm{fantasize($\bx$, sampler)}} method that, given~$\bx$ and a {\ttm{MCSampler}}, constructs a batched set of $N$ \emph{fantasy models} $\{f^i\}_{i=1}^N$ such that $f^i_\calD(\bx) \overset{d}{=} f_{\calD_\bx^i}(\bx), \forall\, \bx \in \bbX^q$, where $\calD_\bx^i := \calD \cup \{\bx, y_\calD^i(\bx)\}$ is the original dataset augmented by a fantasy observation at $\bx$. The fantasy models provide a distribution over functions conditioned on future observations at $\bx$, which is used here to implement one-step look-ahead.
{\ttm{SimpleRegret}} computes $\bbE \bigl[ g(f(x_i)) \, | \, \calD_\bx^i \bigr]$ from~\eqref{eq:OneShotKG:OneStep} for each~$i$ in batch mode. The {\ttm{propagate\_grads}} context enables auto-differentiation through both the \emph{generation} of the fantasy models and the \emph{evaluation} of their respective posteriors at the points $\bx'$.

\section{Experiments}
\label{sec:Experiments}


\subsection{Exploiting Parallelism and Hardware Acceleration}
\label{subsec:Abstractions:ParallelHardware}

\botorch{} utilizes inference and optimization methods designed to exploit parallelization via batched computation, and integrates closely with GPyTorch~\citep{gardner2018gpytorch}. These model have fast test-time (predictive) distributions and sampling. This is crucial for BO, where the same models are evaluated many times in order to optimize the acquisition function. GPyTorch makes use of structure-exploiting algebra and local interpolation for $\mathcal O(1)$ computations in querying the predictive distribution, and $\mathcal O(T)$ for drawing a posterior sample at $T$ points, compared to the standard $\mathcal O(n^2)$ and $\mathcal O(T^3n^3)$ computations~\citep{pleiss2018love}.

 Figure~\ref{appdx:fig:ParallelHardware:BatchEval:Timings} reports wall times for \emph{batch evaluation} of {\ttm{qExpectedImprovement}} at multiple candidate sets $\{\bx^i\}_{i=1}^b$ for different MC samples sizes $N$, on both CPU and GPU for a GPyTorch GP. We observe significant speedups from running on the GPU, with scaling essentially linear in the batch size~$b$, except for very large $b$ and $N$. 
Figure~\ref{appdx:fig:ParallelHardware:FPV:FPV} shows between 10--40X speedups when using fast predictive covariance estimates over standard posterior inference in the same setting.
 The speedups grow slower on the GPU, whose cores do not saturate as quickly as on the CPU when doing standard posterior inference
(for additional details see Appendix~\ref{appdx:sec:ParallelHardware}).
Together, batch evaluation and fast predictive distributions enable efficient, parallelized acquisition function evaluation for a very large number (tens of thousands) of points.
This scalability allows us to implement and exploit novel highly parallelized initialization and optimization techniques.

\begin{figure*}[ht]
    \begin{minipage}[c]{0.5\columnwidth}
    \centering
    \includegraphics[width=\columnwidth]{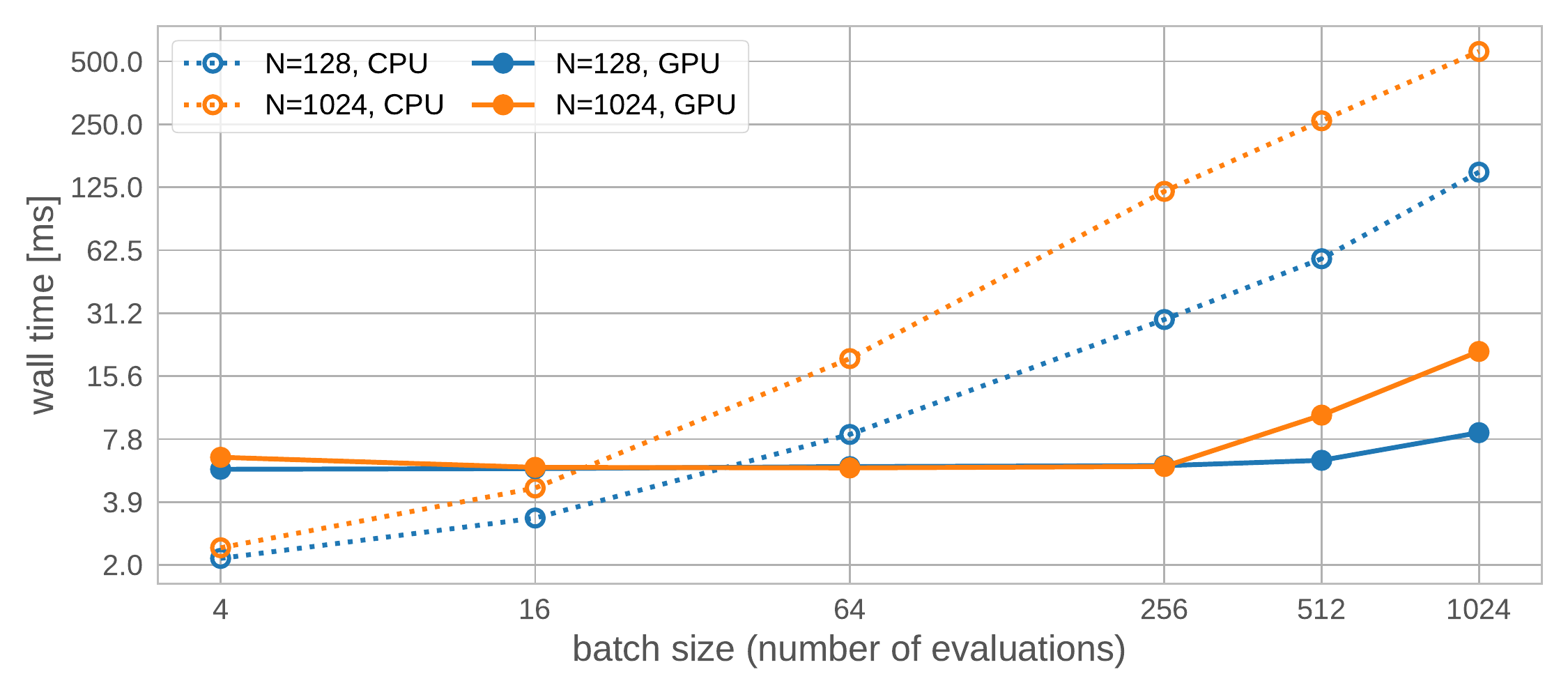}\\[-2ex]
    \caption{Wall times for batched evaluation of qEI}
    \label{appdx:fig:ParallelHardware:BatchEval:Timings}
    \end{minipage}
    \hfill
    \begin{minipage}[c]{0.5\columnwidth}
    \centering
    \includegraphics[width=\columnwidth]{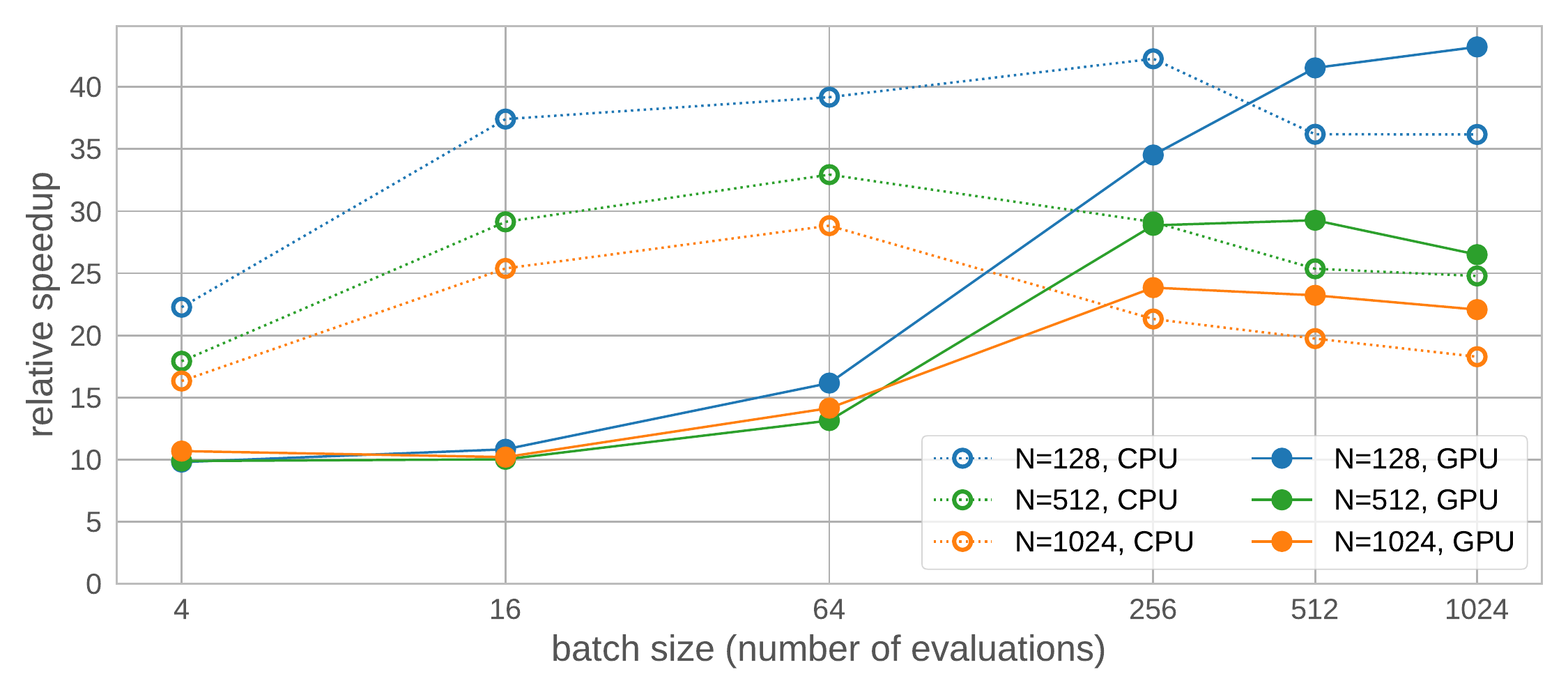}\\[-2ex]
    \caption{Fast predictive distributions speedups}
    \label{appdx:fig:ParallelHardware:FPV:FPV}
\end{minipage}
\vspace{-3ex}
\end{figure*}

\subsection{Bayesian Optimization Performance Comparisons}

We compare (i) the empirical performance of standard algorithms implemented in \botorch{} with those from other popular BO libraries, and (ii) our novel acquisition function, \OKG{}, against other acquisition functions, both within \botorch{} and in other packages.
We isolate three key frameworks---GPyOpt, Cornell MOE (\emph{MOE EI},  \emph{MOE KG}), and Dragonfly---because they are the most popular libraries with ongoing support\footnote{We were unable to install GPFlowOpt due to its incompatibility with current versions of GPFlow/TensorFlow.} and are most closely related to \botorch{} in terms of state-of-the-art acquisition functions. 
GPyOpt uses an extension of EI with a local penalization heuristic (henceforth \emph{GPyOpt LP-EI}) for parallel optimization \citep{gonzalez16b}.
For Dragonfly, we consider its default ensemble heuristic (henceforth \emph{Dragonfly GP Bandit}) \citep{kandasamy2019dragonfly}.

\begin{figure*}[ht]
    \begin{minipage}[c]{0.5\columnwidth}
    \centering
    \includegraphics[width=\columnwidth]{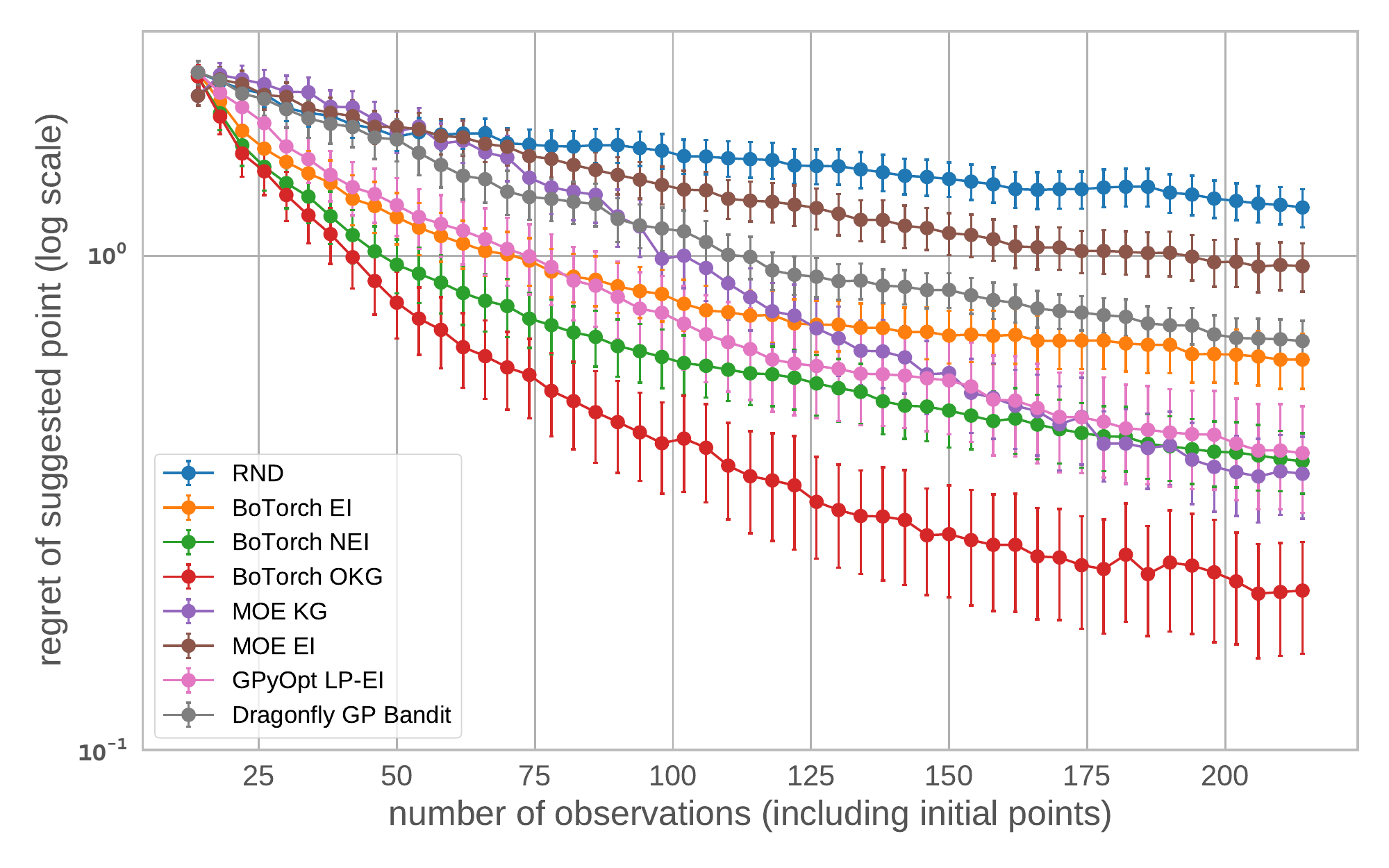}\\[-2ex]
    \caption{Hartmann ($d=6$), noisy, best suggested 
    }
    \label{fig:Experiments:Synthetic:Unconstrained:Hartmann}
    \end{minipage}
    \hfill
    \begin{minipage}[c]{0.5\columnwidth}
    \centering
    \includegraphics[width=\columnwidth]{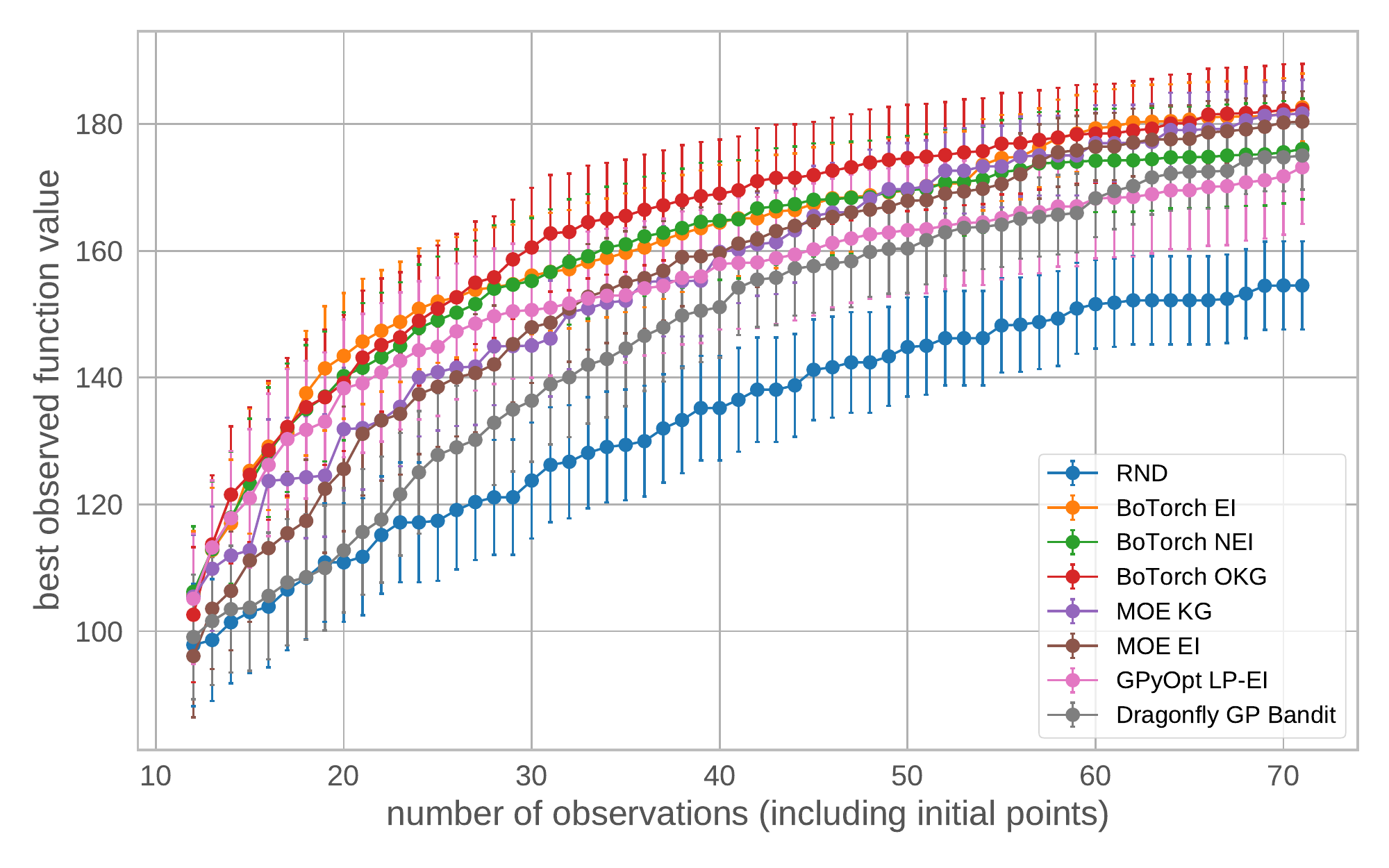}\\[-2ex]
    \caption{ DQN tuning benchmark (Cartpole)}
    \label{fig:Experiments:Hyper:Cartpole}
\end{minipage}
\vspace{-1ex}
\end{figure*}

Our results provide three main takeaways. First, we find that \botorch{}'s algorithms tend to
achieve greater sample efficiency compared to those of other packages (all packages use their default models and settings).  Second, we find that \OKG{} often outperforms all other acquisition functions. Finally, \OKG{} is more computationally scalable than MOE KG (the gold-standard implementation of KG), showing significant reductions in wall time (up to 6X, see Appendix~\ref{appdx:subsec:AddEmpirical:OKGWallTime}) while simultaneously achieving improved optimization performance (Figure \ref{fig:Experiments:Synthetic:Unconstrained:Hartmann}).

\textbf{Synthetic Test Functions:} 
We consider BO for parallel optimization of $q=4$ design points, on four noisy synthetic functions used in \citet{wang2016parallel}: Branin, Rosenbrock, Ackley, and Hartmann. 
Figure~\ref{fig:Experiments:Synthetic:Unconstrained:Hartmann} 
reports means and 95\% confidence intervals over 100 trials for Hartmann; 
results for the other functions are qualitatively similar and are provided in Appendix~\ref{appdx:subsec:AddEmpirical:Synthetic}, together with details on the evaluation.
Results for constrained BO using a differentiable relaxation of the feasibility indicator on the sample level are provided in Appendix~\ref{appdx:subsec:AddEmpirical:Constrained}.

\textbf{Hyperparameter Optimization:}
We illustrate the performance of \botorch{} on real-world applications, represented by three hyperparameter optimization (HPO) experiments:
\textbf{(1)} Tuning 5 parameters of a deep Q-network (DQN) learning algorithm \citep{mnih2013playing,mnih2015human} on the \emph{Cartpole} task from OpenAI gym \citep{brockman2016openai} and the default DQN agent implemented in Horizon \citep{gauci2018horizon}, Figure~\ref{fig:Experiments:Hyper:Cartpole};
\textbf{(2)} Tuning 6 parameters of a neural network surrogate model for the UCI Adult data set \citep{kohavi2019UCI} introduced by \citet{falkner2018robusteff}, available as part of HPOlib2 \citep{eggensperger2019hpolib2},  Figure~\ref{fig:Experiments:SurrogateParamNet} in Appendix~\ref{appdx:subsec:AddEmpirical:hyper};
\textbf{(3)} Tuning 3 parameters of the recently proposed \emph{Stochastic Weight Averaging} (SWA) procedure of \citet{izmailov2018averaging} on the VGG-16 \citep{simonyan2014very} architecture for CIFAR-10, which achieves superior accuracy compared to previously reported results. 
A more detailed description of these experiments is given in Appendix \ref{appdx:subsec:AddEmpirical:hyper}.

\vspace{-5pt}
\section{Discussion and Outlook}
\label{sec:Discussion}

We presented a novel strategy for effectively optimizing MC acquisition functions using SAA, and established strong theoretical convergence guarantees (in fact, our RQMC convergence results are novel more generally, and of independent interest). Our proposed \OKG{} method, an extension of this approach to ``one-shot'' optimization of look-ahead acquisition functions, constitutes a significant development of KG, improving scalability and allowing for generic composite objectives and outcome constraints. This approach can naturally be extended to multi-step and other look-ahead approaches~\citep{jiang2020multistep}.

We make these methodological and theoretical contributions available in our open-source library \botorch{} (\url{https://botorch.org}), a modern programming framework for BO that features a modular design and flexible API, our distinct SAA approach, and algorithms specifically designed to exploit modern computing paradigms such as parallelization and auto-differentiation. \botorch{} is particularly valuable in helping researchers to rapidly assemble novel BO techniques. Specifically, the basic MC acquisition function abstraction provides generic support for batch optimization, asynchronous evaluation, RQMC integration, and composite objectives (including outcome constraints).

Our empirical results show that besides increased flexibility, our advancements in both methodology and computational efficiency translate into significantly faster and more accurate closed-loop optimization performance on a range of standard problems. While other settings such as high-dimensional~ \citep{kandasamy15, wang2016rembo, letham2020alebo}, multi-fidelity~\citep{poloczek2017misokg,wu2019cfkg}, or multi-objective~\citep{knowles2006parego, paria2018randscalar, daulton2020ehvi} BO, and non-MC acquisition functions such as Max-Value Entropy Search~\citep{wang2017mves}, are outside the scope of this paper, these approaches can readily be realized in \botorch{} and are included in the open-source software package.
One can also naturally generalize BO procedures to incorporate neural architectures in \botorch{} using standard PyTorch models. In particular, deep kernel architectures \citep{wilson2016deep}, deep Gaussian processes \citep{damianou2013deep, salimbeni2017dsvi}, and variational auto-encoders~\citep{gomez-bombarelli2018chemical,moriconi2019manifold} can easily be incorporated into \botorch{}'s primitives, and can be used for more expressive kernels in high-dimensions.

In summary, \botorch{} provides the research community with a robust and extensible basis for implementing new ideas and algorithms in a modern computational paradigm, theoretically backed by our novel SAA convergence results.

\clearpage
\subsection*{Broader Impact}

Bayesian optimization is a generic methodology for optimizing black-box functions, and therefore, by its very nature, not tied to any particular application domain. As mentioned earlier in the paper, Bayesian optimization has been used for various arguably good causes, including drug discovery or reducing the energy footprint of ML applications by reducing the computational cost of tuning hyperparameters. In the Appendix, we give an specific example for how our work can be applied in a public health context, namely to efficiently distribute survey locations for estimating malaria prevalence. \botorch{} as a tool specifically has been used in various applications, including transfer learning for neural networks~\citep{maddox2019transfer}, high-dimensional Bayesian optimization~\citep{letham2020alebo}, drug discovery~\citep{boitreaud2020optimol}, sim-to-real transfer~\citep{muratore2020sim2real}, trajectory optimization~\citep{jain2020racingline}, and nano-material design~\citep{nguyen2019nano}.
However, there is nothing inherent to this work and Bayesian optimization as a field more broadly that would preclude it from being abused in some way, as is the case with any general methodology.

\begin{ack}
We wish to thank Art Owen for insightful conversations on quasi-Monte-Carlo methods. We also express our appreciation to Peter Frazier and Javier Gonzalez for their helpful feedback on earlier versions of this paper.

Andrew Gordon Wilson is supported by NSF I-DISRE 193471, NIH R01 DA048764-01A1, NSF IIS-1910266, and NSF 1922658 NRT-HDR: FUTURE Foundations, Translation, and Responsibility for Data Science.
\end{ack}

{\small
\bibliographystyle{plainnat}
\bibliography{botorch}
}

\clearpage

\appendix

\begin{center}
\hrule height 4pt
\vskip 0.25in
\vskip -\parskip
    {\LARGE\bf  Appendix to:\\[2ex] \botorch: A Framework for Efficient Monte-Carlo\\[1ex] Bayesian Optimization}
\vskip 0.29in
\vskip -\parskip
\hrule height 1pt
\vskip 0.2in%
\end{center}

\section{Brief Overview of Other Software Packages for BO}
\label{appdx:sec:PackageComparison}

One of the earliest commonly-used packages is \textbf{Spearmint} \citep{snoek2012practical}, which implements a variety of modeling techniques such as MCMC hyperparameter sampling and input warping \citep{snoek14}. Spearmint also supports parallel optimization via fantasies, and constrained optimization with the expected improvement and predictive entropy search acquisition functions \citep{gelbart2014unknowncon,hernandez15pesconstraints}. Spearmint was among the first libraries to make BO easily accessible to the end user.

\textbf{GPyOpt} \citep{gpyopt2016} builds on the popular GP regression framework GPy \citep{gpy2014}. It supports a similar set of features as Spearmint, along with a local penalization-based approach for parallel optimization \citep{gonzalez16b}. It also provides the ability to customize different components through an alternative, more modular API.

\textbf{Cornell-MOE}~\citep{wu2016parallel} implements the Knowledge Gradient (KG) acquisition function, which allows for parallel optimization, and includes recent advances such as large-scale models incorporating gradient evaluations \citep{wu2017bayesiangrad} and multi-fidelity optimization \citep{wu2019cfkg}. Its core is implemented in C++, which provides performance benefits but renders it hard to modify and extend.

\textbf{RoBO} \citep{klein17robo} implements a collection of models and acquisition functions, including Bayesian neural nets \citep{springenberg-nips2016} and multi-fidelity optimization \citep{klein-corr16}.

\textbf{Emukit} \citep{emukit2018} is a Bayesian optimization and active learning toolkit with a collection of acquisition functions, including for parallel and multi-fidelity optimization. It does not provide specific abstractions for implementing new algorithms, but rather specifies a model API that allows it to be used with the other toolkit components. 

The recent \textbf{Dragonfly} \citep{kandasamy2019dragonfly} library supports parallel optimization, multi-fidelity optimization \citep{kandasamy16}, and high-dimensional optimization with additive kernels \citep{kandasamy15}. It takes an ensemble approach and aims to work out-of-the-box across a wide range of problems, a design choice that makes it relatively hard to extend.


\section{Parallelism and Hardware Acceleration}
\label{appdx:sec:ParallelHardware}

\subsection{Batch Evaluation}
\label{appdx:subsec:ParallelHardware:BatchEval}

Batch evaluation, an important element of modern computing, 
enables automatic dispatch of independent operations across multiple computational resources (e.g. CPU and GPU cores) for parallelization and memory sharing.
All \botorch{} components support batch evaluation,
which makes it easy to write concise and highly efficient code in a platform-agnostic fashion.
Batch evaluation enables fast queries of acquisition functions at a large number of candidate sets in parallel,
facilitating novel initialization heuristics and optimization techniques.

Specifically, instead of sequentially evaluating an acquisition function at a number of candidate sets $\bx_1,\dotsc, \bx_b$,
where $\bx_k \in \bbR^{q \times d}$ for each~$k$, \botorch{} evaluates a batched tensor
$\bx \in \bbR^{b \times q \times d}$. Computation is automatically distributed so that, depending on the hardware used, speedups can be close to linear in the batch size~$b$.
Batch evaluation is also heavily used in computing MC acquisition functions, with the effect that significantly increasing the number of MC samples 
often has little impact on wall time. 
In Figure~\ref{appdx:fig:ParallelHardware:BatchEval:Timings} we observe significant speedups from running on the GPU, with scaling essentially linear in the batch size, except for very large $b$ and $N$. The fixed cost due to communication overhead renders CPU evaluation faster for small batch and sample sizes.

\section{Additional Empirical Results}
\label{appdx:sec:AddEmpirical}

This section describes a number of empirical results that were omitted from the main paper due to space constraints.

\subsection{Synthetic Functions}
\label{appdx:subsec:AddEmpirical:Synthetic}

Algorithms start from the same set of $2d+2$ QMC sampled initial points for each trial, with $d$ the dimension of the design space. We evaluate based on the true noiseless function value at the ``suggested point'' (i.e., the point to be chosen \emph{if BO were to end at this batch}). 
\OKG{}, MOE KG, and NEI use ``out-of-sample'' suggestions (introduced as $\chi_n$ in Section~\ref{appdx:subsec:GeneralSAAresults:AsympOptimKG}),  
while the others use ``in-sample'' suggestions \citep{frazier2018tutorial}.

All functions are evaluated with noise generated from a $\mathcal N(0,.25)$ distribution. Figures~\ref{fig:appdx:AddEmpirical:Synthetic:Unconstrained:Branin}-\ref{fig:appdx:AddEmpirical:Synthetic:Unconstrained:Ackley} give the results for all synthetic functions from Section~\ref{sec:Experiments}. The results show that \botorch{}'s NEI and \OKG{} acquisition functions provide highly competitive performance in all cases.

\begin{figure*}
    \begin{minipage}[l]{0.5\columnwidth}
    \centering
    \includegraphics[width=\columnwidth]{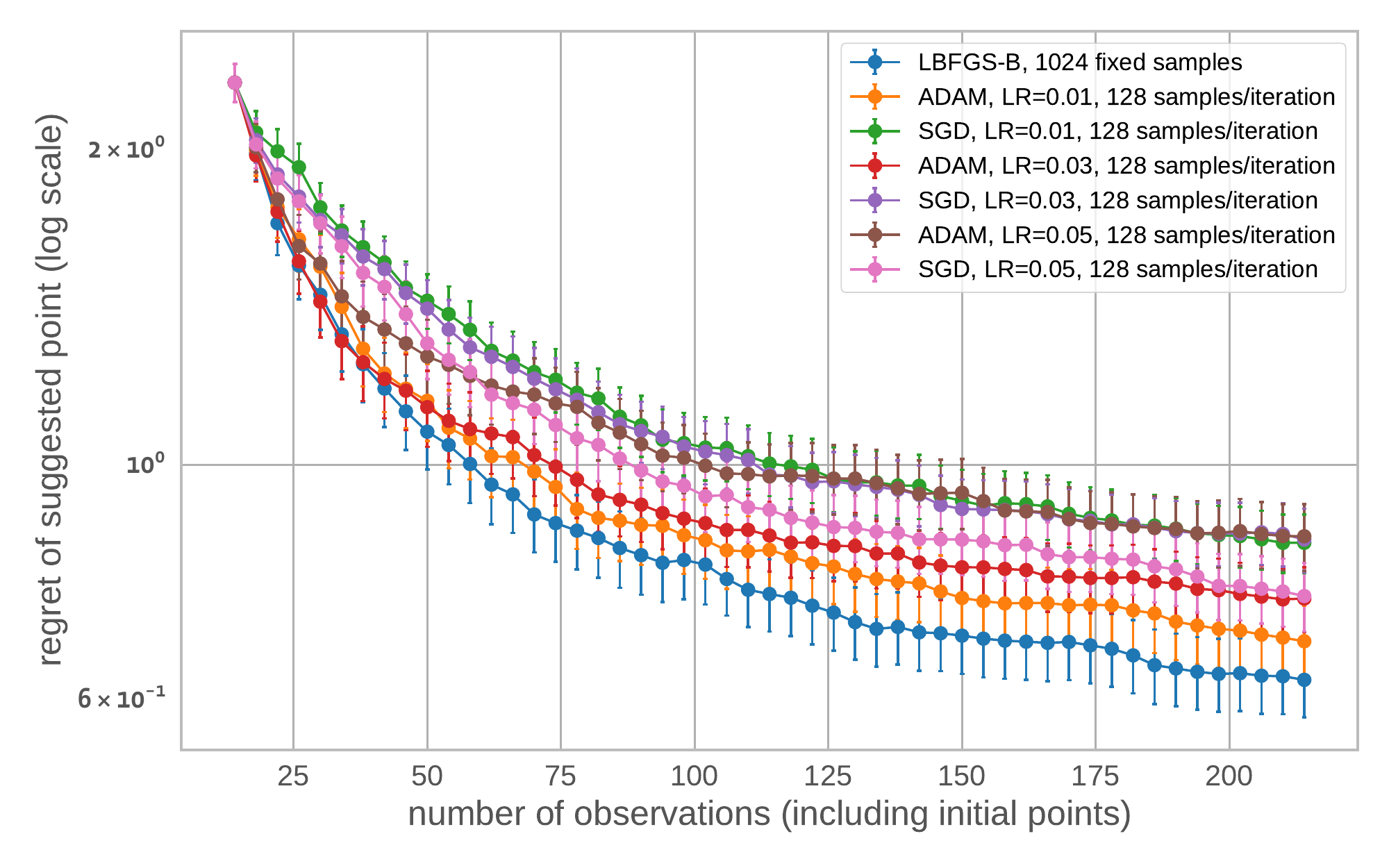}
    \caption{Stochastic/deterministic opt. of EI on Hartmann6}
    \label{fig:appdx:NonStochOpt:StochasticLRDep}
    \end{minipage}
    \hfill
    \begin{minipage}[l]{0.5\columnwidth}
    \centering
    \includegraphics[width=\columnwidth]{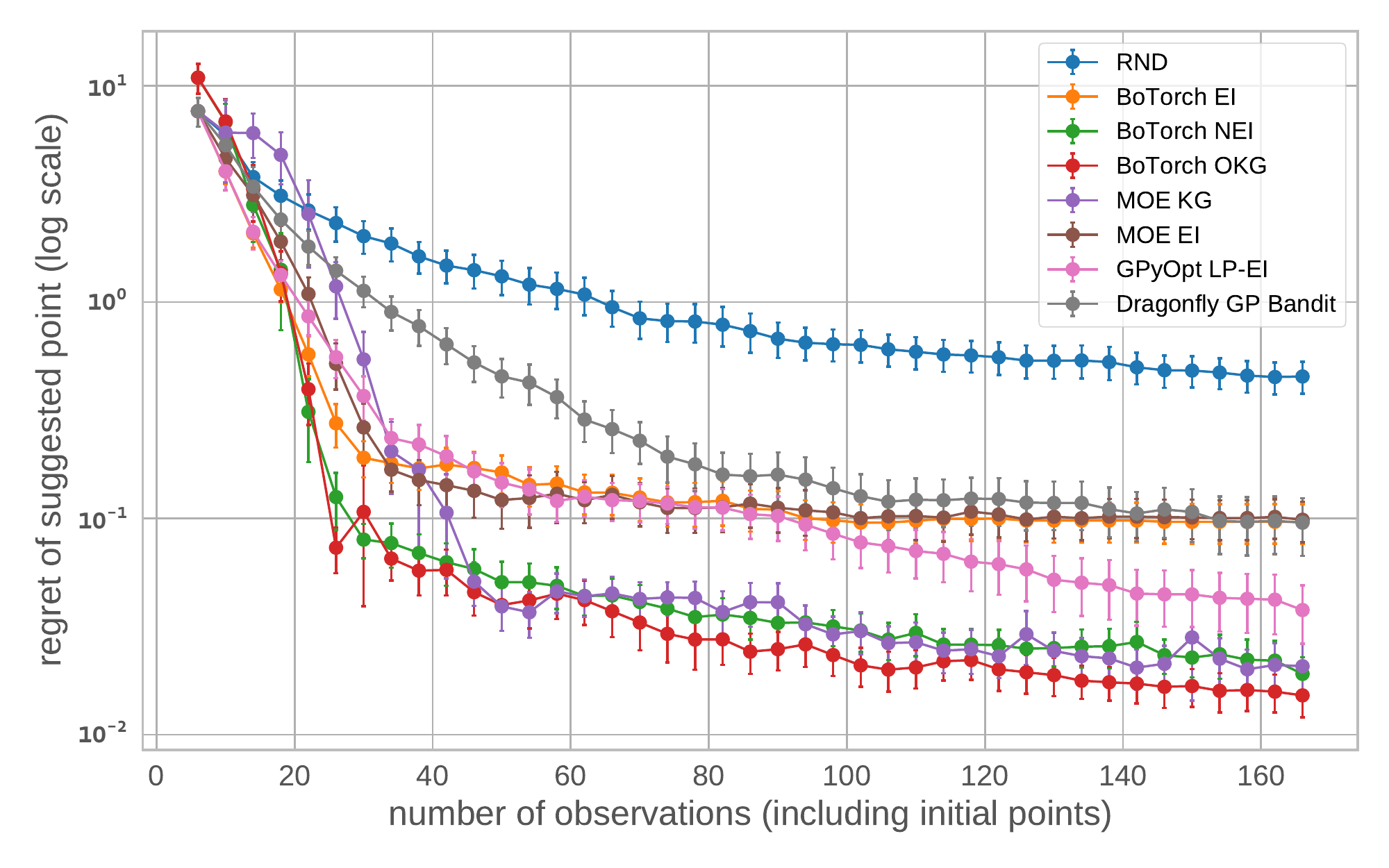}
    \caption{Branin ($d=2$)}
    \label{fig:appdx:AddEmpirical:Synthetic:Unconstrained:Branin}
\end{minipage}
\end{figure*}

\begin{figure*}
    \begin{minipage}[l]{0.5\columnwidth}
    \centering
    \includegraphics[width=\columnwidth]{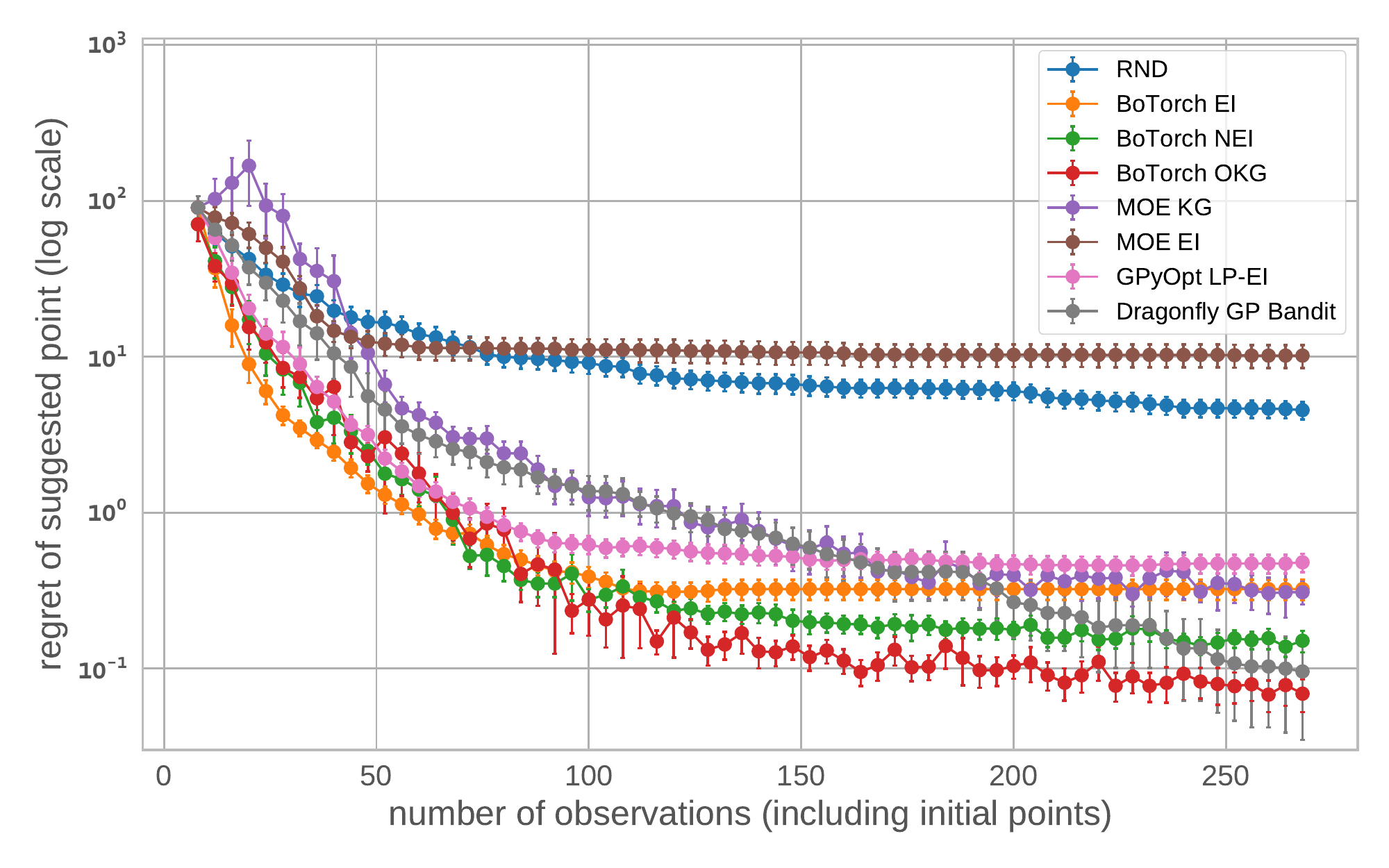}\\[-3ex]
    \caption{Rosenbrock ($d=3$)}
    \label{fig:appdx:AddEmpirical:Synthetic:Unconstrained:Rosenbrock}
    \end{minipage}
    \hfill
    \begin{minipage}[l]{0.5\columnwidth}
    \centering
    \includegraphics[width=\columnwidth]{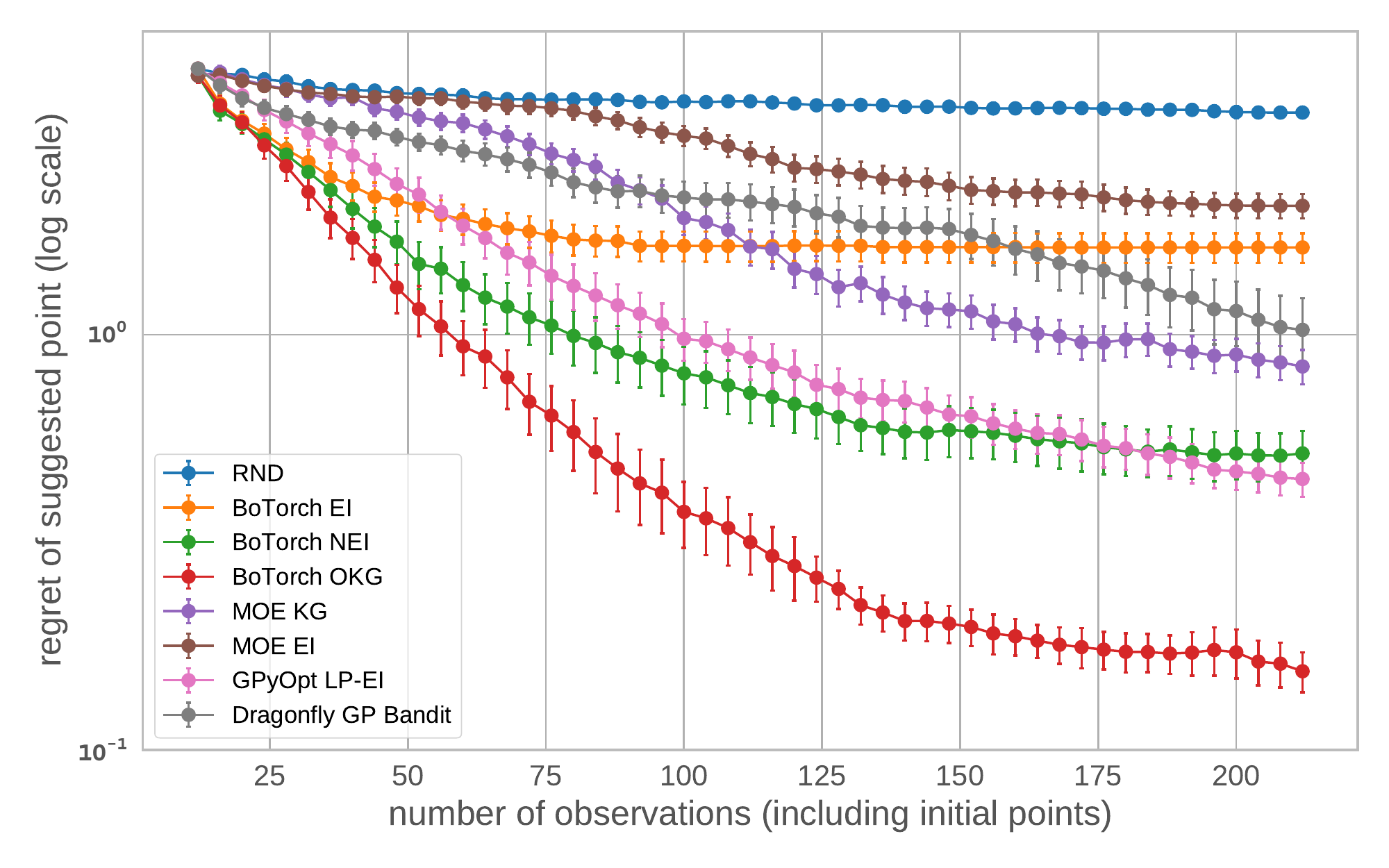}\\[-3ex]
    \caption{Ackley ($d=5$)}
    \label{fig:appdx:AddEmpirical:Synthetic:Unconstrained:Ackley}
\end{minipage}
\end{figure*}

\subsection{One-Shot KG Computational Scaling}
\label{appdx:subsec:AddEmpirical:OKGWallTime}

Figure~\ref{fig:appdx:AddEmpirical:OKGWallTime:WallTimes} shows the wall time for generating a set of $q=8$ candidates as a function of the number of total data points~$n$ for both standard (Cholesky-based) as well as scalable (Linear CG) posterior inference methods, on both CPU and GPU. While the GPU variants have a significant overhead for small models, they are significantly faster for larger models. Notably, our SAA based \OKG{} is significantly faster than \emph{MOE KG}, while at the same time achieving much better optimization performance (Figure~\ref{fig:appdx:AddEmpirical:OKGWallTime:Hartmann}). 

\begin{figure*}
    \begin{minipage}[l]{0.5\columnwidth}
    \centering
    \includegraphics[width=\columnwidth]{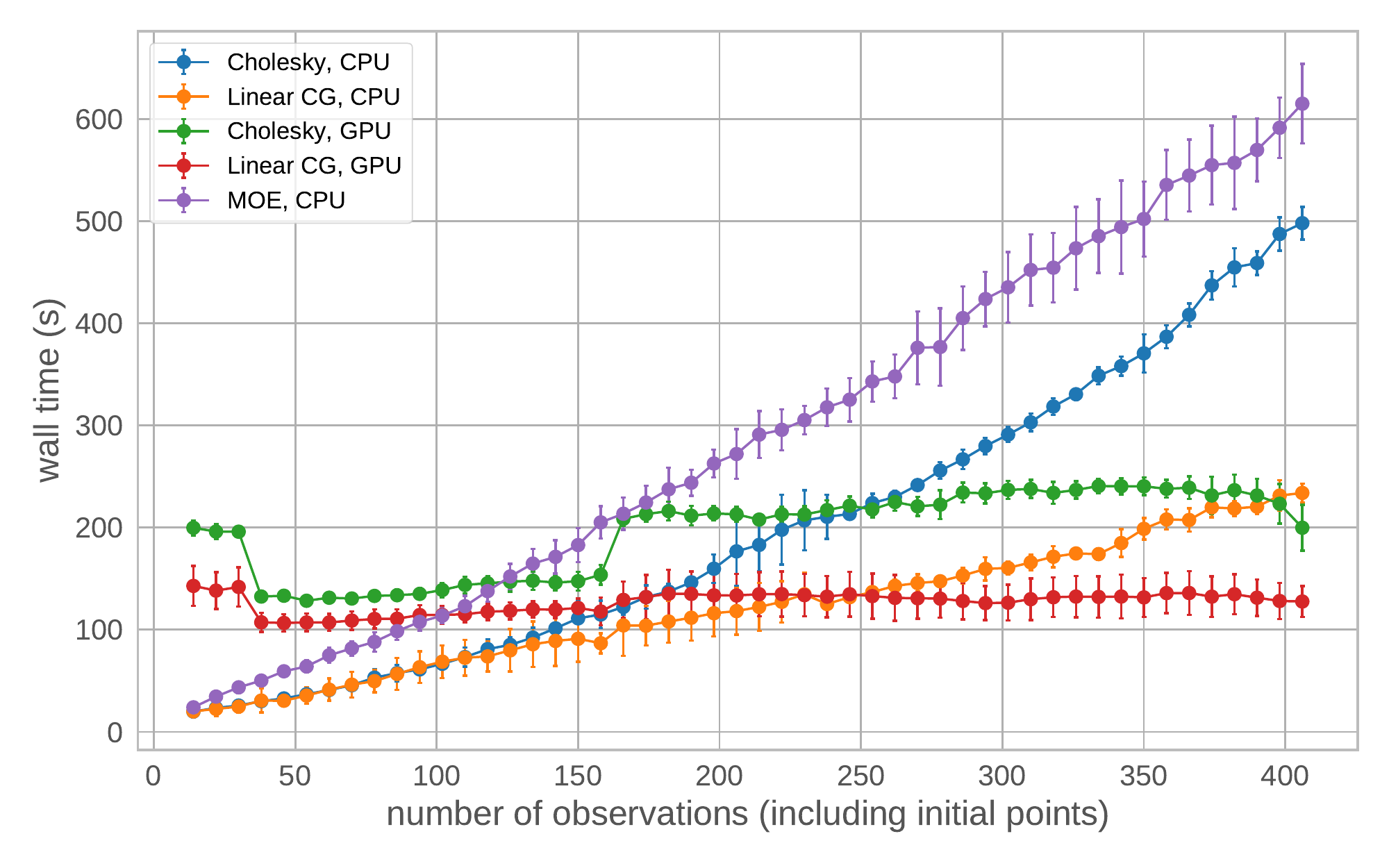}\\[-3ex]
    \caption{KG wall times}
    \label{fig:appdx:AddEmpirical:OKGWallTime:WallTimes}
    \end{minipage}
    \hfill
    \begin{minipage}[l]{0.5\columnwidth}
    \centering
        \includegraphics[width=\columnwidth]{plots/synthetic/Hartmann_unconstrained.pdf}\\[-3ex]
    \caption{Hartmann ($d=6$), noisy, best suggested}
    \label{fig:appdx:AddEmpirical:OKGWallTime:Hartmann}
    \end{minipage}
\end{figure*}

\subsection{Constrained Bayesian Optimization}
\label{appdx:subsec:AddEmpirical:Constrained}

We present results for constrained BO on a synthetic function. We consider a multi-output function $f = (f_1, f_2)$ and the optimization problem:
\begin{align}
\label{eq:appdx:subsec:AddEmpirical:Constrained:Problem}
    \max_{x \in \bbX} \; f_1(x) \quad  \text{s.t.} \quad  f_2(x) \le 0.
\end{align}
Both $f_1$ and $f_2$ are observed with $\mathcal N(0,0.5^2)$ noise and we model the two components using independent GP models. A constraint-weighted composite objective is used in each of the \botorch{} acquisition functions EI, NEI, and \OKG{}. 

Results for the case of a Hartmann6 objective and two types of constraints are given in Figures \ref{fig:appdx:AddEmpirical:Constrained:HartmannL1}-\ref{fig:appdx:AddEmpirical:Constrained:HartmannL2} (we only show results for \botorch{}'s algorithms, since the other packages do not natively support optimization subject to unknown constraints). 

The regret values are computed using a feasibility-weighted objective, where ``infeasible'' is assigned an objective value of zero. For random search and EI, the suggested point is taken to be the best feasible noisily observed point, and for NEI and OKG, we use out-of-sample suggestions by optimizing the feasibility-weighted version of the posterior mean. The results displayed in Figure \ref{fig:appdx:AddEmpirical:Constrained:HartmannL2} are for the constrained Hartmann6 benchmark from \cite{letham2019noisyei}. Note, however, that the results here are not directly comparable to the figures in \cite{letham2019noisyei} because (1) we use feasibility-weighted objectives to compute regret and (2) they follow a different convention for suggested points. We emphasize that our contribution of outcome constraints for the case of KG has not been shown before in the literature.

\begin{figure*}
    \begin{minipage}[l]{0.5\columnwidth}
    \centering
    \includegraphics[width=1\linewidth]{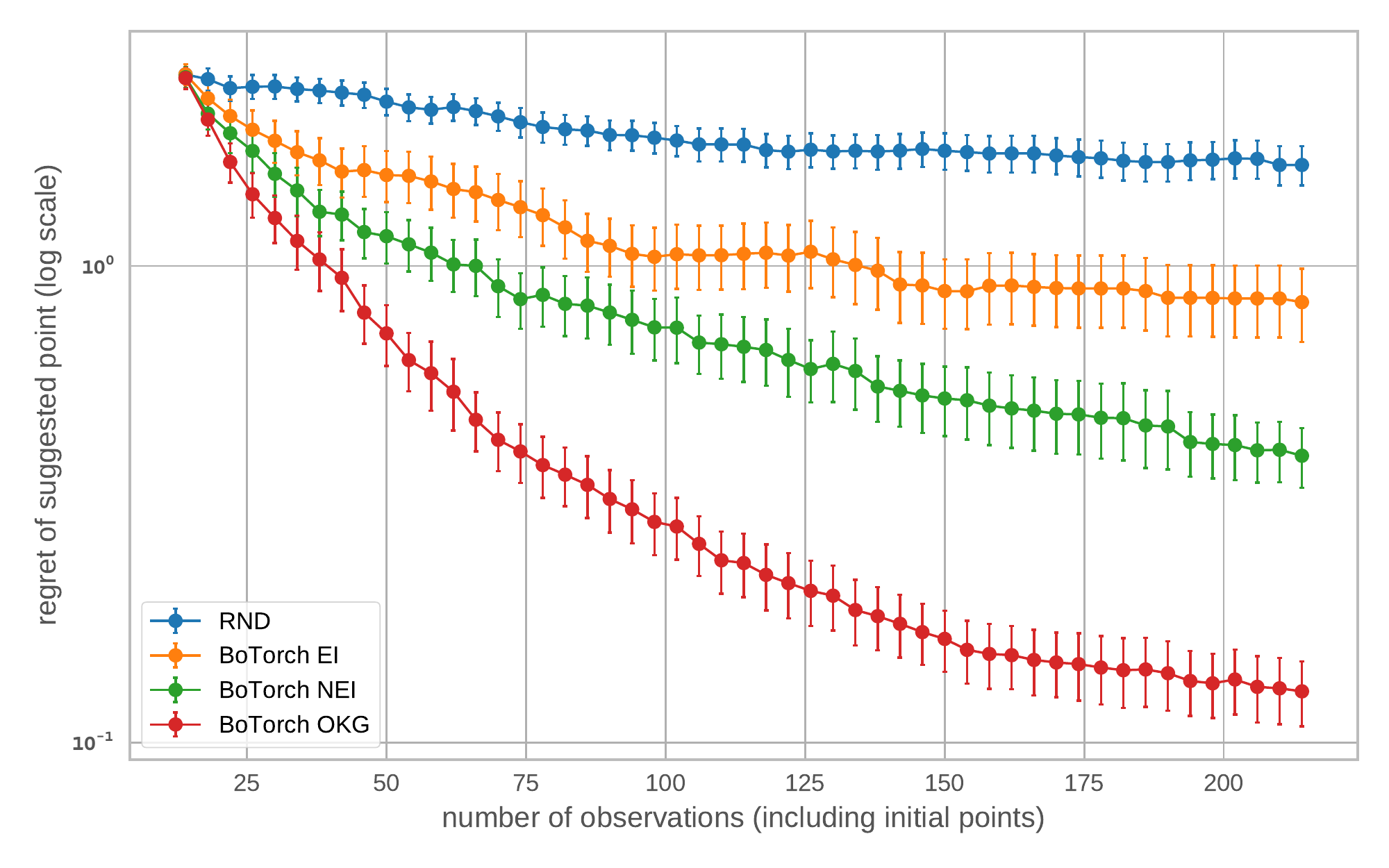}\\[-3ex]
    \caption{Constrained Hartmann6, $f_2(x) = \|x\|_1 - 3$}
    \label{fig:appdx:AddEmpirical:Constrained:HartmannL1}
    \end{minipage}
    \hfill
    \begin{minipage}[l]{0.5\columnwidth}
    \centering
    \includegraphics[width=1\linewidth]{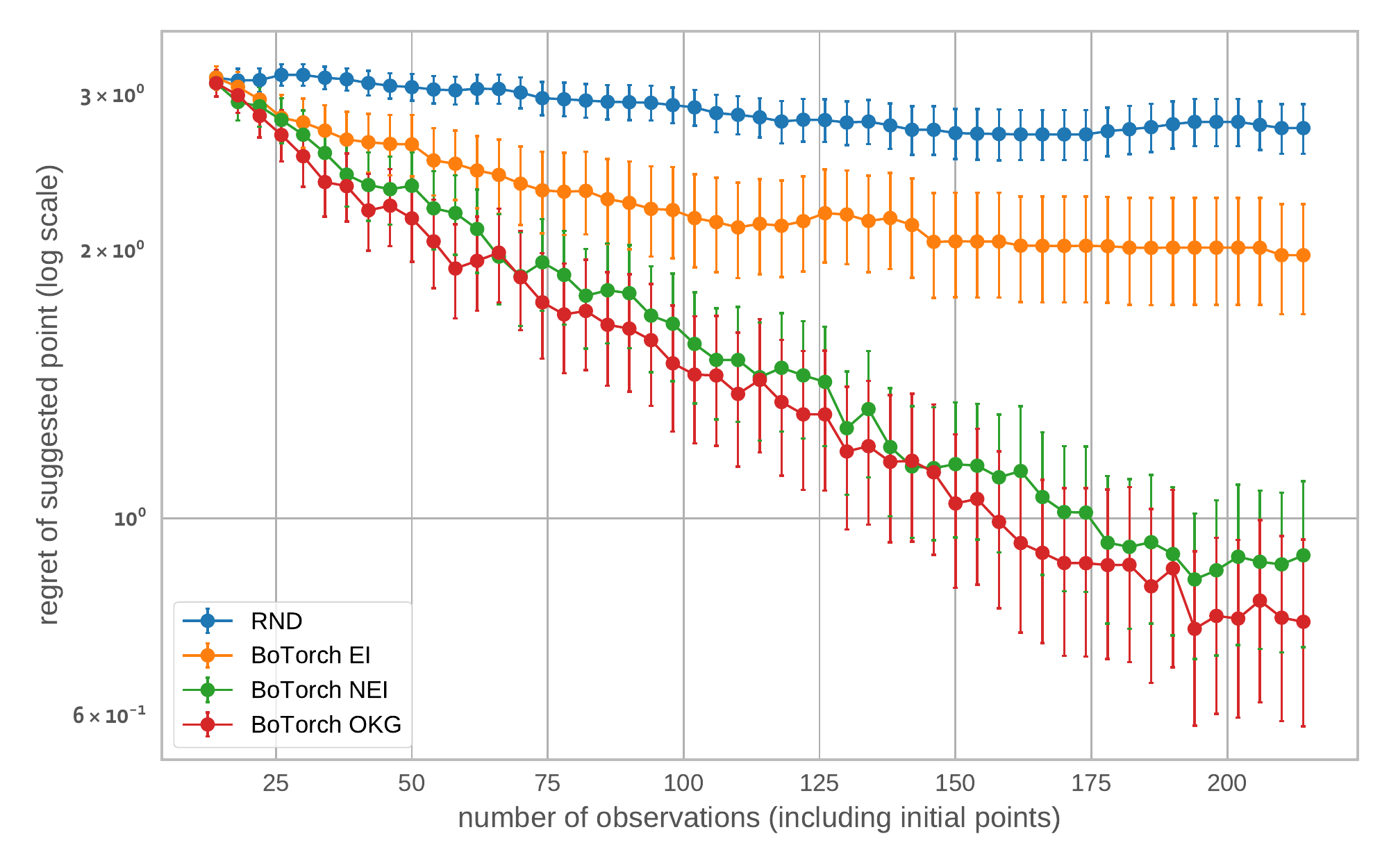}\\[-3ex]
    \caption{Constrained Hartmann6, $f_1(x) = \|x\|_2 - 1$ }
    \label{fig:appdx:AddEmpirical:Constrained:HartmannL2}
\end{minipage}
\end{figure*}

\subsection{Hyperparameter Optimization Details}
\label{appdx:subsec:AddEmpirical:hyper}

This section gives further detail on the experimental settings used in each of the hyperparameter optimization problems. As HPO typically involves long and resource intensive training jobs, it is standard to select the configuration with the best observed performance, rather than to evaluate a ``suggested'' configuration (we cannot perform noiseless function evaluations). 

\textbf{DQN and Cartpole:}
We consider the case of tuning a deep Q-network (DQN) learning algorithm \citep{mnih2013playing,mnih2015human} on the \emph{Cartpole} task from OpenAI gym \citep{brockman2016openai} and the default DQN agent implemented in Horizon \citep{gauci2018horizon}. Figure \ref{fig:Experiments:Cartpole} shows the results of tuning five hyperparameters, \emph{exploration parameter} (``epsilon''), the \emph{target update rate}, the \emph{discount factor}, the \emph{learning rate}, and the \emph{learning rate decay}.
We allow for a maximum of 60 training episodes or 2000 training steps, whichever occurs first. To reduce noise, each ``function evaluation'' is taken to be an average of 10 independent training runs of DQN.  
Figure \ref{fig:Experiments:Cartpole} presents the optimization performance of various acquisition functions from the different packages, using 15 rounds of parallel evaluations of size $q=4$, over 100 trials. While in later iterations all algorithms achieve reasonable performance, 
\botorch{} \OKG{}, EI, NEI, and GPyOpt LP-EI show faster learning early on.

\begin{figure*}
    \begin{minipage}[l]{0.5\columnwidth}
    \centering
    \includegraphics[width=\columnwidth]{./plots/dqn_cartpole.pdf}\\[-3ex]
    \caption{DQN tuning benchmark (Cartpole)}
    \label{fig:Experiments:Cartpole}
    \end{minipage}
    \hfill
    \begin{minipage}[l]{0.5\columnwidth}
    \centering
    \includegraphics[width=\columnwidth]{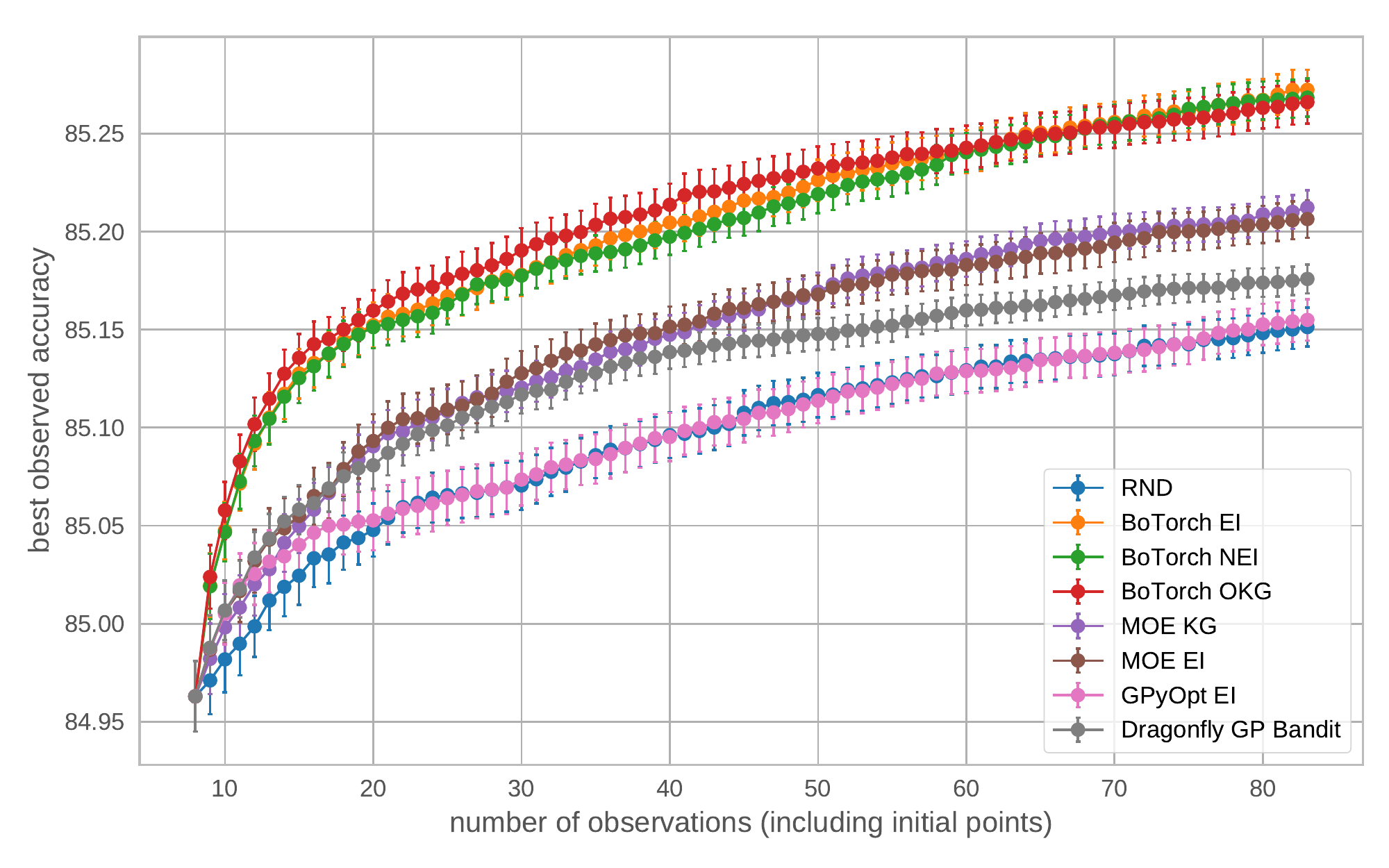}\\[-3ex]
    \caption{NN surrogate model, best observed accuracy}
    \label{fig:Experiments:SurrogateParamNet}
    \end{minipage}
\end{figure*}

\textbf{Neural Network Surrogate:}
%
We consider the neural network surrogate model for the UCI Adult data set introduced by \citet{falkner2018robusteff}, which is available as part of HPOlib2 \citep{eggensperger2019hpolib2}. We use a surrogate model to achieve a high level of precision in comparing the performance of the algorithms without incurring excessive computational training costs.  
This is a six-dimensional problem over network parameters (\emph{number of layers}, \emph{units per layer}) and training parameters (\emph{initial learning rate}, \emph{batch size}, \emph{dropout}, \emph{exponential decay factor for learning rate}). Figure~\ref{fig:Experiments:SurrogateParamNet} shows optimization performance in terms of best observed classification accuracy.
Results are means and 95\% confidence intervals computed from 200 trials with 75 iterations of size~$q=1$. All \botorch{} algorithms perform quite similarly here, with \OKG{} doing slightly better in earlier iterations. Notably, they all achieve significantly better accuracy than all other algorithms.

\textbf{Stochastic Weight Averaging on CIFAR-10:} 
Our final example is for the recently proposed \emph{Stochastic Weight Averaging} (SWA) procedure of \citet{izmailov2018averaging}, for which good hyperparameter settings are not fully understood. The setting is 300 epochs of training on the VGG-16 \citep{simonyan2014very} architecture for CIFAR-10. We tune three SWA hyperparameters: \emph{learning rate}, \emph{update frequency}, and \emph{starting iteration} using \OKG{}.
\citet{izmailov2018averaging} report the mean and standard deviation of the test accuracy over three runs to be $93.64$ and $0.18$, respectively, which corresponds to a 95\% confidence interval of $93.64 \pm 0.20$. 
We tune the problem to an average accuracy of $93.84 \pm 0.03$.

\section{Additional Theoretical Results and Omitted Proofs}
\label{appdx:sec:GeneralSAAresults}

\subsection{General SAA Results}
Recall that we assume that $f(\bx) \sim h(\bx, \epsilon)$ for some $h:\bbX \times \bbR^s \rightarrow \bbR^{q\times m}$ and base random variable $\epsilon \in \bbR^s$ (c.f. Section~\ref{sec:OptimizeAcq} for an explicit expression for~$h$ in case of a GP model). We write
\begin{align}
    A(\bx, \epsilon) := a(g(h(\bx, \epsilon))).
\end{align}

\begin{theorem}[\citet{homemDeMello2008convergence}] 
\label{thm:appdx:GeneralSAAresults:AlmostSure}
Suppose that (i) $\bbX$ is a compact metric space, (ii) $\hat{\alpha}_{\!N}(\bx) \xrightarrow{a.s.} \alpha(\bx)$ for all $\bx \in \bbX^q$, and (iii) there exists an integrable function $\ell:\bbR^s\mapsto \bbR$ such that for almost every $\epsilon$ and all $\bx, \by \in \bbX$, 
\begin{align}
    |A(\bx, \epsilon) - A(\by, \epsilon)| \leq \ell(\epsilon) \|\bx - \by\|.
\label{eq:GeneralSAAresults:AlmostSure:LipschitsCond}
\end{align}
Then 
$\hat{\alpha}_{\!N}^* \xrightarrow{a.s.} \alpha^*$
and 
$\textnormal{dist}(\hat{\bx}_{\!N}^*, \mathcal{X}_f^*) \xrightarrow{a.s.}  0$.
\end{theorem}

\begin{proposition}
\label{prop:appdx:GeneralSAAresults:GP}
    Suppose that (i) $\bbX$ is a compact metric space, (ii) $f$ is a GP with continuously differentiable prior mean and covariance functions, and (iii) $g( \cdot )$ and $a( \cdot, \Phi)$ are Lipschitz continuous. Then, condition~\eqref{eq:GeneralSAAresults:AlmostSure:LipschitsCond} in Theorem \ref{thm:appdx:GeneralSAAresults:AlmostSure} holds.
\end{proposition}

The following proposition follows directly from Proposition 2.1, Theorem 2.3, and remarks on page 528 of~\cite{homemDeMello2008convergence}.

\begin{proposition}[\citet{homemDeMello2008convergence}]
\label{prop:appdx:GeneralSAAresults:Rate}
    Suppose that, in addition to the conditions in Theorem~\ref{thm:appdx:GeneralSAAresults:AlmostSure}, (i) the base samples $E = \{\epsilon^i\}_{i=1}^N$ are i.i.d., (ii) for all $\bx \in \bbX^q$ the moment generating function $M^{\!A}_\bx(t) := \bbE[ e^{tA(\bx, \epsilon)}]$ of $A(\bx, \epsilon)$ is finite in an open neighborhood of $t=0$ and (iii) the moment generating function $M^\ell(t) := \bbE[ e^{t\ell(\epsilon)}]$ is finite in an open neighborhood of $t=0$. Then, there exist $ K <\infty$ and $\beta > 0$ such that  $\mathbb{P} (\textnormal{dist}(\hat{\bx}_{\!N}, \mathcal{X}_f^*) ) \le K e^{-\beta N}$ for all $N \ge 1$.
\end{proposition}

\subsection{Formal Statement of Theorem 1}
\begin{manualtheorem}{\ref{thm:OptimizeAcq:SAAConvergence} (Formal Version)}
Suppose (i) $\bbX$ is compact,
(ii) $f$ has a GP prior with continuously differentiable mean and covariance functions, and (iii) $g( \cdot )$ and $a( \cdot , \Phi)$ are Lipschitz continuous. 
If the base samples $\{\epsilon^i\}_{i=1}^N$ are drawn i.i.d. from $\mathcal N(0,1)$,
then 
\begin{enumerate}[label={(\arabic*)}]
\item $\hat{\alpha}_{\!N}^* \rightarrow \alpha^*$ a.s., and
\item $\textnormal{dist}(\hat{\bx}_{\!N}^*, \mathcal{X}^*) \rightarrow 0$ a.s.
\end{enumerate}
If, in addition, (iii) for all $\bx \in \bbX^q$ the moment generating function $M^{\!A}_\bx(t) := \bbE[ e^{tA(\bx, \epsilon)}]$ of $A(\bx, \epsilon)$ is finite in an open neighborhood of $t=0$ and (iv) the moment generating function $M^\ell(t) := \bbE[ e^{t\ell(\epsilon)}]$ is finite in an open neighborhood of $t=0$, then
\begin{enumerate}[resume*]
\item $\forall\, \delta>0$, $\exists\, K <\infty$, $\beta > 0$ s.t.  $\mathbb{P}\bigl(\textnormal{dist}(\hat{\bx}_{\!N}^*, \mathcal{X}^*) > \delta \bigr) \le K e^{-\beta N}$ for all $N \geq 1$.
\end{enumerate}
\end{manualtheorem}

\subsection{Randomized Quasi-Monte Carlo Sampling for Sample Average Approximation}
\label{appdx:subsec:GeneralSAAresults:RQMC}

In order to use randomized QMC methods with SAA for MC acquisition function, the base samples $E = \{\epsilon^i\}$ will need to be generated via RQMC. For the case of Normal base samples, this can be achieved in various ways, e.g. by using inverse CDF methods or a suitable Box-Muller transform of samples $\epsilon^i \in [0, 1]^s$ (both approaches are implemented in \botorch{}). In the language of Section~\ref{sec:OptimizeAcq}, such a transform will become part of the base sample transform $\epsilon \mapsto h(\bx, \epsilon)$ for any fixed $\bx$. 

For the purpose of this paper, we consider scrambled $(t, d)$-sequences as discussed by~\citet{owen1995randomlypermuted}, which are a particular class of RQMC method (\botorch{} uses PyTorch's implementation of scrambled Sobol sequences, which are $(t, d)$-nets in base 2). Using recent theoretical advances from \citet{owen2020rqmcslln}, it is possible to generalize the convergence results from Theorems~\ref{thm:OptimizeAcq:SAAConvergence} and~\ref{thm:OneShotKG:OneStep:Convergence} to the RQMC setting (to our knowledge, this is the first practical application of these theoretical results).

Let $(N_i)_{i\geq 1}$ be a sequence with $N_i \in \mathbb{N}$ s.t. $N_i \rightarrow \infty$ as $i\rightarrow \infty$. Then we have the following (see Appendix~\ref{appdx:sec:Proofs} for the proofs):

\begin{manualtheorem}{\ref{thm:OptimizeAcq:SAAConvergence}(q)}
    In the setting of Theorem~\ref{thm:OptimizeAcq:SAAConvergence}, let $\{\epsilon^i\}$ be samples from a $(t,d)$-sequence in base~$b$ with gain coefficients no larger than $\Gamma < \infty$, randomized using a nested uniform scramble as in~\cite{owen1995randomlypermuted}. Then, the conclusions of Theorem~\ref{thm:OptimizeAcq:SAAConvergence} still hold. 
    In particular, 
    \begin{enumerate}[label={(\arabic*)}]
    \item $\hat{\alpha}_{\!N_i}^* \rightarrow \alpha^*$ a.s. as $i\rightarrow \infty$,
    \item $\textnormal{dist}(\hat{\bx}_{\!N_i}^*, \mathcal{X}^*) \rightarrow 0$ a.s. as $i\rightarrow \infty$,
    \item $\forall\, \delta>0$, $\exists\, K <\infty$, $\beta > 0$ s.t.  $\mathbb{P}\bigl(\textnormal{dist}(\hat{\bx}_{\!N_i}^*, \mathcal{X}^*) > \delta \bigr) \le K e^{-\beta N_i}$ for all $i \geq 1$.
    \end{enumerate}
\end{manualtheorem}

\begin{manualtheorem}{\ref{thm:OneShotKG:OneStep:Convergence}(q)}
    In the setting of Theorem~\ref{thm:OneShotKG:OneStep:Convergence}, let $\{\epsilon^i\}$ be samples from $(t,d)$-sequence in base~$b$, with gain coefficients no larger than $\Gamma < \infty$, randomized using a nested uniform scramble as in~\cite{owen1995randomlypermuted}. 
    Then,
    \begin{enumerate}[label={(\arabic*)}]
    \item $\hat{\alpha}_{\mathrm{KG},N_i}^* \xrightarrow{a.s.} \alpha_{\mathrm{KG}}^*$ as $i \rightarrow \infty$,
    \item $\textnormal{dist}(\hat{\bx}_{\mathrm{KG},N_i}^*, \mathcal{X}_{\mathrm{KG}}^*) \xrightarrow{a.s.} 0$ as $i \rightarrow \infty$.
    \end{enumerate}
\end{manualtheorem}

Theorem~\ref{thm:OneShotKG:OneStep:Convergence}(q) as stated does not provide a rate on the convergence of the optimizer. We believe that  such result is achievable, but leave it to future work.

Note that while the above results hold for any sequence $(N_i)_i$ with $N_i \rightarrow \infty$, in practice the RQMC integration error can be minimized by using sample sizes that exploit intrinsic symmetry of the $(t, d)$-sequences. Specifically, for integers $b\geq 2$ and $M \geq 1$, let
\begin{align}
    \mathcal{N} := \{mb^k \,|\, m \in \{1, \dotsc, M\}, k\in \mathbb{N}_+\}.
\label{eq:appdx:subsec:GeneralSAAresults:RQMC:SampleSizes}
\end{align}
In practice, we chose the MC sample size $N$ from the unique elements of $\mathcal{N}$.

\subsection{Asymptotic Optimality of \OKG{}}
\label{appdx:subsec:GeneralSAAresults:AsympOptimKG}

Consider the case where $f_\textnormal{true}$ is drawn from a GP prior with $f \overset{d}{=} f_\textnormal{true}$, and that $g(f) \equiv f$. The 
KG \emph{policy} (i.e., when used to select sequential measurements in a dynamic setting) is known to be \emph{asymptotically optimal} \citep{frazier2008knowledge,frazier2009knowledge,poloczek2017misokg, bect2019supermartingaleGP}, meaning that as the number of measurements tends to infinity, an optimal point $x^* \in \mathcal{X}_f^* := \argmax_{x\in \bbX} f(x)$ is identified.
Although it does not necessarily signify good finite sample performance, this is considered a useful property for acquisition functions \citep{frazier2009knowledge}. 
In this section, we state two results showing that OKG also possesses this property, providing further theoretical justification for the MC approach taken by \botorch{}.

Let $\calD_0$ be the initial data and $\calD_n$ for $n \ge 1$ be the data generated by taking measurements according to \OKG{} 
using $N_n$ MC samples in iteration $n$, i.e., $\bx_{n+1} \in \argmax_{\bx\in \bbX^q} \hat{\alpha}_{\mathrm{KG}, N_n}(\bx; \mathcal D_n)$ for all $n$, and let $\chi_n \in \argmax_{x\in \bbX} \bbE[  f(x)\, | \, \calD_n]$. 
Then we can show the following:
\begin{theorem}
\label{thm:OneShotKG:AsymptoticOptimality:OKG}
    Suppose conditions (i) and (ii) of Theorem \ref{thm:OptimizeAcq:SAAConvergence} and (iii) of Theorem  \ref{thm:OneShotKG:OneStep:Convergence} are satisfied.
    In addition, suppose that $\limsup_n N_n = \infty$.
    Then, $f(\chi_n) \rightarrow f(x^*)$ a.s. and in $L^1$.
\end{theorem}

Theorem~\ref{thm:OneShotKG:AsymptoticOptimality:OKG} shows that \OKG{} is asymptotically optimal if the number of fantasies~$N_n$ grows asymptotically with~$n$ (this assumes we have an analytic expression for the inner expectation. If not, a similar condition must be imposed on the number of inner MC samples). 
In the special case of finite $\bbX$, we can quantify the sample sizes $\{N_n\}$ that ensure asymptotic optimality of OKG:
\begin{theorem}
\label{thm:OneShotKG:AsymptoticOptimality:finiteOKG}
Along with conditions (i) and (ii) of Theorem~\ref{thm:OptimizeAcq:SAAConvergence}, suppose that $|\bbX| < \infty$ and $q=1$. Then, if for some $\delta > 0$,
$
N_n \ge A_n^{-1} \log(K_n / \delta)
$
a.s., 
where $A_n$ and $K_n$ are a.s. finite and depend on $\mathcal D_n$ (these quantities can be computed), we have
$f(\chi_n) \rightarrow \max_{x \in \mathbb{X}} f(x)$ a.s..
\end{theorem}

\subsection{Proofs}
\label{appdx:sec:Proofs}

In the following, we will denote by $\mu_{\calD}(x):= \bbE [ f(x) \mid \calD]$ and $K_{\calD}(x,y) := \bbE [ (f(x) - \bbE[f(x)]) (f(y) - \bbE[f(y)])^T \mid \calD]$ the posterior mean and covariance functions of $f$ conditioned on data $\calD$, respectively. Under some abuse of notation, we will use $\mu_{\calD}(\bx)$ and $K_{\calD}(\bx,\by)$ to denote multi point (vector / matrix)-valued variants of $\mu_\calD$ and $K_\calD$, respectively. 
If $f$ has a GP prior, then the posterior mean and covariance $\mu_{\calD}(\bx)$ and $K_{\calD}(\bx,\by)$ have well-known explicit expressions \citep{rasmussen2006gpml}.

For notational simplicity and without loss of generality, we will focus on single-output GP case ($m=1$) in this section. Indeed, in the multi-output case ($m>1$), we have a GP $f$ on $\bbX \times \mathbb{M}$ with $\mathbb{M} = \{1, \dotsc, m\}$, and  covariance function $(x_1, i_1), (x_2, i_2) \mapsto \tilde{K}((x_1, i_1), (x_2, i_2))$.  For $q=1$ we then define $x \mapsto \tilde{f}(x) := [f(x, 0), ..., f(x, m)]$, and then stack these for $q>1$: $\bx \mapsto [\tilde{f}(\bx_1)^T, ..., \tilde{f}(\bx_q)^T]^T$. Then the analysis in the proofs below can be done on $mq$-dimensional and $mq \times mq$-dimensional posterior mean and covariance matrices (instead of $q$ and $q \times q$ dimensional ones for $m=1$). Differentiability assumptions are needed only to establish certain boundedness results (e.g. in the proof of Proposition~\ref{prop:appdx:GeneralSAAresults:GP}), but $\mathbb{M}$ is finite, so we will require differentiability of $K(( \cdot ,i_1), ( \cdot ,i_2))$ for each $i_1$ and $i_2$. Assumptions on other quantities can be naturally extended (e.g. for Theorem~\ref{thm:OneShotKG:OneStep:Convergence} $g$ will need to be Lipschitz on $\bbR^{q\times m}$ rather than on $\bbR^q$, etc.).

\begin{proof}[\textbf{Proof of Proposition~\ref{prop:appdx:GeneralSAAresults:GP}}]
Without loss of generality, we may assume $m=1$ (the multi-output GP case follows immediately from applying the result below to $q' = qm$ and re-arranging the output). For a GP, we have $h_\calD(\bx, \epsilon) = \mu_\calD(\bx) + L_\calD(\bx)\epsilon$ with $\epsilon \sim \mathcal{N}(0, I_q)$, where $\mu_\calD(\bx)$ is the posterior mean and $L_\calD(\bx)$ is the Cholesky decomposition of the posterior covariance $M_\calD(\bx)$.
It is easy to verify from the classic GP inference equations \citep{rasmussen2006gpml} that if prior mean and covariance function are continuously differentiable, then so are posterior mean $\mu_\calD( \cdot )$ and covariance $K_\calD( \cdot )$. Since the Cholesky decomposition is also continuously differentiable \citep{murray2016diffcholesky}, so is~$L_\calD( \cdot )$.
%
As $\bbX$ is compact and $\mu_\calD( \cdot )$ and $L_\calD( \cdot )$ are continuously differentiable, their derivatives are bounded. 
It follows from the mean value theorem that there exist $C_\mu, C_L < \infty$ s.t. $\|\mu_\calD(\bx) - \mu_\calD(\by)\| \leq C_\mu \|\bx - \by\|$ and $\|(L_\calD(\bx) - L_\calD(\by))\epsilon\| \leq C_L \|\epsilon \| \|\bx - \by\|$. 
Thus,
\begin{align*}
    \|h_\calD(\bx, \epsilon) - h_\calD(\by, \epsilon)\| &= \|\mu_\calD(\bx) - \mu(\by) + (L_\calD(\bx) - L_\calD(\by))\epsilon\| \\
    &\leq \|\mu_\calD(\bx) - \mu_\calD(\by)\| + \|(L_\calD(\bx) - L_\calD(\by))\epsilon\| \\
    &\leq \ell_h(\epsilon) \|\bx - \by\|
\end{align*}
where $\ell_h(\epsilon) := C_\mu + C_L \|\epsilon\|$. Since, by assumption, $g( \cdot )$ and $a( \cdot ; \Phi)$ are Lipschitz (say with constants $L_a$ and $L_g$, respectively), it follows that $\|A(\bx, \epsilon) - A(\by, \epsilon)\| \leq L_a L_g \ell_h(\epsilon) \|\bx -\by \|$. 
It thus suffices to show that $\ell_h(\epsilon)$ is integrable. To see this, note that 
$|\ell_h(\epsilon)| \leq C_\mu + C_L C \textstyle\sum_i |\epsilon_i|$
for some $C < \infty$ (equivalence of norms), and that $\epsilon_i \sim \mathcal{N}(0, 1)$ is integrable.
\end{proof}

\begin{lemma}
\label{prop:appdx:GeneralSAAresults:GPMGF}
    Suppose that (i) $f$ is a GP with continuously differentiable prior mean and covariance function, and (ii) that $a( \cdot, \Phi)$ and $g( \cdot )$ are Lipschitz. Then, for all $\bx \in \bbX^q$ the moment generating functions $M^{\!A}_\bx(t) := \bbE[ e^{tA(\bx, \epsilon)}]$ of $A(\bx, \epsilon)$ and $M^\ell(t) := \bbE[ e^{t\ell(\epsilon)}]$ are finite for all $t \in \mathbb R$.
\end{lemma}

\begin{proof}[\textbf{Proof of Lemma~\ref{prop:appdx:GeneralSAAresults:GPMGF}}]
Recall that $h_\calD(\bx, \epsilon) = \mu_\calD(\bx) + L_\calD(\bx) \epsilon$ for the case of $f$ being a GP, where $\mu_\calD(\bx)$ is the posterior mean and $L_\calD(\bx)$ is the Cholesky decomposition of the posterior covariance $K_\calD(\bx)$. Mirroring the argument from the proof of Proposition~\ref{prop:appdx:GeneralSAAresults:GP}, 
it is clear that $A(\bx, \epsilon)$ is Lipschitz in $\epsilon$ for each $\bx \in \bbX^q$, say with constant $\tilde{C}_L$. Note that this implies that $\bbE[|A(\bx, \epsilon)|] < \infty$ for all $\bx$. We can now appeal to results pertaining to the concentration of Lipschitz functions of Gaussian random variables: the Tsirelson-Ibragimov-Sudakov inequality \citep[Theorem 5.5]{boucheron2013concentration} implies that
\[
\log M^A_\bx(t) \le \frac{t^2 \tilde{C}_L^2}{2} + t \, \bbE[ A(\bx, \epsilon) ]
\]
for any $t \in \mathbb R$, which is clearly finite for all $t$ since $\bbE[A(\bx, \epsilon)] \leq \bbE[|A(\bx, \epsilon)|]$. 
From the proof of Proposition \ref{prop:appdx:GeneralSAAresults:GP}, we know that $A(\bx, \epsilon)$ is $\ell(\epsilon)$-Lipschitz in $\bx$, where $\ell(\epsilon)$ is itself Lipschitz in~$\epsilon$. Hence, the concentration result in Theorem 5.5 of \cite{boucheron2013concentration} applies again, and we are done.
\end{proof}

\begin{proof}[\textbf{Proof of Theorem~\ref{thm:OptimizeAcq:SAAConvergence}}]
    Under the stated assumptions, Lemma~\ref{prop:appdx:GeneralSAAresults:GP} ensures that condition~\eqref{eq:GeneralSAAresults:AlmostSure:LipschitsCond} in Theorem~\ref{thm:appdx:GeneralSAAresults:AlmostSure} holds. 
    Further, note that the argument about Lipschitzness of $A(\bx, \epsilon)$ in $\epsilon$ in the proof of Lemma~\ref{prop:appdx:GeneralSAAresults:GPMGF} implies that 
    $\bbE[|A(\bx, \epsilon)|] < \infty$ for all $\bx \in \bbX^q$. 
    Since the $\{\epsilon^i\}_{i=1}^N$ are i.i.d, the strong law of large numbers implies that $\hat{\alpha}_{\!N}(\bx) \rightarrow \alpha(\bx)$ a.s. for all $x\in \bbX$. Claims (1) and (2) then follow by applying Theorem~\ref{thm:appdx:GeneralSAAresults:AlmostSure}, and claim (3) follows by applying Proposition~\ref{prop:appdx:GeneralSAAresults:Rate}.
\end{proof}

\begin{proof}[\textbf{Proof of Theorem~\ref{thm:OptimizeAcq:SAAConvergence}(q)}]

    Mirroring the proof of Theorem~\ref{thm:OptimizeAcq:SAAConvergence}, we need to show that $\hat{\alpha}_{N_i}(\bx) \rightarrow \alpha(\mathbf{x})$ a.s. as $i\rightarrow \infty$ for all $\bx \in \bbX^q$. 
    For any $\bx \in \bbX^q$ and any $\epsilon_0 \in \bbR^q$, we have (by convexity and monotonicity of $|x| \mapsto |x|^2$ and the Lipschitz assumption on $a$ and $g$) that
    \begin{align*}
        |A(\bx,\epsilon)|^2
        &= |A(\bx,\epsilon_0) + A(\bx,\epsilon) - A(\bx,\epsilon_0)|^2 \\
        &\leq |A(\bx,\epsilon_0)|^2 + |A(\bx,\epsilon) - A(\bx,\epsilon_0)|^2 \\
        &\leq |A(\bx,\epsilon_0)|^2 + L_a^2 L_g^2 \|h_\calD(\bx, \epsilon) - h_\calD(\bx, \epsilon_0)\|^2 
    \end{align*}
    where $h_\calD(\bx, \epsilon) = \mu_\calD(\bx) + L_\calD(\bx) \Phi^{-1}(\epsilon)$ with $\Phi^{-1}$ the inverse CDF of $\mathcal{N}(0, 1)$, applied element-wise to the vector $\epsilon$ of qMC samples. 
    Now choose $\epsilon_0 = (0.5, \dotsc, 0.5)$, then
    \begin{align*}
        |A(\bx,\epsilon)|^2
        &\leq |a(g(0))|^2 + L_a^2 L_g^2 \|h_\calD(\bx, \epsilon)\|^2 
    \end{align*}
    Since the $\{\epsilon^i\}$ are generated by a nested uniform scramble, we know from \citet{owen1995randomlypermuted} that $\epsilon \sim U[0,1]^q$, and therefore $\Phi^{-1}(\epsilon) \sim \mathcal{N}(0, I_q)$. Since affine transformations of Gaussians remain Gaussian, we have that $\bbE\left[\|h_\calD(\bx, \epsilon)\|^2\right] < \infty$. This shows that $A(\bx,\epsilon) \in L^2([0,1]^q)$. That $\hat{\alpha}_{N_i}(\bx) \rightarrow 0$ a.s. as $i\rightarrow \infty$ for all $\bx \in \bbX^q$ now follows from \citet[Theorem~3]{owen2020rqmcslln}.
\end{proof}

\begin{lemma}
\label{lem:OneShotKG:OneStep:InnerSampling}
    If $f$ is a GP, then $f_{\calD_\bx}(x')= h(x', \bx, \epsilon, \epsilon_I)$, where $\epsilon\sim \mathcal{N}(0, I_q)$ and $\epsilon_I \sim \mathcal{N}(0, 1)$ are independent and $h$ is linear in both $\epsilon$ and $\epsilon_I$.
\end{lemma}

\begin{proof}[\textbf{Proof of Lemma~\ref{lem:OneShotKG:OneStep:InnerSampling}}]
This essentially follows from the property of a GP that the covariance conditioned on a new observation $(x, y)$ is independent of~$y$.\footnote{In some cases we may consider constructing a heteroskedastic noise model that results in the function $\sigma^2(\bx)$ changing depending on observations $y$, in which case this argument does not hold true anymore. We will not consider this case further here.} 
We can write $f_{\calD_\bx}(x') = \mu_{\calD_\bx}(x') + L_{\calD_\bx}^\sigma(x')\epsilon_I$ , where 
\begin{align*}
\mu_{\calD_\bx}(x') := \mu_\calD(x') + K_\calD(x'\!,\bx)  K_\calD^\sigma(\bx)^{-1}L_\calD^\sigma(\bx)\epsilon,
\end{align*}
$L_{\calD}^\sigma(\bx)$ is the Cholesky decomposition of $K_{\calD}^\sigma(\bx):= K_{\calD}(\bx,\bx) + \text{diag}(\sigma^2(\bx_1), \dotsc, \sigma^2(\bx_q))$, and $L_{\calD_\bx}^\sigma(x')$ is the Cholesky decomposition of 
\begin{align*}
K_{\calD_\bx}(x', x') := K(x', x') - K_\calD(x', \bx) K_\calD^\sigma(\bx)^{-1} K_\calD(\bx, x').
\end{align*}
Hence, we see that $f_{\calD_\bx}(x')= h(x', \bx, \epsilon, \epsilon_I)$, with
\begin{align}
\label{eq:OneShotKG:OneStep:InnerSampling:h}
h(x', \bx, \epsilon, \epsilon_I) &=
\mu_\calD(x') + K_\calD(x'\!,\bx)  K_\calD^\sigma(\bx)^{-1}L_\calD^\sigma(\bx)\epsilon
+ L_{\calD_\bx}^\sigma(x')\epsilon_I,
\end{align}
which completes the argument.
\end{proof}

\begin{theorem}
\label{appdx:thm:OneShotKG:AsymptoticOptimality:Bect}
    Let $(a_n)_{n\geq 1}$ be a sequence of non-negative real numbers such that $a_n \rightarrow 0$. Suppose that (i) $\bbX$ is a compact metric space, (ii) $f$ is a GP with continuous sample paths and continuous variance function $x \mapsto \sigma^2(x)$, and (iii) $(\bx_n)_{n\geq 1}$ is such that $\alpha^n_{\mathrm{KG}}(\bx_n) > \sup_{\bx \in \bbX^q} \alpha^n_{\mathrm{KG}}(\bx) - a_n$ infinitely often almost surely. Then $\alpha^n_{\mathrm{KG}}(\bx) \rightarrow 0$ a.s. for all $\bx \in \bbX^q$.
\end{theorem}

\begin{proof}[\textbf{Proof of Theorem~\ref{appdx:thm:OneShotKG:AsymptoticOptimality:Bect}}]
    \citet{bect2019supermartingaleGP} provide a proof for the case~$q=1$. Following their exposition, one finds that the only thing that needs to be verified in order to generalize their results to $q>1$ is that condition (c) in their Definition~3.18 holds also for the case $q>1$. What follows is the multi-point analogue of step (f) in the proof of their Theorem~4.8, which establishes this. 
    
    Let $\mu: \bbX \rightarrow \bbR$ and $K: \bbX \times \bbX \rightarrow \bbR_+$ denote mean and covariance function of~$f$. Let $Z_\bx:= f(\bx) + \text{diag}(\sigma(\bx))$, where $\sigma(\bx):= (\sigma(\bx_1), \dotsc, \sigma(\bx_q))$, with $\epsilon \sim \mathcal{N}(0, I_d)$ independent of~$f$. Moreover, let $x^* \in \argmax \mu(x)$.
    Following the same argument as \citet{bect2019supermartingaleGP}, 
    we arrive at the intermediate conclusion that $\bbE[ \max\{0, W_{\bx, y}\} ] = 0$, where 
    $W_{\bx, y} := \bbE[f(y) \!\mid\! Z_\bx] - \bbE[f(x^*) \!\mid\! Z_\bx]$. We need to show that this implies that $\max_{x\in \bbX}f(x) = m(x^*)$.

    Under some abuse of notation we will use $\mu$ and $K$ also as the vector / matrix-valued mean / kernel function.  
    Let $K^\sigma(\bx):= K(\bx, \bx) + \text{diag}(\sigma(\bx))$ and observe that
    \begin{align*}
        W_{\bx, y} = \mu(y) - \mu(x^*) + \mathbbm{1}_{\{C(\bx) \succ 0\}} (K(y,\bx) - K(x^*,\bx))K^\sigma(\bx)^{-1}(Z_\bx - \mu(\bx)),
    \end{align*}
    i.e., $W_{\bx, y}$ is Gaussian with $\text{Var}(W_{\bx, y}) = V(\bx, y, x^*)V(\bx, y, x^*)^T$, where $V(\bx, y, x^*) := (K(y,\bx) - K(x^*,\bx)) K^\sigma(\bx)^{-1}$.
    Since $\bbE[ \max\{0, W_{\bx, y}\} ] = 0$, we must have that $\text{Var}(W_{\bx, y}) = 0$. If $K^\sigma(\bx) \succ 0$, this means that $(K(y,\bx) - K(x^*,\bx)) = 0_q$. But if $K^\sigma(\bx)\not\succ 0$, then $K(\bx, \bx) \not \succ 0$, which in turn implies that $K(y,\bx) = K(x^*,\bx) = 0_q$. This shows that $K(y,\bx) = K(x^*,\bx)$ for all $y\in \bbX$ and all $\bx \in \bbX$. In particular, $K(x, y) = K(x, x^*)$ for all $y\in \bbX$. 
    Thus, $K(x,x) - K(x, y) = K(x, x^*) - K(x, x^*)$ for all $y, x \in \bbX$, and therefore $\text{Var}(f(x)- f(y)) = K(x, x) - K(x,y) - K(y, x) + K(y,y) = 0$. 
    As in \cite{bect2019supermartingaleGP} we can conclude that this means that the sample paths of $f - \mu$ are constant over~$\bbX$, and therefore $\max_{x\in\bbX} f(x) = m(x^*)$.
\end{proof}

\begin{proof}[\textbf{Proof of Theorem~\ref{thm:OneShotKG:OneStep:Convergence}}]
From Lemma~\ref{lem:OneShotKG:OneStep:InnerSampling} we have that $f_{\calD_\bx}(x')= h(x', \bx, \epsilon, \epsilon_I)$ with $h$ as in~\eqref{eq:OneShotKG:OneStep:InnerSampling:h}.
Without loss of generality, we can absorb $\epsilon_I$ into $\epsilon$ for the purposes of showing that condition~\eqref{eq:GeneralSAAresults:AlmostSure:LipschitsCond}
holds for the mapping $A_{\mathrm{KG}}(\bx, \epsilon) :=
\max_{x' \in \bbX} \bbE \bigl[ g(f(x')) \, | \, \calD_{\bx} \bigr]$.
Since the affine (and thus, continuously differentiable) transformation $g$ 
preserves the necessary continuity and differentiability properties, we can follow the same argument as in the proof of Theorem~1 of \cite{wu2017discretizationfree}. In particular, using continuous differentiability of GP mean and covariance function, compactness of~$\bbX$, and continuous differentiability of $g$, we can apply the envelope theorem in the same fashion. From this, it follows that for any $\epsilon \in \bbR^q$ and for each $1\leq l \leq q, 1\leq k \leq d$, the restriction of $\bx \mapsto A_{\mathrm{KG}}(\bx, \epsilon)$ to the $k,l$-th coordinate is absolutely continuous for all $\bx$, thus the partial derivative $\partial_{\bx_{lk}} A_{\mathrm{KG}}(\bx, \epsilon)$ exists a.e. Further, for each $l$ there exist $\Lambda_l \in \mathbb{R}^q$ with $\|\Lambda_l\| < \infty$ s.t. $|\partial_{\bx_{kl}} A_{\mathrm{KG}}(\bx, \epsilon)| \leq \Lambda_l^T |\varepsilon|$ a.e. on $\mathbb{X}^q$ (here $| \cdot |$ denotes the element-wise absolute value of a vector). 
This uniform bound on the partial derivatives can be used to show that $A_{\mathrm{KG}}$ is $\ell(\epsilon)$-Lipschitz. Indeed, writing the  difference $A_{\mathrm{KG}}(\by,\epsilon) - A_{\mathrm{KG}}(\bx,\epsilon)$ as a sum of differences in each of the $qd$ components of $\bx$ and $\by$, respectively, using the triangle inequality, absolute continuity of the element-wise restrictions, and uniform bound on the partial derivatives, we have that
\begin{align*}
    |A_{\mathrm{KG}}(\by,\epsilon) - A_{\mathrm{KG}}(\bx,\epsilon)| 
    \leq \sum_{k=1}^q \sum_{l=1}^d \Lambda_l^T|\epsilon| |\by_{kl} - \bx_{kl}| 
    \leq \max_{1\leq l \leq d} \left\{\Lambda_l^T|\epsilon|\right\} \|\by - \bx\|_1
\end{align*}
and so $\ell(\epsilon) = \max_l \{\Lambda_l^T|\epsilon|\}$. Going back to viewing $\epsilon$ as a random variable, it is straightforward to verify that $\ell(\epsilon)$ is integrable. Indeed, 
\begin{align*}
    \mathbb{E}[|\ell(\epsilon)|] &\leq \max_{l} \left\{ \textstyle\sum_{k=1}^q \Lambda_{lk} \bbE[|\epsilon_k|] \right\} 
    = \sqrt{2/\pi} \max_{l} \left\{ \|\Lambda_{l}\|_1\right\}. 
\end{align*}
Since $g$ is assumed to be affine in (iii), we can apply Lemma \ref{lem:OneShotKG:OneStep:InnerSampling} to see that $\bbE[ g(f(x')) \, | \, \calD_{\bx}]$ is a GP. Therefore, $A_{\mathrm{KG}}(\bx, \epsilon)$ represents the maximum of a GP and its moment generating function $\bbE [ \,  e^{t A_{\mathrm{KG}}(\bx, \epsilon)} ]$ is finite for all $t$ by Lemma \ref{lem:appdx:Proofs:MGFsupffinite}. This implies finiteness of its absolute moments \citep[Exercise 9.15]{meyer2012essential} and we have that $\bbE[|A_{\mathrm{KG}}(\bx, \epsilon)|] <\infty$ for all $\bx \in \bbX$.
Since the $\{\epsilon^i\}$ are i.i.d, the strong law of large numbers ensures that $\hat{\alpha}_{\mathrm{KG},N}(\bx) \rightarrow \alpha_{\mathrm{KG}}(\bx)$ a.s. Theorem \ref{thm:appdx:GeneralSAAresults:AlmostSure} now applies to obtain (1) and (2).

Moreover, by the analysis above, it holds that
\[
\ell(\epsilon) = \max_l \{\Lambda_l^T|\epsilon|\} \le q \max_l \| \Lambda_l^T \|_\infty \,  \| \epsilon \|_\infty =: \ell'(\epsilon),
\]
so $\ell'(\epsilon)$ is also a Lipschitz constant for $A_{\mathrm{KG}}(\cdot, \epsilon)$. Here, the absolute value version (the second result) of Lemma \ref{lem:appdx:Proofs:MGFsupffinite} applies, so we have that $\bbE [ \,  e^{t \ell'(\epsilon)} ]$ is finite for all $t$. The conditions of Proposition \ref{prop:appdx:GeneralSAAresults:Rate} are now satisfied and we have the desired conclusion.
\end{proof}

\begin{proof}[\textbf{Proof of Theorem~\ref{thm:OneShotKG:OneStep:Convergence}(q)}]
In the RQMC setting, we have by \citet{owen1995randomlypermuted} that $\epsilon \sim U[0,1]^q$. Therefore, we are now interested in examining $\tilde{A}_{\mathrm{KG}}(\bx,\epsilon) := A_{\mathrm{KG}}(\bx,\Phi^{-1}(\epsilon))$,
since $\Phi^{-1}(\epsilon) \sim \mathcal{N}(0, I_q)$. Following the same analysis as in the proof of Theorem \ref{thm:OneShotKG:OneStep:Convergence}, we have Lipschitzness of $\tilde{A}_{\mathrm{KG}}(\cdot, \epsilon)$:
\begin{align*}
    |\tilde{A}_{\mathrm{KG}}(\by,\epsilon) - \tilde{A}_{\mathrm{KG}}(\bx,\epsilon)|
    \leq \ell(\Phi^{-1}(\epsilon)) \|\by - \bx\|_1,
\end{align*}
where $\ell(\cdot)$ is as defined in the proof of Theorem \ref{thm:OneShotKG:OneStep:Convergence}. As before, $ \ell(\Phi^{-1}(\epsilon))$ is integrable. Like in the proof of Theorem \ref{thm:OneShotKG:OneStep:Convergence}, $\tilde{A}_{\mathrm{KG}}(\bx, \epsilon)$ is the maximum of a GP and its moment generating function $\bbE [ \,  e^{t \tilde{A}_{\mathrm{KG}}(\bx, \epsilon)} ]$ is finite for all $t$ by Lemma \ref{lem:appdx:Proofs:MGFsupffinite}, implying finiteness of its second moment: $\bbE[\tilde{A}_{\mathrm{KG}}(\bx, \epsilon)^2] <\infty$ for all $\bx \in \bbX$. Thus, that $\tilde{A}_{\mathrm{KG}}(\bx,\epsilon) \in L^2([0,1]^q)$ and $\hat{\alpha}_{\mathrm{KG}, N_i}(\bx) \rightarrow \alpha_{\mathrm{KG}}(\bx)$ a.s. as $i\rightarrow \infty$ for all $\bx \in \bbX^q$ follows from \citet[Theorem~3]{owen2020rqmcslln}. Theorem \ref{thm:appdx:GeneralSAAresults:AlmostSure} now allows us to conclude (1) and (2).
\end{proof}

The following Lemma will be used to prove Theorem~\ref{thm:OneShotKG:AsymptoticOptimality:OKG}:

\begin{lemma}
\label{lem:OneShotKG:OneStep:ContSamplePath}
    Consider a Gaussian Process $f$ on $\bbX \subset \bbR^d$ with covariance function $K( \cdot , \cdot ): \bbX \times \bbX \rightarrow \bbR$. Suppose that (i) $\bbX$ is compact, and (ii) $K$ is continuously differentiable. Then $f$ has continuous sample paths.
\end{lemma}

\begin{proof}[Proof of Lemma~\ref{lem:OneShotKG:OneStep:ContSamplePath}]
    Since $K$ is continuously differentiable and~$\bbX$ is compact, $K$ is Lipschitz on $\bbX \times \bbX$, i.e.,  $\exists\, L < \infty$ such that $|K(x,y) - K(x',y')|\leq L \bigl(\|x - x'\| + \|y -y'\|\bigr)$ for all $(x, y), (x', y') \in \bbX \times \bbX$.
    Thus
    \begin{align*}
        \bbE |f(x) &- f(x)|^2 = K(x,x) - 2 K(x,y) + K(y,y)  \\
        &\leq |K(x,x) - K(x,y)| + |K(y,y) - K(x,y)| \\
        &\leq 2L \|x - y\|
    \end{align*}
    Since $\bbX$ is compact, there exists $C:= \max_{x, y\in \bbX}\|x - y\| < \infty$. With this it is easy to verify that there exist $C' < \infty$ and $\eta > 0$ such that $2L \|x - y\| < C' |\log \|x - y\||^{-(1+\eta)}$ for all $x,y \in \bbX$. 
    Continuity of the sample paths then follows from Theorem 3.4.1 in \cite{adler2010gemetryrandom}.
\end{proof}

\begin{proof}[\textbf{Proof of Theorem~\ref{thm:OneShotKG:AsymptoticOptimality:OKG}}]
    From Lemma~\ref{lem:OneShotKG:OneStep:ContSamplePath} we know that the GP has continuous sample paths. 
    If $\bx_{n+1} \in \argmax_{\bx\in \bbX^q} \hat{\alpha}^n_{\mathrm{KG}, N_n}(\bx)$ for all $n$, the almost sure convergence of $\hat{\bx}^n_{\mathrm{KG}, N_n}$ to the set of optimizers of $\alpha^n_{\mathrm{KG}}$ from Theorem~\ref{thm:OneShotKG:OneStep:Convergence} together with continuity of~$\alpha^n_{\mathrm{KG}}$ (established in the proof of Theorem~\ref{thm:OneShotKG:OneStep:Convergence}) implies that for all $\delta > 0$ and each~$n \geq 1$, $\exists \, N_n < \infty$ such that $\alpha^n_{\mathrm{KG}}(\bx_{n+1}) > \sup_{\bx \in \bbX^q} \alpha^n_{\mathrm{KG}}(\bx) - \delta$. As $\limsup_n N_n = \infty$, $\exists \, (a_n)_{n\geq 1}$ with $a_n \rightarrow 0$ such that $\alpha^n_{\mathrm{KG}}(\bx_{n+1}) > \sup_{\bx \in \bbX^q} \alpha^n_{\mathrm{KG}}(\bx) - a_n$ infinitely often. That $\alpha^n_{\mathrm{KG}}(\bx) \rightarrow 0$ a.s. for all $\bx \in \bbX^q$ then follows from Theorem~\ref{appdx:thm:OneShotKG:AsymptoticOptimality:Bect}.
    The convergence result for $f(\chi_n)$ then follows directly from Proposition~4.9 in \cite{bect2019supermartingaleGP}.
\end{proof}

\begin{lemma}
\label{lem:appdx:Proofs:MGFsupffinite}
Let $f$ be a mean zero GP defined on $\mathbb X$ such that $|f(x)| < \infty$ almost surely for each $x \in \mathbb X$. It holds that the moment generating functions of $\sup_{x \in \mathbb X} f(x)$ and $\sup_{x \in \mathbb X} |f(x)|$ are both finite, i.e., 
\begin{equation*}
\bbE \bigl [  e^{t \sup_{x\in\mathbb X} f(x)}  \bigr ] < \infty \quad \text{and} \quad \bbE \bigl [  e^{t \sup_{x\in\mathbb X} |f(x)|}  \bigr ] < \infty
\end{equation*}
for any $t \in \mathbb R$.
\end{lemma}

\begin{proof}[\textbf{Proof of Lemma~\ref{lem:appdx:Proofs:MGFsupffinite}}]
Let $\| f \|:= \sup_{x \in \mathbb X} f$. Since the sample paths of $f$ are almost surely finite, the Borell-TIS inequality \citep[Theorem 2.1]{adler1990introduction} states that $\bbE \, \| f \| < \infty$. 
We first consider $t > 0$ and begin by re-writing the expectation as 
\begin{align}
\bbE \bigl [  e^{t \, \| f \|}  \bigr ] &= \int_{0}^\infty 
\bbP\bigl( e^{t \, \| f \|}  > u \bigr)\, du \nonumber \\
&\le 1+ \int_{1}^\infty 
\bbP\bigl( e^{t \, \| f \|}  > u \bigr)\, du \nonumber \\
&=1 + \int_{1}^\infty 
\bbP\bigl(  \| f \| - \bbE \, \| f \|  > t^{-1}  \log u - \bbE \, \| f \| \bigr)\, du\nonumber  \\
&= 1+ t e^{t \, \bbE \| f \|}  \, \int_{-\bbE \, \|f\|}^\infty 
\bbP\bigl(  \| f \| - \bbE \, \| f \|  > u \bigr) \, e^{t u} \,  du\nonumber  \\
&\le 1+ t e^{t \, \bbE \|f \|}  \,   \biggl[\int_{ \min \{-\bbE \|f\|, \, 0\} }^0  +  \int_{0}^\infty \biggr] 
\bbP\bigl(  \| f \| - \bbE \, \| f \|  > u \bigr) \, e^{t u} \,  du \nonumber \\
& \le   1  + \bigl|  \bbE \, \|f\| \bigr| \, t e^{t \, \bbE \|f \|} +  t e^{t \, \bbE \|f \|}  \, \int_{0}^\infty 
\bbP\bigl(  \| f \| - \bbE \, \| f \|  > u \bigr) \, e^{t u} \,  du, 
\label{eq:positivet}
\end{align}
where a change of variables is performed in the third equality. Let $\sigma^2_{\mathbb X} := \sup_{x \in \mathbb X} \bbE [ f(x)^2 ]$. We can now use the Borell-TIS inequality to bound the tail probability in (\ref{eq:positivet}) by $2 e^{-u^2 / (2 \sigma_\mathbb{X}^2)}$, obtaining:
\begin{align*}
\bbE \bigl [  e^{t \, \| f \|}  \bigr ] \le 1  + \bigl|  \bbE \, \|f\| \bigr| \, t e^{t \, \bbE \|f \|} +  t e^{t \, \bbE \|f \|}  \, \int_{0}^\infty 
2 e^{-u^2 / (2 \sigma_\mathbb{X}^2) + {t u}} \,  du < \infty. 
\end{align*}

Similarly, for $t < 0$, we have: 
\begin{align}
\bbE \bigl [  e^{t \, \| |f| \|}  \bigr ] &= \int_{0}^\infty 
\bbP\bigl( e^{t \, \| |f| \|}  > u \bigr)\, du\nonumber  \\
&\le 1+ \int_{1}^\infty 
\bbP\bigl( e^{t \, \| |f| \|}  > u \bigr)\, du\nonumber  \\
&=1 + \int_{1}^\infty 
\bbP\bigl(  \| |f| \| - \bbE \, \| f \|  < t^{-1}  \log u - \bbE \, \| f \| \bigr)\, du\nonumber  \\
&= 1 - t e^{t \, \bbE \|f \|}  \, \int_{-\infty}^{-\bbE \, \|f\|} 
\bbP\bigl(  \| |f| \| - \bbE \, \| f \|  < u \bigr) \, e^{t u} \,  du\nonumber \\
&\le 1 - t e^{t \, \bbE \|f \|}  \,   \biggl[\int_0^{ \max \{-\bbE \|f\|, \, 0\} }  +  \int_{-\infty}^0 \biggr] 
\bbP\bigl(  \| |f| \| - \bbE \, \| f \|  < u \bigr) \, e^{t u} \,  du\nonumber  \\
& \le   1  - \bigl|  \bbE \, \|f\| \bigr| \, t e^{t \, \bbE \|f \|} -  t e^{t \, \bbE \|f \|}  \, \int_{-\infty}^0 
\bbP\bigl(  \| |f| \| - \bbE \, \| f \|  < u \bigr) \, e^{t u} \,  du, 
\label{eq:negativet}
\end{align}
The same can be done for (\ref{eq:negativet}) to conclude that $\bbE \bigl [  e^{t \, \| f \|}  \bigr ] < \infty$ for all $t$. For the case of $\bbE \bigl [  e^{t \, \| |f| \|}  \bigr ]$ and $t > 0$, we use a similar line of analysis as (\ref{eq:positivet}) along with the observation that
\[
\bbP\bigl(  \| |f| \| - \bbE \, \| f \|  > u \bigr) \le 2 \, \bbP\bigl(  \| f \| - \bbE \, \| f \|  > u \bigr).
\]
For $t < 0$, the result is clear because $\||f| \| \ge 0$.
\end{proof}

\begin{proof}[\textbf{Proof of Theorem~\ref{thm:OneShotKG:AsymptoticOptimality:finiteOKG}}]
    Since we are in the case of finite $\mathbb X$, let $\mu_n$ and $\Sigma_n$ denote the posterior mean vector and covariance matrix of our GP after conditioning on $\mathcal D_n$. First, we give a brief outline of the argument. We know from previous work (Lemma A.6 of \cite{frazier2009knowledge} or Lemma 3 of \cite{poloczek2017misokg}) that given a posterior distribution parameterized by $\mu$ and $\Sigma$, if $\alpha_{\textnormal{KG}}(x; \mu, \Sigma) = 0$ for all $x \in \mathbb X$, then an optimal design is identified:
 \[
\argmax_{x\in \mathbb X} \mu(x) = \argmax_{x \in \mathbb X} f(x)
\]
almost surely. Thus, we can use the true KG values as a ``potential function'' to quantify how the OKG policy performs asymptotically, even though we are never using the KG acquisition function for selecting points. We emphasize that the data that induce $\{\mu_n\}_{n \ge 0}$ and $\{\Sigma_n\}_{n \ge 0}$ are collected using the OKG policy. 

By a martingale convergence argument, there exists a limiting posterior distribution described by random variables $(\mu_\infty, \Sigma_\infty)$, i.e., $\mu_n \rightarrow \mu_\infty$ and $\Sigma_n \rightarrow \Sigma_\infty$ almost surely \citep[Lemma A.5]{frazier2009knowledge}. Let $A \subseteq \mathbb X$ be a subset of the feasible space. As was done in the proof of Theorem 4 of \cite{frazier2009knowledge}, we define the event:
\begin{equation}
H_A = \bigl \{\alpha_{\textnormal{KG}}(x; \mu_\infty, \Sigma_\infty) > 0,\, x\in A \bigr \} \cap  \bigl \{\alpha_{\textnormal{KG}}(x; \mu_\infty, \Sigma_\infty) = 0,\, x \not \in A \bigr \}.
\label{eq:HA}
\end{equation}
Note that $H_A$, for all possible subsets $A$, partition the sample space. Consider some $A \ne \emptyset$. By Lemma A.7 of \cite{frazier2009knowledge}, if $\alpha_{\textnormal{KG}}(x; \mu_\infty, \Sigma_\infty) > 0$, then $x$ is measured a finite number of times, meaning that there exists an almost surely finite random variable $M_0$ such that on iterations after $N_0$, OKG stops sampling from $A$. By the definition of $H_A$ in (\ref{eq:HA}), there must exist another random iteration index $M_1 \ge M_0$ such that when $n \ge N_1$,
\[
\min_{x \in \mathcal A} \alpha_{\textnormal{KG}}(x; \mu_n, \Sigma_n) > \max_{x \not \in \mathcal A} \alpha_{\textnormal{KG}}(x; \mu_n, \Sigma_n),
\]
implying that the exact KG policy must prefer points in $\mathcal A$ over all others after iteration $M_1$. This implies that
\[
H_A \subseteq \Bigl \{ \argmax_{x \in \mathbb X} \hat{\alpha}_{\textnormal{KG}, N_n}(x, \mu_n, \Sigma_n) \not \subseteq \argmax_{x \in \mathbb X} \alpha_{\textnormal{KG}}(x, \mu_n, \Sigma_n), \, \forall \, n \ge M_1 - 1  \Bigr \} =: E,
\]
because if not, then there exists an iteration after $M_0$ where an element from $A$ is selected, which is a contradiction. 
As shown in the proof of Lemma~\ref{lem:OneShotKG:OneStep:InnerSampling}, the next period posterior mean $\mathbb E \bigl[ g(f(x')) \, | \, \mathcal D_{\mathbf x} \bigr]$ is a GP. Therefore, by Lemma \ref{lem:appdx:Proofs:MGFsupffinite}, the moment generating function of $\max_{x' \in \mathbb X} \mathbb E \bigl[ g(f(x')) \, | \, \mathcal D_{\mathbf x} \bigr]$ is finite. Theorem 2.6 of \cite{homemDeMello2008convergence} establishes that our choice of $N_n$ guarantees 
\[
\mathbb P \Bigl [ \argmax_{x \in \mathbb X} \hat{\alpha}_{\textnormal{KG}, N_n}(x, \mu_n, \Sigma_n) \not \subseteq \argmax_{x \in \mathbb X} \alpha_{\textnormal{KG}}(x, \mu_n, \Sigma_n)\, | \, \mathcal F_n \Bigr] \le \delta,
\]
from which it follows that 
\[
\sum_{n=0}^\infty \log \mathbb P \Bigl [ \argmax_{x \in \mathbb X} \hat{\alpha}_{\textnormal{KG}, N_n}(x, \mu_n, \Sigma_n) \not \subseteq \argmax_{x \in \mathbb X} \alpha_{\textnormal{KG}}(x, \mu_n, \Sigma_n)\, | \, \mathcal F_n \Bigr] = -\infty.
\]
After writing the probability of $E$ as an infinite product and performing some manipulation, we see that the above condition implies that the probability of event $E$ is zero, and we conclude that $\bbP(H_A) = 0$ for any nonempty $A$. Therefore, $\bbP(H_\emptyset) = 1$ and $\alpha_{\textnormal{KG}}(x; \mu_\infty, \Sigma_\infty) = 0$ for all $x$ almost surely.
\end{proof}

\section{Illustration of Sample Average Approximation}
\label{appdx:sec:NonStochOpt}

QMC methods have been used in other applications in machine learning, including variational inference~\citep{buchholz2018quasi} and evolutionary strategies~\citep{rowland2018gcmc}, but rarely in BO. \citet{letham2019noisyei} use QMC in the context of a specific acquisition function. \botorch{}'s abstractions make it straightforward (and mostly automatic) to use QMC integration with any acquisition function.

Using SAA, i.e., fixing the base samples $E = \{\epsilon^i\}$, introduces a consistent bias in the function approximation. While i.i.d. re-sampling in each evaluation ensures that $\hat{\alpha}_{\!N}(\bx, \Phi, \calD)$ and $\hat{\alpha}_{\!N}(\by, \Phi, \calD)$ are conditionally independent given~$(\bx, \by)$, this no longer holds when fixing the base samples.

\begin{figure}[ht]
\begin{center}
\includegraphics[width=0.8\textwidth]{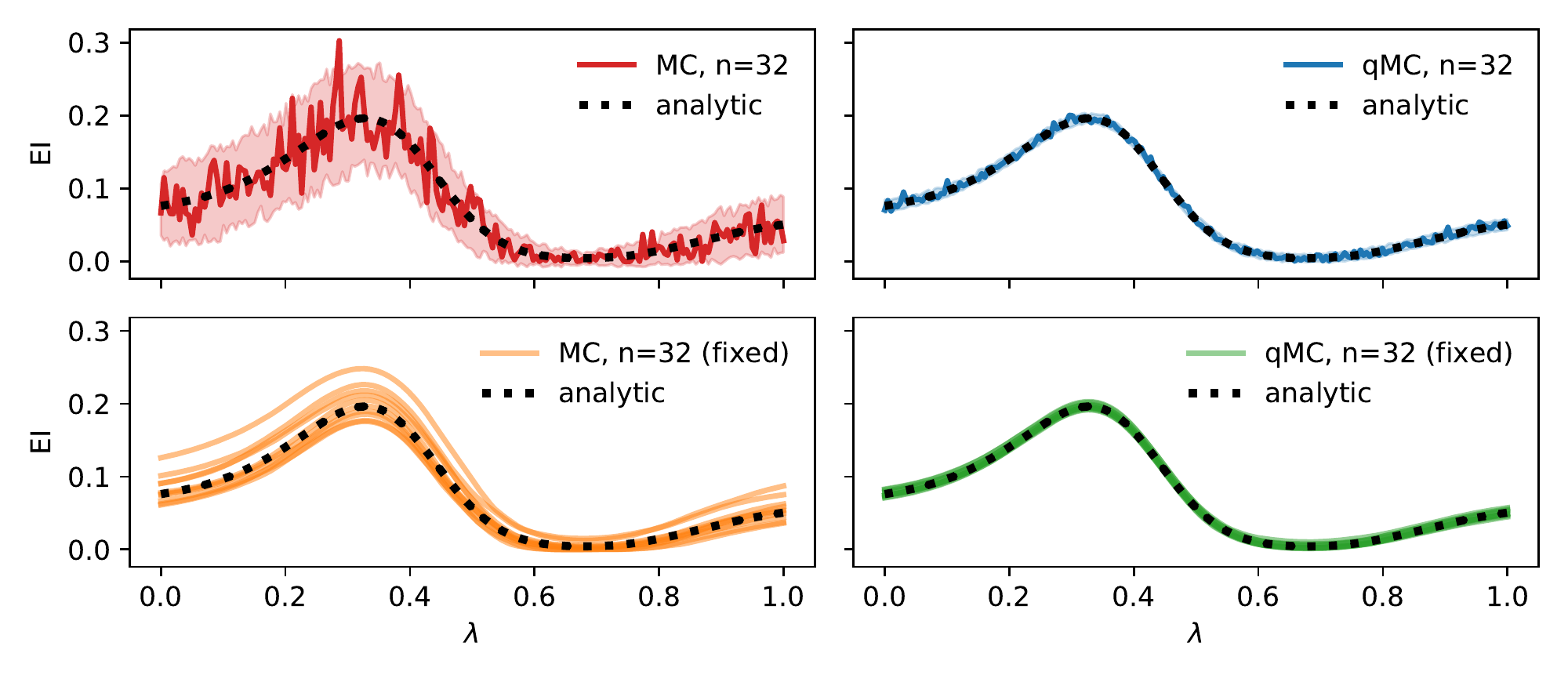}
\caption{MC and QMC acquisition functions, with and without re-drawing the base samples between evaluations. The model is a GP fit on 15 points randomly sampled from $\bbX = [0, 1]^6$ and evaluated on the (negative) Hartmann6 test function. The acquisition functions are evaluated along the slice $x(\lambda) = \lambda \mathbf{1}$.}
\label{fig:appdx:NonStochOpt:MC_QMC_qualitative}
\end{center}
\end{figure}

Figure~\ref{fig:appdx:NonStochOpt:MC_QMC_qualitative} illustrates this behavior for EI (we consider the simple case of $q=1$ for which we have an analytic ground truth available). 
The top row shows the MC and QMC version, respectively, when re-drawing base samples for every evaluation. The solid lines correspond to a single realization, and the shaded region covers four standard deviations around the mean, estimated across 50 evaluations. 
It is evident that QMC sampling significantly reduces the variance of the estimate. 
The bottom row shows the same functions for 10 different realizations of fixed base samples. Each of these realizations is differentiable w.r.t. $x$ (and hence $\lambda$ in the slice parameterization). In expectation (over the base samples), this function coincides with the true function (the dashed black line). Conditional on the base sample draw, however, the estimate displays a consistent bias. The variance of this bias (across re-drawing the base samples) is much smaller for the QMC versions.

\begin{figure*}[ht]
    \centering
    \includegraphics[width=\textwidth]{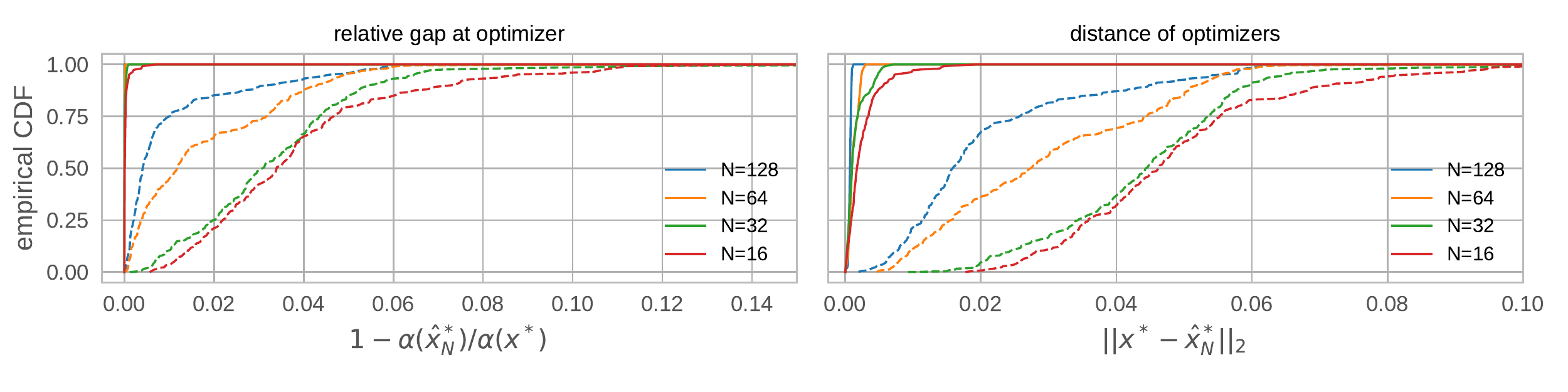}\\[-3ex]
    \caption{Performance for optimizing QMC-based EI. Solid lines: fixed base samples, optimized via L-BFGS-B. Dashed lines: re-sampling base samples, optimized via Adam (lr=0.025).}
    \label{fig:appdx:NonStochOpt:MC_QMC_optimizer_cdfs}
\end{figure*}

Even thought the function \emph{values} may show noticeable bias, the bias of the \emph{maximizer} (in~$\bbX$) is typically very small. Figure~\ref{fig:appdx:NonStochOpt:MC_QMC_optimizer_cdfs} illustrates this behavior, showing empirical cdfs of the relative gap $1 - \alpha(\hat{x}_N^*) / \alpha(x^*)$ and the distance $\|x^* - \hat{x}_N^*\|_2$ over 250 optimization runs for different numbers of samples, where $x^*$ is the optimizer of the analytic function EI, and $\hat{x}_N^*$ is the optimizer of the QMC approximation. The quality of the solution of the deterministic problem is excellent even for relatively small sample sizes, and generally better than of the stochastic optimizer. 

Figure~\ref{fig:appdx:NonStochOpt:MC_QMC_convergence_bias_variance} shows empirical mean and variance of the metrics from Figure~\ref{fig:appdx:NonStochOpt:MC_QMC_optimizer_cdfs} as a function of the number of MC samples~$N$ on a log-log scale. The stochastic optimizer used is Adam with a learning rate of 0.025. Both for the SAA and the stochastic version we use the same number of random restart initial conditions generated from the same initialization heuristic. 

Empirical asymptotic convergence rates can be obtained as the slopes of the OLS fit (dashed lines), and are given in Table~\ref{tab:appdx:NonStochOpt:MC_QMC_convergence_bias_variance_rates}. It is quite remarkable that in order to achieve the same error as the MC approximation with 4096 samples, the QMC approximation only requires 64 samples. This holds true for the bias and variance of the (relativized) optimal value as well as for the distance from the true optimizer. That said, as we are in a BO setting, we are not necessarily interested in the estimation error $\hat{\alpha}_{\!N}^* -\alpha^*$ of the optimum, but primarily in how far~$x_{\!N}^*$ is from the true optimizer $x^*$. 

\begin{figure*}[ht]
    \centering
    \includegraphics[width=\textwidth]{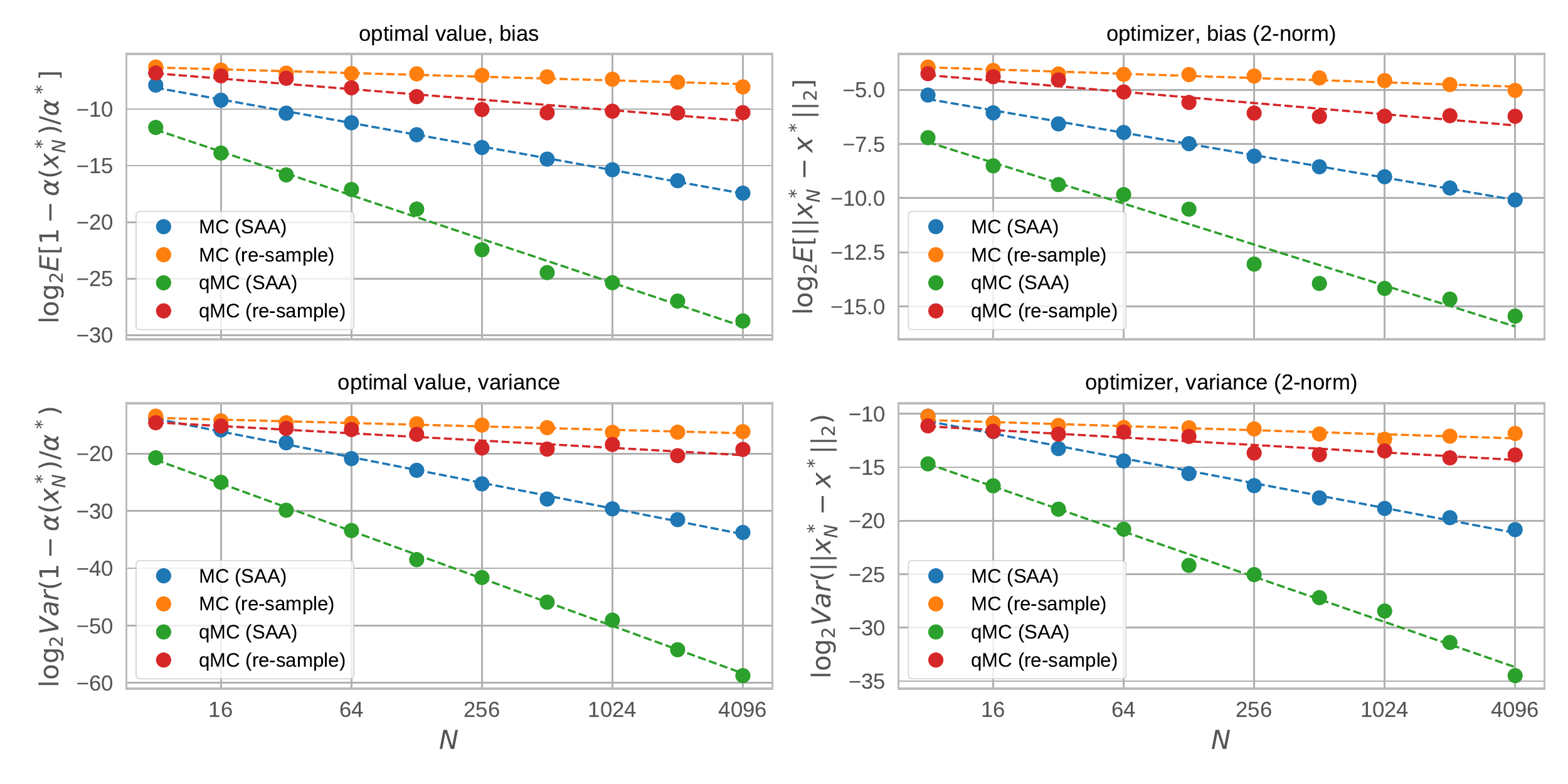}\\[-3ex]
    \caption{Bias and variance of optimizer $x_{\!N}^*$ and true EI value $\text{EI}(x_{\!N}^*)$ evaluated at the optimizer as a function of the number of (Q)MC samples for both SAA and stochastic optimzation (``re-sample'').}
    \label{fig:appdx:NonStochOpt:MC_QMC_convergence_bias_variance}
\end{figure*}

\begin{table}[ht]
    \centering
    \begin{tabular}{lrrrr}
    \toprule
    {} &    MC &  QMC &  MC$^\dagger$ &  QMC$^\dagger$ \\
    \midrule
    $\bbE[1 - \hat{\alpha}_{\!N}^*/\alpha^*]$   & $-0.52$ & $-0.95$ & $-0.10$ & $-0.26$ \\
    $\text{Var}(1 - \hat{\alpha}_{\!N}^*/\alpha^*)$   & $-1.16$ & $-2.11$ &  $-0.19$ & $-0.35$ \\
    $\bbE[\|x_{\!N}^* - x^*\|_2]$   & $-1.04$ & $-1.94$ & $-0.16$ & $-0.47$ \\
    $\text{Var}(\|x_{\!N}^* - x^*\|_2)$   & $-2.24$ & $-4.14$ & $-0.30$ & $-0.63$ \\
    \bottomrule
    \end{tabular}
    \vspace{1ex}
    \caption{Empirical asymptotic convergence rates for the setting in Figure~\ref{fig:appdx:NonStochOpt:MC_QMC_convergence_bias_variance} ($^\dagger$denotes re-sampling + optimization with Adam).}
    \label{tab:appdx:NonStochOpt:MC_QMC_convergence_bias_variance_rates}
\end{table}

A somewhat subtle point is that whether better optimization of the acquisition function results in improved closed-loop BO performance depends on the acquisition function as well as the underlying problem. More exploitative acquisition functions, such as EI, tend to show worse performance for problems with high noise levels. In these settings, not solving the EI maximization exactly adds randomness and thus induces additional exploration, which can improve closed-loop performance. While a general discussion of this point is outside the scope of this paper, \botorch{} does provide a framework for optimizing acquisition functions well, so that these questions can be compartmentalized and acquisition function performance can be investigated independently from the quality of optimization.

Perhaps the most significant advantage of using deterministic optimization algorithms is that, unlike for algorithms such as SGD that require tuning the learning rate, the optimization procedure is essentially hyperparameter-free. Figure~\ref{fig:appdx:NonStochOpt:StochasticLRDep} shows the closed-loop optimization performance of qEI for both deterministic and stochastic optimization for different optimizers and learning rates. While some of the stochastic variants (e.g. ADAM with learning rate 0.01) achieve performance similar to the deterministic optimization, the type of optimizer and learning rate matters. In fact, the rank order of SGD and ADAM w.r.t. to the learning rate is reversed, illustrating that selecting the right hyperparameters for the optimizer is itself a non-trivial problem.

\section{Additional Implementation Details}
\label{appdx:sec:ImplementatonDetails}

\subsection{Batch Initialization for Multi-Start Optimization}
\label{appdx:subsec:ImplementatonDetails:ICHeuristic}

For most acquisition functions, the optimization surface is highly non-convex, multi-modal, and (especially for ``improvement-based'' ones such as EI or KG) often flat (i.e. has zero gradient) in much of the domain~$\bbX$. Therefore, optimizing the acquisition function is  itself a challenging problem.

The simplest approach is to use zeroth-order optimizers that do not require gradient information, such as DIRECT or CMA-ES \citep{jones1993direct, hansen2001cmaes}. These approaches are feasible for lower-dimensional problems, but do not scale to higher dimensions. Note that performing parallel optimization over $q$ candidates in a $d$-dimensional feature space means solving a $qd$-dimensional optimization problem.

A more scalable approach incorporates gradient information into the optimization. As described in Section~\ref{sec:OptimizeAcq}, \botorch{} by default uses quasi-second order methods, such as L-BFGS-B. Because of the complex structure of the objective, the initial conditions for the algorithm are extremely important so as to avoid getting stuck in a potentially highly sub-optimal local optimum. To reduce this risk, one typically employs multi-start optimization (i.e. start the solver from multiple initial conditions and pick the best of the final solutions). To generate a good set of initial conditions,  \botorch{} heavily exploits the fast batch evaluation discussed in the previous section. 
Specifically, \botorch{} by default uses $N_{\text{opt}}$ initialization candidates generated using the following heuristic:
\begin{enumerate}[leftmargin=4ex,topsep=0ex,itemsep=0ex]
    \item Sample $\tilde{N}_0$ quasi-random $q$-tuples of points~$\tilde{\bx}_0 \in \bbR^{\tilde{N}_0 \times q \times d}$ from~$\bbX^q$ using quasi-random Sobol sequences.
    \item Batch-evaluate the acquisition function at these candidate sets: $\tilde v = \alpha(\tilde{\bx}_0; \Phi, \calD)$.
    \item Sample $N_0$ candidate sets $\bx \in \bbR^{N_0 \times q \times d}$ according to the weight vector $p \propto \exp(\eta v)$, where $v = (\tilde v - \hat{\mu}(\tilde v)) / \hat\sigma(\tilde v)$ with $\hat\mu$ and $\hat\sigma$ the empirical mean and standard deviation, respectively, and $\eta >0$ is a temperature parameter. Acquisition functions that are known to be flat in large parts of~$\bbX^q$ are handled with additional care in order to avoid starting in locations with zero gradients.
\end{enumerate}

Sampling initial conditions this way achieves an exploration/exploitation trade-off controlled by the magnitude of~$\eta$. As $\eta \rightarrow 0$ we perform Sobol sampling, while $\eta \rightarrow \infty$ means the initialization is chosen in a purely greedy fashion.  The latter is generally not advisable, since for large $\tilde{N}_0$ the highest-valued points are likely to all be clustered together, which would run counter to the goal of multi-start optimization. 
Fast batch evaluation allows evaluating a large number of samples ($\tilde{N}_0$ in the tens of thousands is feasible even for moderately sized models).

\subsection{Sequential Greedy Batch Optimization}
\label{appdx:subsec:ImplementatonDetails:SeqGreedy}

The pending points approach discussed in Section~\ref{sec:ModularBO} provides a natural way of generating parallel BO candidates using \emph{sequential greedy} optimization, where candidates are chosen sequentially, while in each step conditioning on selected points and integrating over the uncertainty in their outcome (using MC integration). By using a full MC formulation, in which we jointly sample at new and pending points, we avoid constructing an individual ``fantasy'' model for each sampled outcome, a common (and costly) approach in the literature~\citep{snoek2012practical}. In practice, the sequential greedy approach often performs well, and may even outperform the joint optimization approach, since it involves a sequence of small, simpler optimization problems, rather than a larger and complex one that is harder to solve.

\citep{wilson2018maxbo} provide a theoretical justification for why the sequential greedy approach works well with a class of acquisition functions that are submodular.

\section{Active Learning Example}
\label{sec:appdx:ActiveLearning}

Recall from Section~\ref{sec:Abstractions:ImplementationExamples} the negative integrated posterior variance (NIPV) \citep{seo2000activedata,chen2014stochkriging} of the model:
\begin{align}
\label{eq:appdx:ActiveLearning:NIV}
    \text{NIPV}(\bx) = - \int_{\bbX} \bbE \bigl[ \variance(f(x) \!\mid\! \calD_{\bx}) \mid \calD \bigr] \, dx.
\end{align}
We can implement~\eqref{eq:appdx:ActiveLearning:NIV} using standard \botorch{} components, as shown in Code Example~\ref{codex:appdx:ActiveLearning:qNIPV}. Here {\ttm{mc\_points}} is the set of points used for MC-approximating the integral. 
In the most basic case, one can use QMC samples drawn uniformly in~$\bbX$. By allowing for arbitrary {\ttm{mc\_points}}, we permit weighting regions of~$\bbX$ using non-uniform sampling. Using {\ttm{mc\_points}} as samples of the maximizer of the posterior, we recover the recently proposed Posterior Variance Reduction Search \citep{nguyen2017pvrs} for BO.

\begin{python}[
caption=Active Learning (NIPV),
label=codex:appdx:ActiveLearning:qNIPV,
float=ht
]
class qNegativeIntegratedPosteriorVariance(AnalyticAcquisitionFunction):

    @concatenate_pending_points
    @t_batch_mode_transform()
    def forward(self, X: Tensor) -> Tensor:
        fant_model = self.model.fantasize(
            X=X, sampler=self._dummy_sampler,
            observation_noise=True
        )
        sz = [1] * len(X.shape[:-2]) + [-1, X.size(-1)]
        mc_points = self.mc_points.view(*sz)
        with settings.propagate_grads(True):
            posterior = fant_model.posterior(mc_points)
        ivar = posterior.variance.mean(dim=-2)
        return -ivar.view(X.shape[:-2])
\end{python}

This acquisition function supports both parallel selection of points and asynchronous evaluation. Since MC integration requires evaluating the posterior variance at a large number of points, this acquisition function benefits significantly from the fast predictive variance computations in GPyTorch \citep{pleiss2018love, gardner2018gpytorch}.

To illustrate how NIPV may be used in combination with scalable probabilistic modeling, we examine the problem of efficient allocation of surveys across a geographic region. Inspired by~\citet{cutajar2019deep}, we utilize publicly-available data from     \citetalias{malariaatlas2019} dataset, which includes the yearly mean parasite rate (along with standard errors) of \emph{Plasmodium falciparum} at a $4.5 \text{km}^2$ grid spatial resolution across Africa.
In particular, we consider the following active learning problem: given a spatio-temporal probabilistic model fit to data from 2011-2016, which geographic locations in and around Nigeria should one sample in 2017 in order to minimize the model's error for 2017 across all of Nigeria?

We fit a heteroskedastic GP model to 2500 training points prior to 2017 (using a noise model that is itself a GP fit to the provided standard errors). We then select $q=10$ sample locations for 2017 using the NIPV acquisition function, and make predictions across the entirety of Nigeria using this new data. Compared to using no 2017 data, we find that our new dataset reduces MSE by 16.7\% on average (SEM = 0.96\%) across 60 subsampled datasets.
By contrast, sampling the new 2017 points at a regularly spaced grid results only in a 12.4\% reduction in MSE (SEM = 0.99\%). The mean relative improvement in MSE reduction from NIPV optimization is 21.8\% (SEM = 6.64\%). 
Figure~\ref{appdx:fig:ActiveLearning:MalariaSamples} shows the NIPV-selected locations on top of the base model's estimated parasite rate and standard deviation.

\begin{figure}[ht]
    \centering
    \includegraphics[width=0.85\columnwidth]{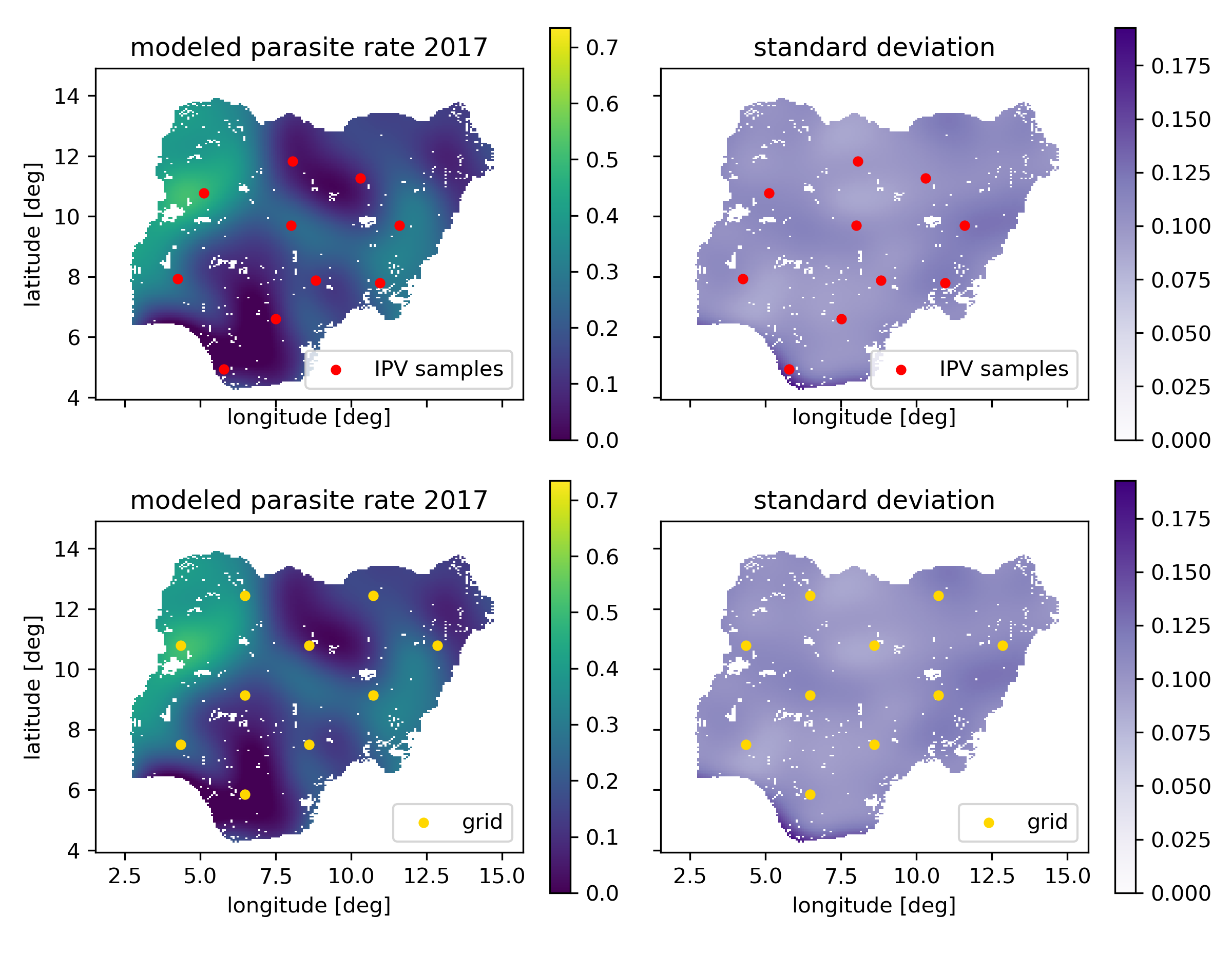}\\[-2.5ex]
    \caption{Locations for 2017 samples from IPV minimization and the base grid. Observe how the NIPV samples cluster in higher variance areas.}
    \label{appdx:fig:ActiveLearning:MalariaSamples}
\end{figure}

\section{Additional Implementation Examples}
\label{appdx:sec:Examples}

\subsection*{Comparing Implementation Complexity}

Many of \botorch{}'s benefits are qualitative, including the simplification and acceleration of implementing new acquisition functions. Quantifying this in a meaningful way is very challenging. Comparisons are often made in terms of Lines of Code (LoC) - while this metric is problematic when comparing across different design philosophies, non-congruent feature sets, or even programming languages, it does provides a general idea of the effort required for developing and implementing new methods. 

MOE’s KG involves thousands of LoC in C++ and python spread across a large number of files,\footnote{\url{https://github.com/wujian16/Cornell-MOE}} while our more efficient implementation is <30 LoC. \citet{astudillo2019composite} is a full paper in last year's installment of this conference,\footnote{Code available at \url{https://github.com/RaulAstudillo06/BOCF}} whose composite function method we implement and significantly extend (e.g to support KG) in 7 LoC using BoTorch's abstractions. The original NEI implementation is >250 LoC, while the one from Code Example~\ref{codex:Abstractions:ImplementationExamples:NEI:qNEI} is 14 LoC.

\subsection{Composite Objectives}
\label{appdx:subsec:Example:Composite}

We consider the Bayesian model calibration of a simulator with multiple outputs from Section 5.3 of \citet{astudillo2019composite}. In this case, the simulator from \citet{bliznyuk2008bayesian} models the concentrations of chemicals at 12 positions in a one-dimensional channel. Instead of modeling the overall loss function (which measures the deviation of the simulator outputs with a set of observations) directly, we follow \citet{astudillo2019composite} and model the underlying concentrations while utilizing a composite objective approach.
A powerful aspect of \botorch{}'s modular design is the ability to easily combine different approaches into one. For the composite function problem in this section this means that we can easily extend the work by \citet{astudillo2019composite} not only to use the Knowledge Gradient, but also to the ``parallel BO'' setting of jointly selecting $q>1$ points.
Figures~\ref{appdx:fig:Example:Composite:qKG_cf_q1} and \ref{appdx:fig:Example:Composite:qKG_cf_q1} show results for this with~$q=1$ and $q=3$, repspectively. The plots show log regret evaluated at the maximizer of the posterior mean averaged over 250 trials.
While the performance of EI-CF is similar for~$q=1$ and~$q=3$, KG-CF reaches lower regret significantly faster for~$q=1$ compared to $q=3$, suggesting that ``looking ahead`` is  beneficial in this context.

\begin{figure*}
    \begin{minipage}[l]{0.5\columnwidth}
    \centering
    \includegraphics[width=\columnwidth]{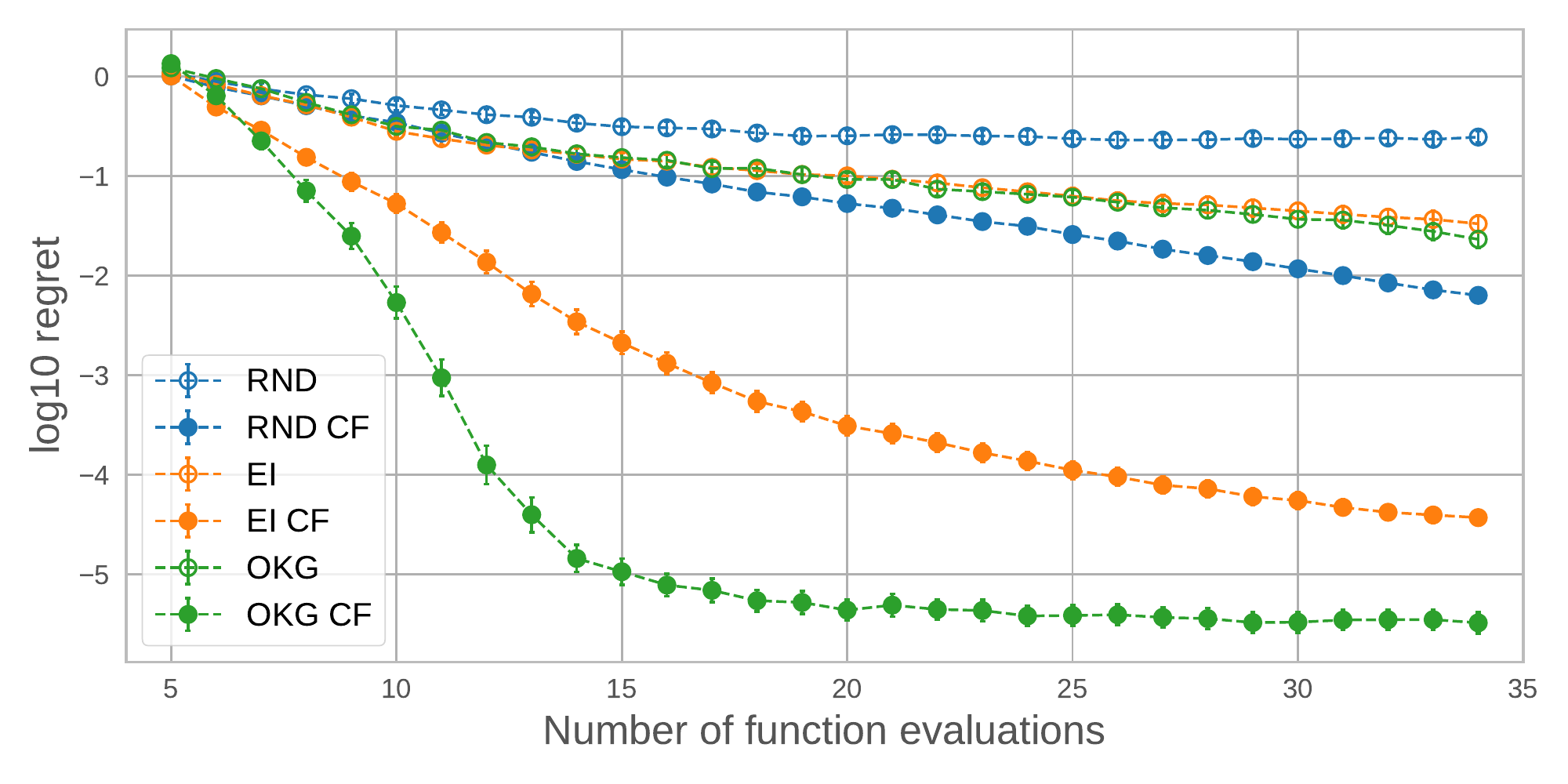}
    \caption{Composite function optimization for $q=1$}
    \label{appdx:fig:Example:Composite:qKG_cf_q1}
    \end{minipage}
    \hfill
    \begin{minipage}[l]{0.5\columnwidth}
    \centering
\includegraphics[width=\columnwidth]{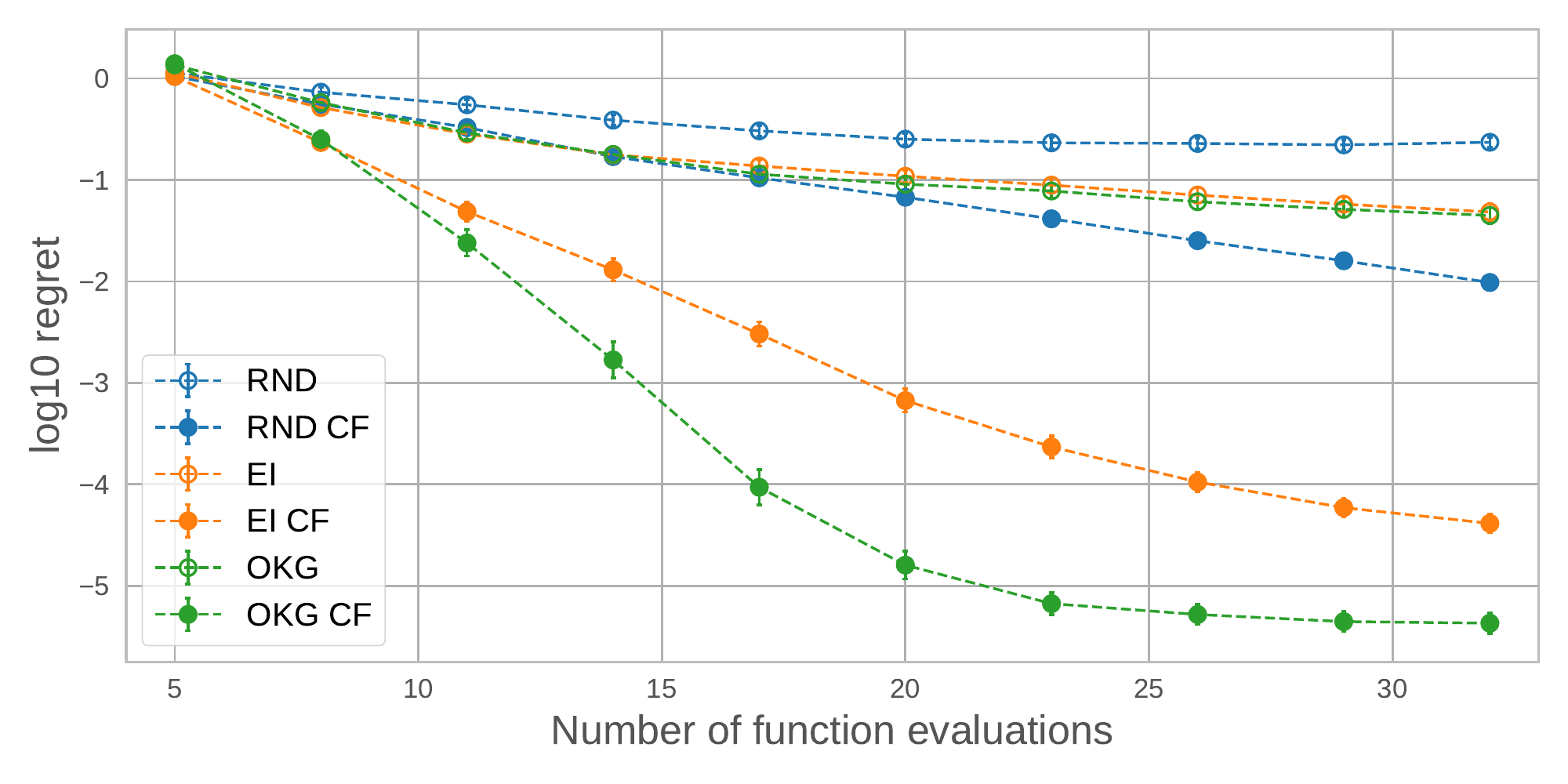}
\caption{Composite function optimization for $q=3$}
    \label{appdx:fig:Example:Composite:qKG_cf_q3}
\end{minipage}
\end{figure*}

\subsection{Generalized UCB}
\label{appdx:subsec:Example:GenUCB}

Code Example~\ref{codex:appdx:Example:GenUCB:qUCB} presents a  generalized version of parallel UCB from \citet{wilson2017reparamacq} supporting pending candidates, generic objectives, and QMC sampling. If no {\ttm{sampler}} is specified, a default QMC sampler is used. Similarly, if no {\ttm{objective}} is specified, the identity objective is assumed.

\begin{python}[
caption=Generalized Parallel UCB,
label=codex:appdx:Example:GenUCB:qUCB,
basicstyle=\small\ttfamily,,
float,
]
class qUpperConfidenceBound(MCAcquisitionFunction):

    def __init__(
        self,
        model: Model,
        beta: float,
        sampler: Optional[MCSampler] = None,
        objective: Optional[MCAcquisitionObjective] = None,
        X_pending: Optional[Tensor] = None,
    ) -> None:
        super().__init__(model, sampler, objective, X_pending)
        self.beta_prime = math.sqrt(beta * math.pi / 2)

    @concatenate_pending_points
    @t_batch_mode_transform()
    def forward(self, X: Tensor) -> Tensor:
        posterior = self.model.posterior(X)
        samples = self.sampler(posterior)
        obj = self.objective(samples)
        mean = obj.mean(dim=0)
        z = mean + self.beta_prime * (obj - mean).abs()
        return z.max(dim=-1).values.mean(dim=0)
\end{python}

\subsection{Full Code Examples}
\label{appdx:subsec:Example:FullImplementations}

In this section we provide full implementations for the code examples. Specifically, we include parallel Noisy EI (Code Example~\ref{codex:appdx:Example:NEI:qNEI}), \OKG{} (Code Example~\ref{codex:appdx:Example:OSKG:qKG}), and (negative) Integrated Posterior Variance (Code Example~\ref{codex:appdx:Example:NIPV:qNIPV}).

\begin{python}[
caption=Parallel Noisy EI (full),
label=codex:appdx:Example:NEI:qNEI,
basicstyle=\small\ttfamily,
float,
]
class qNoisyExpectedImprovement(MCAcquisitionFunction):

    def __init__(
        self,
        model: Model,
        X_baseline: Tensor,
        sampler: Optional[MCSampler] = None,
        objective: Optional[MCAcquisitionObjective] = None,
        X_pending: Optional[Tensor] = None,
    ) -> None:
        super().__init__(model, sampler, objective, X_pending)
        self.register_buffer("X_baseline", X_baseline)

    @concatenate_pending_points
    @t_batch_mode_transform()
    def forward(self, X: Tensor) -> Tensor:
        q = X.shape[-2]
        X_bl = match_shape(self.X_baseline, X)
        X_full = torch.cat([X, X_bl], dim=-2)
        posterior = self.model.posterior(X_full)
        samples = self.sampler(posterior)
        obj = self.objective(samples)
        obj_n = obj[...,:q].max(dim=-1).values
        obj_p = obj[...,q:].max(dim=-1).values
        return (obj_n - obj_p).clamp_min(0).mean(dim=0)
\end{python}

\begin{python}[
caption=One-Shot Knowledge Gradient (full),
label=codex:appdx:Example:OSKG:qKG,
basicstyle=\small\ttfamily,
float,
]
class qKnowledgeGradient(MCAcquisitionFunction):

    def __init__(
        self,
        model: Model,
        sampler: MCSampler,
        objective: Optional[Objective] = None,
        inner_sampler: Optional[MCSampler] = None,
        X_pending: Optional[Tensor] = None,
    ) -> None:
        super().__init__(model, sampler, objective, X_pending)
        self.inner_sampler = inner_sampler

    def forward(self, X: Tensor) -> Tensor:
        splits = [X.size(-2) - self.Nf, self.N_f]
        X, X_fantasies = torch.split(X, splits, dim=-2)
        # [...] some re-shaping for batch evaluation purposes
        if self.X_pending is not None:
            X_p = match_shape(self.X_pending, X)
            X = torch.cat([X, X_p], dim=-2)
        fmodel = self.model.fantasize(
            X=X,
            sampler=self.sampler,
            observation_noise=True,
        )
        obj = self.objective
        if isinstance(obj, MCAcquisitionObjective):
            inner_acqf = SimpleRegret(
                fmodel, sample=self.inner_sampler, objective=obj,
            )
        else:
            inner_acqf = PosteriorMean(fmodel, objective=obj)
        with settings.propagate_grads(True):
            values = inner_acqf(X_fantasies)
        return values.mean(dim=0)
\end{python}

\begin{python}[
caption=Active Learning (full),
label=codex:appdx:Example:NIPV:qNIPV,
basicstyle=\small\ttfamily,
float,
]
class qNegIntegratedPosteriorVariance(AnalyticAcquisitionFunction):
    def __init__(
        self,
        model: Model,
        mc_points: Tensor,
        X_pending: Optional[Tensor] = None,
    ) -> None:
        super().__init__(model=model)
        self._dummy_sampler = IIDNormalSampler(1)
        self.X_pending = X_pending
        self.register_buffer("mc_points", mc_points)

    @concatenate_pending_points
    @t_batch_mode_transform()
    def forward(self, X: Tensor) -> Tensor:
        fant_model = self.model.fantasize(
            X=X,
            sampler=self._dummy_sampler,
            observation_noise=True,
        )
        batch_ones = [1] * len(X.shape[:-2])
        mc_points = self.mc_points.view(*batch_ones, -1, X.size(-1))
        with settings.propagate_grads(True):
            posterior = fant_model.posterior(mc_points)
        ivar = posterior.variance.mean(dim=-2)
        return -ivar.view(X.shape[:-2])
\end{python}

\end{document}